\documentclass[twoside,11pt]{article}

\usepackage[abbrvbib, preprint]{jmlr2e}
\usepackage{bbm}
\usepackage{hyperref}
\usepackage{amsmath}
\usepackage{amssymb}
\usepackage{mathrsfs}
\usepackage{caption,subcaption}
\usepackage{booktabs}
\usepackage{multicol}
\usepackage{multirow}
\usepackage{xcolor}
\usepackage{color}
\usepackage{algorithm}
\usepackage{pifont}
\usepackage{mathrsfs}
\usepackage{algpseudocode}
\usepackage{comment}
\usepackage{listings}
\usepackage{enumerate}
\usepackage[normalem]{ulem}
\newcommand{\bm}[1]{\boldsymbol{#1}}

\newcommand{\dd}{\mathrm{d}}

\newcommand{\E}{\mathbb{E}}
\newcommand{\p}{\mathbb{P}}

\newcommand{\R}{\mathbb{R}}
\newcommand{\Lc}{\mathcal{L}}
\newcommand{\bU}{\mathbf{U}}
\newcommand{\cH}{\mathcal{H}}

\newtheorem{assumption}{Assumption}

\newtheorem{cond}{Condition}

\newcommand{\tabincell}[2]{\begin{tabular}{@{}#1@{}}#2\end{tabular}}
\newcommand{\cmark}{\ding{51}}%
\newcommand{\xmark}{\ding{55}}%

\newcommand{\1}{\mathbf{1}}

\newcommand{\cL}{\mathcal{L}}

\DeclareMathOperator{\Var}{Var}
\usepackage{mathtools}
\mathtoolsset{showonlyrefs}

\numberwithin{equation}{section}

\newcommand{\defeq}{\coloneqq}
\newcommand{\normmix}[1]{\left\|#1\right\|_{\beta, \text{mix}}}

\usepackage{lastpage}
\ShortHeadings{Zeroth-Order Deep Learning for PDEs}{Jia, Ouyang, Pham, and Zhou}
\firstpageno{1}

\begin{document}

\title{A Zeroth-Order Deep Learning Method for Fully Nonlinear Parabolic Partial Differential Equations with Unknown Coefficients}

\author{\name Yanwei Jia \email yanweijia@cuhk.edu.hk \\
\addr Department of Systems Engineering and Engineering Management\\
The Chinese University of Hong Kong\\
Shatin, New Territories, Hong Kong SAR
       \AND
\name Du Ouyang \email duouyang99@outlook.com \\
\addr Department of Mathematical Sciences\\
Tsinghua University\\
Beijing, Beijing 100084, China
\AND
\name Huy\^en Pham \email huyen.pham@polytechnique.edu \\
\addr CMAP \\ Ecole Polytechnique\\
Route de Saclay, Palaiseau 91128, France
\AND
\name Xun Yu Zhou \email xz2574@columbia.edu \\
\addr Department of Industrial Engineering and Operations Research \&\\
The Data Science Institute\\
Columbia University\\
New York, NY 10027, USA
}

\editor{}

\maketitle

\begin{abstract}%   <- trailing '%' for backward compatibility of .sty file
%High-dimensional partial differential equations (PDEs) with unknown coefficients are prevalent in applications, and solving them efficiently in a data-driven way remains a fundamental challenge in scientific machine learning, including continuous-time reinforcement learning. Conventional deep learning approaches rely heavily on repeated automatic differentiation of neural networks to evaluate differential operators, which leads to numerical instabilities and error amplification in high-dimensional spaces. Existing model-based probabilistic methods mitigate this by relying on stochastic representations, but these require explicit knowledge of the underlying data-generating dynamics in order to compute derivatives, precluding their use in black-box environments.

High-dimensional partial differential equations (PDEs) with unknown coefficients arise widely in scientific machine learning, including continuous-time reinforcement learning, yet solving them efficiently in a data-driven way remains challenging. Existing deep learning solvers often rely on repeated automatic differentiation to evaluate differential operators, which can cause instability and amplify derivative errors in high dimensions, while probabilistic methods based on stochastic representations require explicit knowledge of the data-generating dynamics and therefore do not apply to black-box environments. 
We introduce two types of simulators
as data-generating mechanisms, and take a ``representing-then-learning" approach that
learns the solutions and their derivatives under settings where the underlying PDE operators are accessible
only through simulations and pointwise evaluations. Our representation of derivatives relies on the zeroth-order derivative (ZOD) estimators derived from perturbed Monte Carlo trajectories. This fully model-free approach generates targets for the gradient and Hessian networks using only function evaluations.
We provide a statistical learning analysis of the proposed approach, including a bias--variance tradeoff for ZODs. Assuming a standard contraction property of the underlying operator, we establish a non-asymptotic error bound that decomposes the total error into discretization error, approximation error, statistical error, and ZOD bias. Crucially, we derive the sample complexity of the learned representations in (weighted) Sobolev space, characterizing the error up to second-order derivatives. Numerical experiments illustrate the competitive performance of the method in moderate and high dimensions.
\end{abstract}

\begin{keywords}
Nonlinear PDEs, zeroth-order derivative estimators, model-free methods, deep learning,  sample complexity in Sobolev space
\end{keywords}

\section{Introduction}
High-dimensional nonlinear partial differential equations (PDEs) are ubiquitous and fundamental in various fields of science and engineering, ranging from physics and fluid dynamics to mathematical finance and stochastic control. Recent progress in scientific machine learning employs  deep neural networks (DNNs) for high-dimensional function approximation to solve PDEs, aiming to overcome the curse of dimensionality where the traditional numerical methods fall apart. While various deep learning methods have been proposed and have achieved promising numerical results, there are still limitations in their stability, scope of applications, and availability of theoretical foundations.

Broadly speaking, existing methods fall into two main categories. The first category consists of the residual-based methods, such as the deep Galerkin method (cf. \citealt{siri18}),  and physics-informed neural networks (PINNs; cf. \citealt{raissi19}). These approaches approximate solutions by minimizing the PDE residuals over sampled points in space--time, and are applicable to a wide class of PDEs. However, they rely on repeated evaluation of differential operators through automatic differentiation and may suffer from optimization difficulties and stability issues, particularly in high dimensions or when second-order derivatives are involved.
The second category consists of probabilistic methods that exploit stochastic representations of PDEs, including in particular methods based on backward stochastic differential equations (BSDEs) such as deep BSDEs \citep{han17} and deep backward dynamic programming (DBDP; cf. \citealt{hure20}). However, these methods are primarily tailored to the special type of nonlinear PDEs where BSDE representations are available based on Peng's nonlinear Feynman--Kac formula \citep{peng1992nonlinear}; in particular, these PDEs are linear in the Hessians.  For {\it fully} nonlinear PDEs, which are nonlinear in the unknown functions up to the second-order derivatives,  such representations generally do not exist, limiting their applicability.\footnote{We refer to \citet{bec23,germain23} for an overview, and to \citet{zhangetal26,doumeche2025convergence,gazoulis2025stability,bonito2026convergence,wang20222} for more recent progress on these two types of methods.}

%Notably, besides obtaining the function approximation to the solution to a PDE, one often needs to find its derivatives for the downstream tasks, e.g., executing a control in a robotic system and hedging in finance, which concern the sensitivity of the solution  with respect to the spatial variable.
%\sout{The above-mentioned conventional deep learning methods typically follow a ``learning-then-differentiating" paradigm. They train a neural network to minimize a value-based loss, which nonetheless does not theoretically guarantee that its derivatives will converge to the true derivatives. Consequently, when higher-order derivatives are computed via automatic differentiation on an imperfectly trained network, approximation errors compound, sometimes  catastrophically.} %This phenomenon is well-documented in the literature surrounding Sobolev training and non-parametric regression.
%\ora{
%Existing methods generally treat derivative information implicitly: either derivatives are recovered by differentiating a learned value network (e.g., PINNs and deep Galerkin methods), or they are embedded within model-specific stochastic formulations (e.g., deep BSDE and DBDP methods).
%}

%Another important aspect is that 
Notably, most of the existing literature on numerically solving PDEs, including the aforementioned papers, are ``model-based", namely, the studied PDEs are completely specified with their coefficient functions assumed to be known and given (hence so are the derivatives of those coefficients). However, PDEs that model {\it complex}  systems
can be impossible or impractical to specify, leading to equations without knowledge of the underlying operators and the  derivatives of their coefficients. These PDEs operate like black-boxes for which we can only observe certain input--output data streams.
Solving those black-box PDEs is motivated by ample real-life applications such as fluid turbulence, climate control, and the stabilization of plasma fusion reactors; see e.g. \cite{li2020fourier,jiao2021one,lagergren2020learning} and the references therein for more motivating applications. It is also related to the recently developed continuous-time reinforcement learning \citep{wang2020reinforcement,jia2022policy,jia2022pg,jia2022q}, where the Hamilton--Jacobi--Bellman (HJB) equation becomes a black-box PDE when the parameters of the underlying control problem are unknown.

Another important aspect is that, besides obtaining the function approximation to the solution to a PDE, one often needs to find its derivatives for the downstream tasks, e.g., executing a control in a robotic system and hedging in finance, which concern the sensitivity of the solution  with respect to the spatial variable. Existing methods generally treat derivatives implicitly. For example, PINNs and deep Galerkin follow a ``learning-then-differentiating" approach:  They train a neural network to minimize a value-based loss, which nonetheless does not theoretically guarantee that its derivatives will converge to the true derivatives. Consequently, when higher-order derivatives are computed via automatic differentiation on an imperfectly trained network, approximation errors compound, sometimes  catastrophically.
On the other hand, deep BSDE and DBDP require the equation coefficients to be known and given and embed gradients  within model-specific stochastic formulations. 

In this paper, we propose a new, data-driven approach to learning the solutions to a large class of {\it fully} nonlinear, black-box parabolic PDEs with strong theoretical guarantees in (weighted) Sobolev space up to second-order spatial derivatives. The PDEs consist of linear second-order parabolic operators and fully nonlinear source (nonhomogeneous) terms depending on the unknown solutions up to their second-order derivatives. Instead of recovering derivative information indirectly from a learned value approximation or from model-specific sensitivity representations, we adopt a representing-then-learning paradigm, where derivative quantities are first represented through stochastic estimators and then learned directly. This approach kills two birds with one stone: it deals with the unknown PDEs and learns the derivatives at the same time. 

One of the main contributions of the paper is to define {\it precisely} what ``data" and ``data-driven" mean in our context of solving PDEs. Instead of the knowledge of the functional forms of PDE coefficients, we only require
access to a ``\textit{simulator}" that, given an initial state, produces stochastic trajectories (i.e., ``data"). These trajectories are diffusion processes, or equivalently solutions to the stochastic differential equations (SDEs) whose infinitesimal generator coincides with the aforementioned parabolic PDE operators. Moreover, the nonlinear source term and the function in the terminal condition
are also black-box functions that can only output their function values when inputs are given.
Our approach is applicable under two simulator settings: a ``weak simulator" that takes an initial state as input and generates an independently and identically distributed (i.i.d.) random trajectory at each query, and a ``strong simulator" that returns a {\it family} of trajectories indexed by a given set of initial states and driven by the {\it same} underlying randomness (Brownian motion in our setting), while remaining independent across queries. The notions of weak and strong simulators align with the classical definitions of weak and strong solutions to SDEs and capture the key distinction between them: a strong solution is defined with respect to a fixed Brownian motion, whereas in a weak solution the Brownian motion is itself a part of the solution (cf. \citealt[Chapter 5, Definitions 2.1 and 3.1]{karatzas2014brownian}). Clearly,  a strong simulator is automatically a weak one (but not {\it vice versa}) for which the ``random seed" is also an input, so that querying the same initial state with the same seed produces the same trajectory. Importantly, the availability of the strong simulator covers model-based PDEs as a special case. It is also valid in applications such as finance where the ``random seed" represents the realized asset price movements and   inventory systems where it represents exogenous demands. %,  and generative AI, where it is the seed used for noise generation in a more complex algorithm. %Technically, producible simulator implies that the generated trajectories can be written as a function of the initial state and a sequence of exogenous random sources.
Note that in both cases, the simulator remains a black box, and we neither know nor attempt to learn or infer its internal structure.\footnote{There is a line of research called operator learning that aims to learn PDE solution operators by input--output data and DNNs (cf. \citealt{lu2021learning}), which is ultimately a model-based approach. By contrast, our approach is ``model-free" (up to the form of the PDEs we study) in that we learn the solutions directly without learning the PDEs themselves.}

Our method combines:
(i) approximate value iteration, to reduce a nonlinear problem to a sequence of linear problems, each effectively a Feynman--Kac-type equation (equivalent to the Bellman equation for policy evaluation in reinforcement learning); and (ii) zeroth-order derivative (ZOD) estimators, to approximate gradients and Hessians via function values at random perturbed points. At each iteration, we {\it jointly} learn the value, gradient, and Hessian through a composite loss function, ensuring consistency across all quantities. This loss function consists of two parts: the value loss and the derivative loss. The value loss is in the same spirit as the martingale loss function \citep{jia2022policy} for policy evaluation in continuous-time reinforcement learning, which captures the mean-squared error of the function value. The derivative loss captures the distance between the gradient/Hessian networks and their stochastic ZOD targets. Crucially, ZOD estimators require only function evaluations at perturbed trajectories and do not rely on model derivatives or automatic differentiation.

The main contributions of this paper are as follows:
\begin{itemize}
\item We introduce the notions of weak and strong simulators, and develop a zeroth-order deep learning method to solve a wide class of fully nonlinear black-box PDEs.  At each iteration, we query the available simulator with randomly generated initial states based on a specially designed distribution and collect corresponding trajectories. In addition, we only rely on the point-wise evaluation of the nonlinear source term in a PDE, which is a general function of the state, function value, gradient, and Hessian. Thus, our method leads to a {\it model-free} algorithm applicable in black-box environments, in particular without attempting to estimate the PDE coefficients.
\item To ensure convergence in (weighted) Sobolev space, we introduce a novel use of ZOD estimators as targets for approximating gradients and Hessians in the context of solving PDEs. In general, ZOD estimators are biased and noisy, and we characterize their unique bias--variance tradeoff through their bandwidth. In particular, the conventional variance reduction techniques for ZOD in optimization (e.g., the multi-point scheme) are not sufficient in our setting, and we show that the variance of ZOD can only be bounded with the strong simulator but not the weak one. Furthermore, using ZODs, we essentially first derive a (biased) stochastic representation for the gradient and Hessian and then fit two neural networks for them respectively. %\sout{This representing-then-learning paradigm is therefore fundamentally different from the conventional ``learning-then-differentiating" paradigm.}
This representing-then-learning paradigm is therefore fundamentally different from approaches that recover derivatives indirectly from value approximations or through model-specific sensitivity representations.
Notably, no automatic differentiation or the knowledge of the operator is required to derive ZOD.
\item Most importantly, we establish convergence and sample complexity of the proposed learning method in the (weighted) Sobolev space up to second-order derivatives. To our best knowledge, these are the first theoretical results on model-free learning for PDEs with an error analysis in second-order Sobolev space for fully nonlinear PDEs. Our theoretical analysis is built on several standard assumptions on PDEs and statistical learning theory for DNNs. On the one hand, we require a contraction property of the solution operator to an associated linear PDE. This property is typical to ensure the existence and uniqueness of the solution to the original fully nonlinear PDE, and is often necessary for most value iteration algorithms. On the other hand, we exclude the optimization error from our analysis and assume that the class of neural networks is sufficiently smooth and regular. Under these assumptions, we obtain an error bound characterizing the trade-offs among approximation error, sampling error, and ZOD bias. The overall error consists of an exponentially decaying initialization error (in terms of the number of approximate value iterations) and a geometrically weighted accumulation of one-step learning errors (Theorem \ref{thm: global_convergence}). When the $d$-dimensional state process is observed on a uniform time grid with $N$ grid points and the sample size is $B$, our results show that, besides the neural-network optimization and approximation errors, choosing the ZOD bandwidth optimally yields a one-step high-probability error bound of order $O((N^{-1}+\mathcal P(d)B^{-1/2})^{1/4})$ under the weak simulator with one-point ZOD estimators, and of order $O((N^{-1}+\mathcal P(d)B^{-1/2})^{1/2})$ under the strong simulator with multi-point ZOD estimators, up to logarithmic factors, where $\mathcal{P}(d)$ is a constant that depends polynomially on the dimension $d$ (Theorem \ref{thm: one step error bound}). We also experimentally demonstrate accurate approximation of both solutions and derivatives in moderate- and high-dimensional problems.
\end{itemize}

\subsection*{Related Literature}
As a key ingredient of our method, ZOD estimators have been widely used in (black-box) stochastic optimization and reinforcement learning, where gradient information is unavailable; see, e.g., \citet{nesterov2017Random,salimans2017evolution} among many others, but not yet in the context of PDE solvers to our best knowledge.\footnote{ZOD has been used to replace auto-differentiation. For instance, \citet{he2023learning,shi2024stochastic,hu2025bias} use ZOD to approximate the differential operators of a neural network in PINN. The role of ZOD is fundamentally different in our approach.}
% Zeroth-order optimization methods have been extensively studied in stochastic optimization and reinforcement learning, where gradient information is unavailable or expensive to compute. In particular, policy optimization methods based on random perturbations, such as evolution strategies \citep{salimans2017evolution} and random search algorithms \citep{mania2018simple}, construct gradient estimators using function evaluations at perturbed parameters. These approaches are closely related to classical gradient-free optimization methods based on Gaussian smoothing \citep{nesterov2017Random,duchi2015optimal}.
% %and bandit feedback \citep{flaxman2005online}.
Due to the unique context of solving parabolic PDEs, we rely on the celebrated Feynman--Kac formula to obtain a stochastic representation of the function value. Moreover, our ZOD estimator differs from the traditional one in optimization in two aspects. First, at every iteration, we query the simulator to generate a stochastic trajectory with the given initial state, which is then used to compute a functional value using the current function value, gradient and Hessian. Second, the ZOD estimator is not used for optimization directly; instead, it is used as a target to train other networks. Therefore, the bias--variance tradeoff of ZOD and its impact on the entire iterative procedure are both new problems in our context.

Another key ingredient of our method is the approximate value iteration, which is a classical approach adopted in reinforcement learning to solve the Bellman equations or HJB equations in continuous-time optimal controls, a well-known class of nonlinear equations; see, e.g., \citet{bertsekas1996neuro,sutton1998reinforcement} for discrete-time Bellman equations and \citet{yong1999stochastic,fleming2006controlled,pham2009continuous} for continuous-time HJB equations. The core idea for the approximate value iteration is to reduce a fully nonlinear equation into successive linear equations. At every step, it is sufficient to solve the linear equation approximately, and the total error would be accumulated and controlled under suitable regularity conditions. We borrow this idea and solve a linear PDE at each iteration. Moreover, this linear PDE can be interpreted as a Feynman--Kac (Bellman) equation for policy evaluation in reinforcement learning and solvable in a model-free way \citep{jia2022policy}. However, \citet{jia2022policy} only show the consistency of the value function, not its derivatives. In the current PDE setting, learning derivatives is equally crucial for the iteration to converge to the true solution because the nonlinear source term of the PDE depends explicitly on the gradient and Hessian.

%\sout{The deficiency of the conventional learning-then-differentiating paradigm in deep learning is well known in the context of least-square regression},
In supervised regression, recovering derivatives by automatic differentiation of a network trained only on function values is known to be statistically problematic,
despite the availability of  the well-known universal approximation theorem of neural networks that can be employed  to approximate derivatives \citep{hornik1990universal,siegel2023optimal}. Using the typical mean-squared error as the loss function, the theoretical guarantee is often on the function value (e.g., \citealt{barron1994approximation}). To discipline the behavior of the learned derivatives, a commonly used trick is to regularize and penalize large derivatives of the neural network; see, e.g., \citet{kim2004smoothing,fischer2020sobolev} among  many others. One can then establish statistical errors in the Sobolev space. By contrast, we go down a different path due to our access to a simulator and aim to solve a PDE rather than simply fitting a function to a given, exogenous data set. Indeed, we can simulate stochastic trajectories {\it strategically} to {\it create} training data. In other words, in our approach, data sets are also endogenously designed to learn a function in an ``optimal" way in the Sobolev space.

Therefore, our approach, representing-then-learning, should be regarded as a (black-box/model-free) simulation-based learning method, and the first ``representing" stage is essentially a sampling stage. From this perspective, our approach is comparable to the sensitivity analysis associated with mathematical expectations in the simulation literature, such as the likelihood ratio method (cf. \citealt{glasserman2004monte}) and the Malliavin calculus approach (cf. \citealt{ma2002representation}). In a similar context of deep learning for solving PDEs, \citet{han26} propose the method of deep Picard iteration (DPI), where neural networks are trained using Monte Carlo estimates of the solution and its derivatives. Picard iteration coincides with the approximate value iteration in our context. The main difference between DPI (and its extensions) and our approach is that DPI employs the exact stochastic representations for derivatives based on Malliavin calculus (Bismut–Elworthy–Li formulas; cf. \citealt{banos2018bismut}), which is unbiased but requires the knowledge of the coefficients of the differential operator and their derivatives, and thus operates in a model-based setting.
Another line of work is the differential learning framework as in \citet{lef23}, which proposes to jointly learn the solution and its derivatives for fully nonlinear PDEs with convex Hamiltonians. Their approach leverages stochastic control representations and incorporates derivative information through Malliavin-type sensitivities. By contrast, our approach is applicable in a model-free context at the cost of inevitably incurring more bias. More importantly, we have theoretically justified the convergence of our resulting algorithm, showing that the bias introduced by ZOD can be controlled. %While this improves the accuracy of derivative estimation, it also requires model knowledge and structural assumptions, such as convexity of the Hamiltonian.

This paper also contributes broadly to the literature on solving PDEs by machine learning. In addition to the above-mentioned methods such as the deep Galerkin method \citep{siri18}, PINNs \citep{raissi19}, deep BSDEs \citep{han17}, DBDP \citep{hure20}, and deep Picard iteration \citep{han26}, there are also specifically designed methods for special classes of PDEs, such as deep Ritz methods for variational problems \citep{yu2018deep,lu2021priori}, actor--critic methods for HJB equations
\citep{zhou2021actor,jia2022pg,zhu2025optimal}, and the high-order scheme for policy evaluation \citep{mou2024bellman}. However, few of them analyze statistical errors and sample complexity in the presence of neural network-based function approximation. \citet{mou2025statistical} analyzes errors in the Sobolev space of a least-squares temporal-difference algorithm arising from continuous-time policy evaluation. \citet{lumachine} analyze two particular algorithms to solve linear PDEs. Little is known about fully nonlinear PDEs, especially in a black-box setting. The present paper not only develops a new deep learning method that is applicable to a general class of nonlinear unknown PDEs, but also theoretically establishes the sample complexity.

%The above methods highlight a fundamental challenge: accurate estimation of derivatives is essential for fully nonlinear PDEs, yet existing approaches rely either on automatic differentiation or on model-based stochastic representations. Residual-based methods (DGM, PINNs) depend on automatic differentiation, which does not guarantee accurate approximation of derivatives of the true solution and may lead to instability in high dimensions. Probabilistic methods based on Malliavin calculus (cf. \citealt{glasserman2004monte,ma2002representation}) require explicit knowledge of the model, which is unavailable in many applications involving black-box simulators. This limitation is particularly critical in model-free settings, such as reinforcement learning or simulation-based models, where only trajectory data or function evaluations are accessible.

The remainder of the paper is organized as follows. In Section 2, we introduce the problem formulation, put forward the two simulators, and present the approximate value iteration framework for fully nonlinear PDEs. Section 3 develops the probabilistic representation of derivatives based on zeroth-order estimators and introduces the representing-then-learning paradigm. In Section 4, we develop  the proposed deep learning algorithm, including the loss construction and implementation details. Section 5 provides the theoretical analysis, establishing convergence of the method and quantifying the impact of discretization, approximation, sampling, and perturbation errors. Section 6 presents numerical experiments that validate the performance of the proposed approach in both low- and high-dimensional settings. Section \ref{conclusions} concludes. Supplementary materials and proofs of statements are placed in the Appendix.

\section{Problem Description and Preliminaries}
In this paper, we study how to effectively solve a class of nonlinear parabolic PDEs:
\begin{equation}\label{eq:pde}
    \begin{cases}
\partial_t v + \Lc v + f( x,v, \nabla v, \nabla^2 v) = 0, \quad (t,x) \in [0,T)\times\R^d, \\
v(T,x) = g(x), \quad x \in \R^d,
\end{cases}
\end{equation}
where $\nabla v \in \mathbb R^d$ and $\nabla^2 v\in \mathbb R^{d\times d}$ are respectively the gradient and Hessian of $v$ with respect to the spatial variables, and the second order differential operator $ \Lc$ is
$$ \Lc \varphi := \sum_{i,j=1}^d a_{ij}(x) \partial_{ij} \varphi + \sum_{i=1}^d b_i(x) \partial_i \varphi, $$
with $\partial_t,\partial_i$ being respectively the partial derivative in the time variable $t$ and the $i$-th entry of the spatial variable $x_i$, and $\partial_{ij}$  the corresponding second-order derivative. Let $a(x)\equiv (a_{ij}(x))\in \mathbb S^d_+$, assumed to be a positive semidefinite matrix.

The operator $\mathcal L$ can therefore be viewed as the infinitesimal generator of a time-homogeneous diffusion process driven by a $d$-dimensional Brownian motion $W$, satisfying the following stochastic differential equation (SDE):
\begin{equation}\label{eq:sde}
     \dd X_s = b(X_s) \dd s + \sqrt{2} \sigma(X_s) \dd W_s,
\end{equation}
where $\sigma(x)\sigma(x)^\top = a(x)$.\footnote{For ease of presentation, in this paper we consider only the case where $\mathcal L$ and $f$ do not depend on time, i.e. time-homogeneous problems. All the conclusions can be readily extended to the time inhomogeneous case.} We further denote by $X^{t,x}$  the solution to the SDE with initial condition $X_t = x$.

We first reiterate  the following important consideration and goal in our study: %setups about our problems and explain our goals:
\begin{enumerate}[(1)]
\item The dimension of the spatial variable $d$ can be potentially high, and we want to avoid the curse of dimensionality arising from the conventional grid-based numerical methods (such as finite-difference and finite-element methods).
\item The PDE can be potentially described in a black-box manner and we aim to develop model-free, data-driven methods. A black-box environment means (i) the functions $g,f$ are black-box functions that can only output their values if inputs are queried; and (ii) the operator $\mathcal L$ is not known and only its associated processes/trajectories $\{X_s^{t,x}\}_{t\leq s\leq T}$ can be observed upon querying a \textit{simulator}. Thus, ``data" here means the trajectories that are returned upon a query.
\end{enumerate}

We now formalize the ``simulators" considered in this paper. Note that any simulator defined below remains unknown to us, being a device outputting data upon a query, and we do not attempt to learn its underlying structure. We consider two settings.
\begin{cond}
\label{cond:weak simulator}
A weak simulator returns a trajectory $\{X_s^{t,x}\}_{t\leq s\leq T}$ for each query with input $(t,x)$, such that
\begin{enumerate}[(i)]
    \item $\{X_s^{t,x}\}_{t\leq s\leq T}$ is the weak solution to \eqref{eq:sde};
    \item the random sources (Brownian motions) in any two different queries are independent of each other.
\end{enumerate}
\end{cond}

\begin{cond}
\label{cond:strong simulator}
A strong simulator returns a collection of trajectories $\left\{\{X_s^{t,x_k}\}_{t\leq s\leq T},\;k=1,\cdots,K\right\}$ for each query whose input is a finite set of initial states $\{(t,x_1),\cdots,(t,x_K)\}$, such that
\begin{enumerate}[(i)]
\item upon each query, every trajectory of $\{X_s^{t,x_k}\}_{t\leq s\leq T}$ is the strong solution to \eqref{eq:sde} with respect to the same Brownian motion $\{W_s\}_{t\leq s\leq T}$ with the initial states $(t,x_k)$;
\item the random sources (Brownian motions) in any two different queries are independent of each other.
\end{enumerate}
\end{cond}

Evidently, a weak simulator is implied by a strong simulator, which is further implied by a PDE with known coefficients. Neither converse implication, however, is true.
The notions of weak and strong simulators coincide with the classical definitions of weak and strong solutions to SDEs and capture the subtle distinction between them: a strong solution is defined with respect to a fixed Brownian motion, whereas in a weak solution the Brownian motion is a part of the solution. The weak simulator is the typical setting in which each query produces an i.i.d. sample path. By contrast, the strong simulator not only generates multiple trajectories per query, but also keeps the underlying randomness fixed within that query, much like fixing a ``random seed". Mathematically, querying a strong simulator with a set of initial states $\{(t,x_1),\cdots,(t,x_K)\}$ means that $K$ trajectories $\{ X^{t,x_k}_s \}_{t\leq s\leq T}$ are generated satisfying
\[ X_{\tau}^{t,x_k} = x_k + \int_t^{\tau} b(X_s^{t,x_k})\dd s + \int_t^{\tau} \sqrt{2}\sigma(X_s^{t,x_k})\dd W_s,\] for all $\tau\in [t,T]$ and all $k=1,\cdots,K$, almost surely, where the same Brownian motion $W$ drives all the $K$ trajectories (but $W$ is not the output of the simulator and hence is ``hidden"). Upon another query, another hidden independent Brownian motion will be used to generate these trajectories.

These precise definitions of the simulators also give precise meanings to the terms ``model-free" and ``data-driven", and it is possible to extend them to settings other than PDEs.
Our approach and results in this paper involve both settings, but the theoretical guarantees differ, depending on which simulator is available. In short, access to a strong simulator leads to substantially sharper results than access only to a weak simulator because the former facilitates variance reduction; see Proposition \ref{prop: zod variance reduction} and Theorem \ref{thm: one step error bound}.

Next, we discuss a few important building blocks of our approach and highlight the challenges of the existing approaches. As they are both well-known in the literature, we  omit to specify needed regularity conditions to avoid unnecessary technicality in the discussion. Rigorous conditions and analysis will be presented in Section \ref{sec:convergence analysis}.
\subsection{Value Iteration of the Nonlinear PDEs}
We consider the value iteration for solving the nonlinear PDE \eqref{eq:pde}: Define the solution mapping of a linear parabolic equation $\mathcal S: v \mapsto u$, such that $u$ solves
\begin{equation}\label{eq:pde linearize 0}
    \begin{cases}
\partial_t u + \Lc u +  f( x,v, \nabla v, \nabla^2 v) = 0, \quad (t,x) \in [0,T)\times\R^d, \\
u(T,x) = g(x), \quad x \in \R^d.
\end{cases}
\end{equation}

By assuming suitable regularity conditions on $\mathcal L$ and $f,g$, $\mathcal S$ is a well-defined mapping from a suitable space to itself. For any given initial value $v_0$, we can therefore define the value iteration sequence by $v_1 = \mathcal S v_0,\cdots,v_{n+1} = \mathcal S v_n,\cdots$. Based on the Feynman--Kac formula (cf. \citealt{karatzas2014brownian}), the solution mapping $\mathcal S$ admits a probabilistic representation:
\begin{equation}\label{eq:feymann-kac}
\mathcal S v(t,x) =     u(t,x) =  \E\left[ g(X_T^{t,x}) + \int_t^T f\left(  X_s^{t,x},v(s,X_s^{t,x}),\nabla v(s,X_s^{t,x}),\nabla^2 v(s,X_s^{t,x}) \right)\dd s\right] .
\end{equation}

Hence, correspondingly, the value iteration can be represented by
\begin{equation} \label{eq:iter_fk}
\begin{aligned}
    v_{n+1}(t,x) &= \E\left[ g(X_T^{t,x}) + \int_t^T f\left(  X_s^{t,x},v_n(s,X_s^{t,x}),\nabla v_n(s,X_s^{t,x}),\nabla^2 v_n(s,X_s^{t,x}) \right)\dd s\right].
    \end{aligned}
\end{equation}

We make three important observations. First, value iteration reduces the difficulty of solving a {\it nonlinear} PDE to solving {\it linear} PDEs successively and, thus, is commonly employed as the first attempt in many numerical procedures. It is also widely used to establish the well-posedness of nonlinear PDEs, and is referred to as Picard iteration in the PDE literature (e.g., \citealt{han26}). In particular, if we assume that the solution mapping $\mathcal S$ is a contraction under a suitable norm, then this iteration converges at an exponential rate. Such contraction properties are dictated by those of $\mathcal L$ and $f,g$, and we will present a sufficient condition in Appendix \ref{appendix: sufficient condition for contraction} for the contraction property.

Second, besides value iteration, there are alternative ways to solve nonlinear PDEs by iteratively solving linearized PDEs, e.g., Newton's method, which expands the nonlinear term $f$ into its linear approximation.
Our approach in this paper is a zeroth-order method: it uses only knowledge of $f$ itself, rather than its first-order derivative (which is unknown in the black-box/model-free setting), but can be generalized to Newton's method if such a form is known.

Third, the value iteration \eqref{eq:iter_fk} provides a theoretical mechanism to solve the nonlinear PDE by solving a sequence of linear problems. However, numerical challenges remain, and numerically solving linear (black-box) PDEs in high dimensions is already difficult. The probabilistic representation of the solution mapping $\mathcal S$ transforms the problem of solving PDEs into evaluating expectations over stochastic trajectories which can in turn be computed using Monte Carlo methods, thus avoiding using grid points in the conventional numerical PDE methods. Such a probabilistic representation seems to be the only way to remedy the curse of dimensionality in solving PDEs. Moreover, the probabilistic representation \eqref{eq:feymann-kac} only requires the function {\it values} of $g,f$ along the trajectories of $X_t$ and not the precise specification of $\mathcal L$ nor the exact {\it functional forms} of $g,f$.  Hence it is applicable to a black-box environment.\footnote{This problem also arises in the context of policy evaluation in reinforcement learning (see, e.g., \citealt{jia2022policy}).}

\subsection{Is Deep Neural Network a Savior For All?}\label{sec: necessity}
To overcome the curse of dimensionality, deep neural networks (DNNs) have been widely adopted in recent algorithms based on probabilistic representations such as \eqref{eq:feymann-kac} to approximate solutions to PDEs.
DNNs have been shown, both theoretically and empirically, to approximate high-dimensional functions efficiently without suffering from an exponential dependence on dimension;
see e.g. \cite{han18}, \cite{bec23}, \cite{germain23}.
 In our context, the value iteration requires successively solving linear PDEs and, hence, DNNs can be used to approximate each solution $v_n$. Such an approach has been explored  by \cite{jia2022policy} and more broadly in the reinforcement learning literature for policy evaluation.

Does this, then, mean that DNNs provide a universal remedy for all these problems? The answer is negative. In value iteration, the solution obtained at the previous step is used as an input to define the next linear PDE. Therefore, it is crucial to ensure that the one-step error does not propagate and amplify over iterations. In particular, for a fully nonlinear PDE like \eqref{eq:pde linearize 0}, the next linear PDE in the value iteration \eqref{eq:iter_fk} depends on the previous solution via its
spatial {\it derivatives} up to the second order (i.e., $f$ depends on $\nabla v_n$ and $\nabla^2 v_n$).
Consequently, the one-step approximation error must be controlled not only in the current function iterate, but also in its first- and second-order derivatives, which imposes significantly more stringent requirements on the approximation scheme.

%On the one hand, using probabilistic method based on \eqref{eq:feymann-kac} inevitably incur  sampling error. The  original curse of dimensionality on the number of grids and computational complexity now transfers into the problem on the number of random samples and sample complexity. Moreover, the classical universal approximation theorem \cite{cyber89, hornick91} guarantees approximation in function values, but does not ensure accurate approximation of derivatives, especially in high-dimensional settings.

To illustrate the issue, let $V_n$ be the DNN approximation of $v_n$ at step $n$. To compute the target for step $n+1$, we must evaluate $f(\cdot, V_n, \nabla V_n, \nabla^2 V_n)$. A standard approach would be to train $V_{n+1}$ to minimize the $L^2$ distance to the target value implied by \eqref{eq:iter_fk}.
However, convergence in the $L^2$ norm of the value function does not guarantee convergence of its derivatives (i.e., $\|V_n - v^*\|_{L^2} \to 0$ does not imply $\|\nabla V_n - \nabla v^*\|_{L^2} \to 0$), particularly in high-dimensional spaces. This phenomenon is well-known in Sobolev training and non-parametric regression; see e.g. \cite{czarnecki2017sobolev}.

We emphasize that this error is not due to the numerical error in computing derivatives as in, say, the finite-difference method. When using DNNs, derivatives are computed via automatic differentiation (auto-diff) through \textit{backpropagation}, which provides exact derivatives of the
network representation.\footnote{The only exception occurs for non-differentiable functions, e.g.
$v(x) = x^+$ at $x=0$. In practice, this is rarely an issue since neural networks are smooth or almost everywhere differentiable.}
Therefore, if $\nabla V_n$ and $\nabla^2 V_n$ are computed directly via automatic differentiation of an imperfectly trained network $V_n$, the errors in derivatives may be significantly amplified, destabilizing the source term $f$ for the next iteration.
Consequently, learning the value function alone is insufficient, and one has to learn the gradient and Hessian, denoted by $G_n \approx \nabla v_n$ and $H_n \approx \nabla^2 v_n$, and update them {\it alongside}  with $V_n$. %To ensure the stability and convergence of the value iteration, it is essential to explicitly approximate and constrain the derivatives. In this work, we propose to maintain independent network approximations for the gradient and Hessian, denoted by $G_n \approx \nabla v_n$ and $H_n \approx \nabla^2 v_n$, and update them jointly with $V_n$ using a novel zeroth-order derivative estimation scheme.

Existing approaches approximate derivatives using differential learning techniques in a model-based setting, often relying on the simulation of tangent processes associated with the underlying stochastic differential equation; see \cite{glasserman2004monte} and \cite{huge20}.
These methods, however, typically require access to derivatives of the coefficients (e.g., the exact form of $\mathcal L$) and in addition may introduce additional variance and computational complexity.

The above discussions  motivate us to develop alternative approaches that avoid explicit differentiation while maintaining stability over iterations.

\section{Approximate Probabilistic Representation of Derivatives}

In this section, we detail the proposed numerical scheme.
Based on the discussions in Section \ref{sec: necessity}, the core challenge lies in generating accurate and low-variance training targets for networks $(G,H) \approx (\nabla v, \nabla^2 v)$ that can be learned using function values of $v$. In other words, %\sout{we will not follow the ``learning-then-differentiating" routine;}
we will not rely on recovering derivatives from a learned value approximation, whether through automatic differentiation or model-dependent sensitivity formulas;
rather, we need to derive a probabilistic model-free representation of the derivatives $(\nabla v, \nabla^2 v)$, i.e., forms similar to \eqref{eq:feymann-kac}, and then train networks $(G,H)$ to approximate them -- a ``representing-then-learning" paradigm.

To achieve this, we rely on the ``\textit{zeroth-order derivative}" (ZOD) estimators as the fundamental building block of our method. Since ZOD has not yet been utilized in the context of solving PDEs with DNNs, we first review its properties for the reader's convenience.

\subsection{Classical Zeroth-Order Derivative Estimators}

%We first introduce the general ZOD estimator and its variance reduction counterpart, and then we study the ZOD estimator under the SDE setting.

The basic idea of ZOD is to use the function values at randomly perturbed points, $u(t,x+\epsilon Z)$, to approximate the derivative $\nabla u$ and even higher-order derivatives such as $\nabla^2 u$, where $\epsilon > 0$ is the bandwidth of the perturbation and $Z$ follows a standard, sphere-symmetric distribution. In our case, without loss of generality, we take $Z \sim \mathcal N(0,I)$, a $d$-dimensional standard normal distribution. The following Lemmas \ref{lemma:zod classical} to \ref{lemma: zod multi-point hessian} are standard results under slightly different regularity conditions; see, e.g., \citet{nesterov2017Random,balasubramanian2022ZerothOrder}. For completeness and for the reader's convenience, we provide their proofs in the appendix. These lemmas construct several well-known ZOD estimators and present their properties.

\begin{lemma}
\label{lemma:zod classical}
Given function $u(\cdot,\cdot)$, $ (t,x)\in [0,T]\times \R^d$, and $ Z \sim \mathcal{N}(0,I_d)$. If $ u(t,\cdot) $ is three times continuously differentiable and its derivatives satisfy the polynomial growth condition,
then $\frac{Z}{\epsilon} u(t,x+\epsilon Z)$ is a biased estimator for $\nabla u(t,x)$, and its bias and variance are upper bounded by
\[ \left| \E\left[ \frac{Z}{\epsilon} u(t,x+\epsilon Z)  \right] -\nabla u(t,x) \right| \leq C\epsilon^2,\ \operatorname{Var}\left[\frac{Z}{\epsilon} u(t,x+\epsilon Z) \right] \leq C\frac{1}{\epsilon^2} .  \]
Furthermore, if $ u(t,\cdot)$ is four times continuously differentiable and its derivatives satisfy the polynomial growth condition,
then $ \frac{ZZ^{\top}- I_d}{\epsilon^2}u(t,x+\epsilon Z)$ is a biased estimator for $\nabla^2 u(t,x)$, and its bias and variance are upper bounded by
\begin{equation}
    \left|\E\left[\frac{ZZ^{\top}- I_d}{\epsilon^2}u(t,x+\epsilon Z) \right]- \nabla^2 u(t,x)\right| \le C\epsilon^2,\  \operatorname{Var}\left[ \frac{ZZ^{\top}- I_d}{\epsilon^2}u(t,x+\epsilon Z)\right] \le C\frac{1}{\epsilon^4}
\end{equation}
\end{lemma}

\paragraph{Variance Reduction of Multi-Point ZOD Estimators}
As shown in Lemma \ref{lemma:zod classical}, the na\"ive one-point estimators suffer from a critical drawback: their variances diverge as the perturbation magnitude, $\epsilon$, approaches zero. This means that to get an accurate derivative estimate, $\epsilon$ needs to be small to reduce the bias but the number of samples need to increase dramatically to reduce the variance. This issue can often be mitigated by the symmetric (or central) differencing technique, known as the multi-point zeroth-order estimators.

We begin by considering the estimation of the gradient, $\nabla u(t,x)$. A one-point estimator, as demonstrated in Lemma \ref{lemma:zod classical}, has variances that scale as $O(\epsilon^{-2})$, whereas a two-point estimator exhibits  {\it bounded} variances, as stipulated in Lemma \ref{lemma: zod multi-point gradient}:
\begin{lemma}
\label{lemma: zod multi-point gradient}
Given function $u(\cdot,\cdot)$, $ (t,x)\in [0,T]\times \R^d$, and $ Z \sim \mathcal{N}(0,I_d)$. Assume $u(t, \cdot)$ is three times continuously differentiable and its derivatives satisfy the polynomial growth condition. Define
\begin{equation}
\label{eq:symmetric_gradient_estimator}
\hat{g}(t, x; \epsilon) \defeq \frac{u(t, x+\epsilon Z) - u(t, x-\epsilon Z)}{2\epsilon} Z.
\end{equation}
Then
\begin{align}
\left| \E\left[ \hat{g}(t, x; \epsilon) \right] - \nabla u(t,x) \right| \le C \epsilon^2, \;\;\;
\Var\left[ \hat{g}(t, x; \epsilon) \right]  \le C',
\end{align}
for some constants $C, C'> 0$ that are independent of $\epsilon$.
\end{lemma}

The one-point estimator for Hessian has even larger variances that diverge at a rate of $O(\epsilon^{-4})$ and hence is much more inefficient. However, with symmetric differencing, we can construct a stable estimator with bounded variances.

\begin{lemma}
\label{lemma: zod multi-point hessian}
Let the function $u: [0,T] \times \R^d \to \R$ satisfy the conditions of Lemma~\ref{lemma: zod multi-point gradient}, and further assume $u(t, \cdot)$ is four times continuously differentiable and its derivatives satisfy the polynomial growth condition. Define
\begin{equation}
\label{eq:symmetric_hessian_estimator}
\hat{H}(t, x; \epsilon) \defeq \frac{u(t, x+\epsilon Z) - 2u(t,x) + u(t, x-\epsilon Z)}{2\epsilon^2} (ZZ^\top - I_d).
\end{equation}
Then
\begin{align}
\left| \E\left[ \hat{H}(t, x; \epsilon) \right] - \nabla^2 u(t,x) \right| \le C \epsilon^2, \;\;\; %\label{eq:bias_bound} \\
\Var\left[ \hat{H}(t, x; \epsilon) \right] \le C', %\label{eq:variance_bound}
\end{align}
for some constants $C, C' > 0$ that are independent of $\epsilon$.
\end{lemma}

\subsection{Illustrations of Representing-Then-Learning Paradigm}
\label{sec:illustrate toy}

To illustrate our ``representing-then-learning" paradigm, we first %demonstrate how it can be
use a toy task of regression and contrast it with a conventional routine where the derivative is obtained from differentiating the learned value function. This simple regression task is indeed equivalent to solving a linear PDE, which will be demonstrated numerically  in Subsection \ref{subsec:spiky-diagnostic}.

The regression task is described as follows: suppose $(X,Y)\in \mathbb R^2$ follows a joint distribution that can be written as $\pi(\dd x)\rho(x,\dd y)$, where $\pi(\dd x)$ is the marginal distribution of $X$ and $\rho(x,\dd y)$ is the conditional distribution of $Y|X=x$.
We can generate samples from both $X$ and $Y|X=x$, but do not know the functions $\pi$ and $\rho$.
We are interested in numerically finding the function $f(x)$ and its derivatives $f'(x), f''(x)$ %when $\rho(x,\dd y)$ is unknown but we have samples from this conditional distribution,
where $f(X) = \E[Y|X]$ is the conditional expectation.

\paragraph{Learning-then-Differentiating} Conventionally, one can first generate samples $X_i \sim \pi(\cdot)$ and $Y_i \sim \rho(X_i,\cdot)$, and then fit a DNN regression by
\[ V = \arg\min_{\varphi\in \mathcal H} \frac{1}{n}\sum_{i=1}^n \left( Y_i - \varphi(X_i) \right)^2, \]
where $\mathcal H$ stands for the functional class of the chosen DNN family. The resulting function approximation $V$ is well-known to satisfy performance guarantee (ignoring the optimization error).
However, differentiating (equivalent to auto-diff) $V$ to obtain $V',V''$ does not have any theoretical guarantee.\footnote{One possible mitigation to this is to consider the regularized regression by
$\min_{\varphi\in \mathcal H} \frac{1}{n}\sum_{i=1}^n \left( Y_i - \varphi(X_i) \right)^2 + \lambda |\varphi''(X_i)|^2 $, with a tuning parameter $\lambda > 0$ that induces smoother function approximation; see e.g. \cite{wahba1990spline}.
}

\paragraph{Representing-then-Learning} Our approach is based on obtaining a sample approximation for the derivatives directly via ZOD estimators. Let us illustrate the one-point ZOD for estimating derivatives first. Besides the samples $(X_i,Y_i)$, we further sample $Z_i \sim \mathcal N(0,1)$ and sample $Y_i^{(+)} \sim \rho(X_i + \epsilon Z_i,\cdot)$. Now consider the newly constructed sample $\frac{Z_i}{\epsilon} Y_i^{(+)}$. It follows from Lemma \ref{lemma:zod classical} that
\[ \E\left[ \frac{Z_i}{\epsilon} Y_i^{(+)} \Big| X_i \right] =  \E\left[ \E\left[ \frac{Z_i}{\epsilon} Y_i^{(+)} \Big|X_i,Z_i\right] \Big| X_i  \right]  = \E\left[ \frac{Z_i}{\epsilon} f(X_i+\epsilon Z_i) \Big|X_i \right] \approx f'(X_i) + \text{bias} . \]
Thus, $\frac{Z_1}{\epsilon} Y_1^{(+)},\cdots,\frac{Z_n}{\epsilon} Y_n^{(+)}$ are $n$ independent, biased samples of $f'(X_1),\cdots,f'(X_n)$ respectively with heterogeneous bias and standard errors. Then, we learn a function $G$ by
\[ G = \arg\min_{\varphi\in \mathcal H} \frac{1}{n}\sum_{i=1}^n \left(  \frac{Z_i}{\epsilon} Y_i^{(+)} - \varphi(X_i) \right)^2 .\]
Similarly, we can construct samples, $\frac{Z_i^2 - 1}{\epsilon^2} Y_i^{(+)} $, $i=1,\cdots,n$,  for the second-order derivative because
\[
\begin{aligned}\E\left[ \frac{Z_i^2- 1}{\epsilon^2} Y_i^{(+)} \Big| X_i \right] =  \E\left[ \E\left[ \frac{Z_i^2-1}{\epsilon^2} Y_i^{(+)} \Big| X_i, Z_i \right] \Big|X_i \right] &= \E\left[ \frac{Z_i^2-1}{\epsilon^2} f(X_i+\epsilon Z_i) \Big|X_i \right]\\
&\approx f''(X_i) + \text{bias} .
\end{aligned}
\]
Hence, we can learn a function $H$ by
\[  H = \arg\min_{\varphi\in \mathcal H} \frac{1}{n}\sum_{i=1}^n \left(  \frac{Z_i^2-1}{\epsilon^2} Y_i^{(+)} - \varphi(X_i) \right)^2 .\]
The learned DNNs $G,H$ will then be reasonable approximations to $f',f''$ respectively, subject to the optimization error (controlled by the optimization solvers), approximation error (controlled by the functional class $\mathcal H$), sampling error (controlled by $n$ and $\epsilon$), and the bias (controlled by $\epsilon$).

Furthermore, the multi-point ZOD estimators can be constructed as follows: We sample $Y_i^{(\pm)} \sim \rho(X_i \pm \epsilon Z_i,\cdot)$, and consider $\frac{Y_i^{(+)} - Y_i^{(-)}}{2\epsilon}Z_i$ and $\frac{Y_i^{(+)} - 2Y_i + Y_i^{(-)}}{2\epsilon^2}\left( Z_i^2-1\right)$ for the biased samples of $f'(X_i),f''(X_i)$ respectively. In the following, for ease of presentation we only describe the construction under the one-point ZOD, but multi-point ZOD can be similarly constructed.

\subsection{Estimators for Solutions to Linear PDEs}

We now apply the ZOD framework to our specific problem of solving nonlinear PDEs. Recall that in the value iteration, the target function at the $(n+1)$-th step is given by the conditional expectation:
\begin{equation}\label{eq: value function}
    \hat{V}_{n+1}(t,x) = \E \left[ R^{t,x}(\mathbf{U}_n) \right],
\end{equation}
where $\mathbf{U}_n = (V_n, G_n, H_n)$ represents the networks from the previous iteration. The stochastic reward functional is defined as:
\begin{equation}\label{eq: def reward}
    R^{t,x}(\mathbf{U}_n) := g(X_T^{t,x}) + \int_t^T f\left(  X^{t,x}_s, V_n(s, X^{t,x}_s), G_n(s, X^{t,x}_s), H_n(s, X_s^{t,x})\right) \dd s.
\end{equation}
Since $\hat{V}_{n+1}$ is not available analytically, we cannot directly apply the ZOD estimators to the function $\hat{V}_{n+1}$ itself. However, we have access to the random variable $R^{t,x}$, which is an {\it unbiased} estimator of $\hat{V}_{n+1}(t,x)$.

This motivates the {\it One-Point ZOD Estimator}. Instead of estimating $\hat{V}_{n+1}$ and then employing ZOD, we directly use the perturbed realization of the reward $R^{t, x+\epsilon Z}$ to estimate the spatial derivatives.
%Our approach to enhance the performance is to make use of the flexibility of the simulation data to construct a one-point ZOD estimator.
Taking the first-order derivative as an example, recall that $\frac{Z}{\epsilon}\hat V_{n+1}(t,x+\epsilon Z)$ is a biased estimator for $\nabla \hat V_{n+1}(t,x)$, and $R^{t,x+\epsilon Z}(\mathbf{U}_n)$ is an unbiased estimator of $\hat V_{n+1}(t,x+\epsilon Z)$. Thus, we are motivated to take $\frac{Z}{\epsilon}R^{t,x+\epsilon Z}(\mathbf{U}_n)$ as a one-point estimator for $\nabla \hat V_{n+1}(t,x)$, without relying on an intermediary estimate of $\hat V_{n+1}$. In a similar fashion, a one-point estimator for the second-order derivative can also be constructed. The following proposition presents the error bounds of the one-point ZOD estimators in $ L_2$ norm.

For integers $l,k\ge0$, $C_b^{l,k}([0,T]\times\mathbb R^d)$ denotes the class of functions that are $ l$ times continuously differentiable in $t$ and $k$ times continuously differentiable in $x$, and whose derivatives are uniformly bounded, and $C_p^{l,k}([0,T]\times\mathbb R^d)$ denotes the corresponding class in which these derivatives have at most polynomial growth in $x$. Similarly, $C_b^k(\mathbb R^d)$ and $C_p^k(\mathbb R^d)$ denote the time-independent counterparts. These notations are understood componentwise for vector- or matrix-valued functions.

We need the following standing assumption.

\begin{assumption}\label{assum: sde coef}
The coefficients $b$ and $\sigma$ are sufficiently regular so that the SDE \eqref{eq:sde} admits a unique strong solution. Moreover, they satisfy the linear growth and bounded condition, i.e., there are constants $ L_b$ and $ \sigma_0$ such that
    \begin{equation}
        \begin{aligned}
            |b(t,x)| \le L_b(1+|x|), \quad |\sigma(t,x)| \le \sigma_0, \quad \text{for all}\ \ (t,x)\in [0,T]\times\R^d.
        \end{aligned}
    \end{equation}
\end{assumption}

Assumption \ref{assum: sde coef} is a standard condition for SDEs, and we require the existence and uniqueness of a strong solution since we consider both the weak and strong simulators. Furthermore, we impose regularity conditions on the drift and volatility coefficient so that the usual moment estimates for the solution of SDE hold, which is needed in the subsequent analysis. In particular, we require stronger boundedness condition on the volatility coefficient for us to establish the desired statistical error guarantee later; otherwise, the underlying SDEs may be too volatile, losing the property required for the empirical loss functions used in the training to concentrate around its population counterpart.

Under Assumption \ref{assum: sde coef}, we present the property of ZOD estimators under our context.
\begin{proposition}
\label{prop:zod one-point}
Suppose Assumption \ref{assum: sde coef} holds. Assume also that the coefficients $ b,\sigma, g,f,V_n,G_n,H_n$ are smooth enough so that $ \hat V_{n+1}\in C_p^{0,3}([0,T]\times\R^d)$ and  $\mathbf{U}_n,f,g$ all have polynomial growth in $x$, uniformly in $t\in[0,T]$. Then $\frac{Z}{\epsilon}R^{t,X_t+\epsilon Z}(\mathbf{U}_n)$ is a biased estimator for $\nabla \hat V_{n+1}(t,X_t)$ conditioning on $ X_t$, and there is a constant $ C>0$ independent of $ \epsilon$ such that its bias and variance are upper bounded by
\[ \E\left[\int_0^T\left| \E\left[ \frac{Z}{\epsilon} R^{t,X_t+\epsilon Z}(\mathbf{U}_n)\Big| X_t  \right] -\nabla \hat V_{n+1}(t,X_t) \right|^2\dd t\right] \leq C\epsilon^4,\ \left|\operatorname{Var}\left[\frac{Z}{\epsilon} R^{t,X_t+\epsilon Z}(\mathbf{U}_n) \right]\right| \leq C\frac{1}{\epsilon^2} ,  \]
Moreover, if $\hat V_{n+1} \in C_p^{0,4}([0,T]\times\R^d)$, then $ \frac{ZZ^{\top}- I_d}{\epsilon^2}R^{t,X_t+\epsilon Z}(\mathbf{U}_n)$ is a biased estimator for $ \nabla^2 \hat V_{n+1}(t,X_t)$ conditioning on $ X_t$, and there is a constant $ C>0$ independent of $ \epsilon$ such that its bias and variance are upper bounded by
\begin{equation}
\begin{aligned}
\E\left[\int_0^T\left|\E\left[\frac{ZZ^{\top}- I_d}{\epsilon^2}R^{t,X_t+\epsilon Z}(\mathbf{U}_n) \Big |X_t\right]- \nabla^2 \hat V_{n+1}(t,X_t)\right|^2 \dd t \right] &\le C\epsilon^4, \\
    \left|\operatorname{Var}\left[ \frac{ZZ^{\top}- I_d}{\epsilon^2}R^{t,X_t+\epsilon Z}(\mathbf{U}_n)\right]\right| &\le C\frac{1}{\epsilon^4}.
    \end{aligned}
\end{equation}
\end{proposition}

Proposition \ref{prop:zod one-point} describes the one-point ZOD estimators for the derivatives of the solution to a linear PDE and confirms that their biases are small and controlled by $\epsilon^2$. A significant limitation is that the variances of the estimators are generally  unbounded, and in fact have the orders of $\epsilon^{-2}$ and $\epsilon^{-4}$ for the gradient and Hessian respectively, similarly to the classical ZOD in Lemma \ref{lemma:zod classical}. We refer to them as ``ZOD-1" in our subsequent numerical experiment, and will demonstrate that they usually produce worse results in learning derivatives.

\paragraph{Variance Reduction with Strong Simulators}
Theoretically, it is natural to investigate the symmetric versions (analogous to Lemma \ref{lemma: zod multi-point gradient} and Lemma \ref{lemma: zod multi-point hessian}) of ZOD estimators for the gradient and Hessian of the solution to this linear PDE and inquire whether  this variance reduction technique is helpful in the current PDE setting as well. It turns out that the answer depends on which type of simulators is accessible to us, i.e., when simulating multiple  trajectories starting from {\it different} initial states $x$, $x+\epsilon Z$, and $x-\epsilon Z$, whether we can control the noises in the environment in a way that the corresponding rewards $R^{t,x},R^{t,x+\epsilon Z},R^{t,x-\epsilon Z}$ are highly correlated and therefore cancel the noises off to great extent.\footnote{This idea resembles the classical simulation techniques of \textit{control variates} and \textit{antithetic variates} (cf. \citealt[Chapter 4]{glasserman2004monte}) to use correlated samples for variance  reduction. } With a weak simulator, we need to query the simulator {\it separately}, one at a time,  to generate the trajectories $X^{t,x},X^{t,x+\epsilon Z},X^{t,x-\epsilon Z}$. Because the Brownian motions in the underlying stochastic processes are independent, these trajectories, as well as the associated rewards $R^{t,x},R^{t,x+\epsilon Z},R^{t,x-\epsilon Z}$, are conditionally independent (conditioned on the realized $Z$), and hence we cannot expect the benefit of variance reduction. However, with a strong simulator, we are presented with trajectories of $X^{t,x},X^{t,x+\epsilon Z},X^{t,x-\epsilon Z}$ driven by the {\it same} Brownian path upon one single query , and $R^{t,x},R^{t,x+\epsilon Z},R^{t,x-\epsilon Z}$ are expected to be highly correlated, leading to additional noise cancellation in $R^{t,x+\epsilon Z}-R^{t,x-\epsilon Z}$ and $R^{t,x+\epsilon Z}+R^{t,x-\epsilon Z}-2R^{t,x}$. The next Proposition \ref{prop: zod variance reduction} formalizes and indeed validates this idea.

\begin{proposition}\label{prop: zod variance reduction}
Suppose Assumption \ref{assum: sde coef} holds, and $\mathbf{U}_n,f,g$ all have polynomial growth in $x$, uniformly in $t\in[0,T]$. %For approximating $ \nabla \hat V_{n+1}(t,X_t)$ and $\nabla^2 \hat V_{n+1}(t,X_t) $, the multi-point ZOD estimators in our context are
Construct
    \begin{equation}
    \begin{aligned}
        \hat g &:= \frac{Z}{2\epsilon} \left(R^{t,X_t + \epsilon Z}(\mathbf{U}_n) - R^{t,X_t - \epsilon Z}(\mathbf{U}_n)\right),\\
        \hat h &:= \frac{ZZ^{\top}-I_d}{2\epsilon^2}\left(R^{t,X_t + \epsilon Z}(\mathbf{U}_n) + R^{t,X_t - \epsilon Z}(\mathbf{U}_n)  - 2R^{t,X_t}(\mathbf{U}_n)\right).
        \end{aligned}
    \end{equation}
\begin{enumerate}[(a)]
\item With a weak simulator as specified in Condition \ref{cond:weak simulator}, if the coefficients $ b,\sigma, g,f,V_n,G_n,H_n$ are smooth enough so that $ \hat V_{n+1}\in C_p^{0,3}([0,T]\times\R^d)$, then $ \hat g$ is a biased estimator for $ \nabla \hat V_{n+1}(t,X_t)$ conditioning on $ X_t$, and its bias and variance are upper bounded by
\begin{equation}
    \E\left[ \int_0^T \left| \E\left[ \hat g\mid X_t\right] - \nabla \hat V_{n+1}(t,X_t)\right|^2\dd t\right] \le C\epsilon^4, \ \left|\operatorname{Var}[\hat g]\right| \le C\frac{1}{\epsilon^2}.
\end{equation}
Furthermore, if $\hat V_{n+1} \in C_p^{0,4}([0,T]\times\R^d)$, then $ \hat h$ is a biased estimator for $ \nabla^2 \hat V_{n+1}(t,X_t)$ conditioning on $ X_t$, and its bias and variance are upper bounded by
\begin{equation}
\begin{aligned}
    \E\left[\int_0^T\left|\E\left[\hat h \mid X_t\right]- \nabla^2 \hat V_{n+1}(t,X_t)\right|^2 \dd t \right] \le C\epsilon^4, \
    \left|\operatorname{Var}[ \hat h]\right| \le C\frac{1}{\epsilon^4}.
    \end{aligned}
\end{equation}
\item With a strong simulator as specified in Condition \ref{cond:strong simulator}, if $ b,\sigma \in C_b^{0,4}([0,T]\times\R^d)$, $\tilde f \in C_p^{0,3}([0,T]\times \R^d)$ where $\tilde f(t,x) := f(t,x,\mathbf{U}_n(t,x))$ and $g\in C_p^3(\R^d)$,  then $ \hat g$ has the same order of biases as the one-point estimator in Proposition \ref{prop:zod one-point}
    but bounded variances as $ \epsilon \rightarrow 0$.
    Furthermore, if $ b,\sigma \in C_b^{0,5}([0,T]\times\R^d)$, $ \tilde f \in C_p^{0,4}([0,T]\times\R^d)$ and $g\in C_p^4(\R^d)$,  then $ \hat h$ has the same order of biases as  the one-point estimator in Proposition \ref{prop:zod one-point}     but bounded variances as $ \epsilon \rightarrow 0$.
\end{enumerate}

\end{proposition}

Proposition \ref{prop: zod variance reduction} reveals a unique bias--variance tradeoff associated with the multi-point ZOD estimators in our particular context and the critical difference between the weak and strong simulators. It turns out that the conventional bounded variance in Lemmas \ref{lemma: zod multi-point gradient} and \ref{lemma: zod multi-point hessian} no longer holds for the weak simulator and only holds for the strong simulator. In particular, a multi-point ZOD estimator does not reduce variance with a weak simulator, and its variance still has the same order as the one-point ZOD in Proposition \ref{prop:zod one-point}. Hence, in general, there is no benefit to apply multi-point ZOD when only a weak simulator is accessible. We refer to the multi-point ZOD estimator with a strong simulator as ``ZOD-m" in our subsequent numerical experiments, and will demonstrate that it produces much more stable results in learning derivatives while requiring fewer samples.

\section{The Zeroth-Order Deep Learning Algorithm}

Let $\mathbf{U} = (V, G, H)$ denote the triplet of functions. We define the approximating neural networks as:
\[
\varphi^{\theta} := (\varphi^{\theta}_v, \varphi^{\theta}_g, \varphi^{\theta}_h),
\]
where $\varphi^{\theta}_v: [0,T] \times \mathbb{R}^d \to \mathbb{R}$, $\varphi^{\theta}_g: [0,T] \times \mathbb{R}^d \to \mathbb{R}^d$, and $\varphi^{\theta}_h: [0,T] \times \mathbb{R}^d \to \mathbb{R}^{d \times d}$ approximate the value, gradient, and Hessian, respectively. The parameters are denoted by $\theta \in \Theta \subset \mathbb{R}^p$.

We define the reward function starting from state $(t,x)$ given a function triplet $\mathbf{U}$ as:
\begin{equation}
    R^{t,x}(\mathbf{U}) := g(X_T^{t,x}) + \int_t^T f\left( X^{t,x}_s, V(s, X^{t,x}_s), G(s, X^{t,x}_s), H(s, X_s^{t,x})\right) \dd s.
\end{equation}
Crucially, the source term $f$ is evaluated using the explicit outputs of $G$ and $H$, rather than derivatives computed via auto-diff.

We implement the value iteration \eqref{eq:iter_fk} as follows: At iteration $n$, let $\mathbf{U}_n = (V_n, G_n, H_n)$ be the trained networks from the previous step. We aim to learn $\mathbf{U}_{n+1}$ such that:
\begin{enumerate}
    \item $V_{n+1}(t,x) \approx \hat{V}_{n+1}(t,x) := \E[R^{t,x}(\mathbf{U}_n)]$.
    \item $G_{n+1}(t,x) \approx \nabla \hat{V}_{n+1}(t,x)$.
    \item $H_{n+1}(t,x) \approx \nabla^2 \hat{V}_{n+1}(t,x)$.
\end{enumerate}

\subsection{Loss Function Design}

To train the candidate networks $\varphi = (\varphi_v,\varphi_g,\varphi_h)$ at iteration $n$, we minimize a composite loss function that enforces consistency between the networks and the stochastic targets derived from $\mathbf{U}_n$:
\begin{equation}\label{eq:total_loss}
    L(\varphi, \mathbf{U}_n) := L_v(\varphi_v, \mathbf{U}_n) +  L_g(\varphi_g, \mathbf{U}_n) +   L_h(\varphi_h, \mathbf{U}_n),
\end{equation}
where each individual loss will be defined below, with $\beta>0$ some fixed constant.

\subsubsection{Value Loss}
The value network $\varphi^{\theta}_v$ is trained to match the conditional expectation of the reward. The loss is defined as the mean squared error (MSE) against the stochastic realization:
\begin{equation}
    L_v(\varphi_v, \mathbf{U}_n) := \beta \E\left[ \int_0^T e^{\beta s}\left| \varphi_v(s, X_s) - R^{s,X_s}(\mathbf{U}_n) \right|^2 \dd s \right].
\end{equation}
Minimizing this loss implies $\varphi_v(t,x) \approx \E[R^{t,x}(\mathbf{U}_n)]$.

\subsubsection{Zeroth-Order Derivatives Losses}
Since $\hat{V}_{n+1}$ is defined as an expectation, its derivatives are not directly available. We employ the one-point ZOD estimators (Proposition \ref{prop:zod one-point}) to construct biased targets using perturbed trajectories.\footnote{Here we present the one-point ZOD estimators for illustration. Multi-point ZOD {\it \`a la} Proposition \ref{prop: zod variance reduction} can be discussed similarly.} Let $Z \sim \mathcal{N}(0, I_d)$ be a perturbation vector independent of the trajectory $X$.

\textbf{Gradient Loss.} We rely on the relationship $\nabla \E[R^{t,x}] \approx \E[\frac{Z}{\epsilon} R^{t, x+\epsilon Z}]$. The gradient network $\varphi^{\theta}_g$ minimizes:
\begin{equation}
    L_g(\varphi^{\theta}_g, \mathbf{U}_n) :=\beta\E\left[\int_0^T e^{\beta t}\left| \varphi^{\theta}_g(t, X_t) - \frac{Z}{\epsilon} R^{t,X_t+\epsilon Z}(\mathbf{U}_n) \right|^2 \dd t\right].
\end{equation}

\textbf{Hessian Loss.} Similarly, the Hessian network $\varphi^{\theta}_h$ minimizes the error against the second-order ZOD estimator, where the norm stands for the Frobenius norm for matrices:
\begin{equation}
    L_h(\varphi^{\theta}_h, \mathbf{U}_n) :=\E\left[\int_0^T e^{\beta t}\left| \varphi^{\theta}_h(t, X_t) - \frac{ZZ^{\top} - I_d}{\epsilon^2} R^{t,X_t+\epsilon Z}(\mathbf{U}_n) \right|^2 \dd t\right].
\end{equation}
\subsection{Discretization and Implementation}

In the numerical implementation, expectations are estimated via Monte Carlo sampling. We define a uniform time grid $\mathcal{G} := \{0=t_0 < \dots < t_N = T\}$ with steps $\Delta t_j := t_{j+1} - t_j = \frac{T}{N}$. Let $\{X^{(b)}\}_{b=1}^B$ denote a batch of base trajectories starting from $(0, x_0)$.

For the value loss, the empirical reward along a base trajectory is computed using the frozen networks $\mathbf{U}_n$:
\begin{equation}\label{eq: reward rv}
    \tilde R^{t_j,X^{(b)}_{t_j}}(\mathbf{U}_n) := g(X_T^{(b)}) + \sum_{i = j}^{N-1} f\left(X_{t_i}^{(b)}, V_{n}(t_i,X_{t_i}^{(b)}), G_n(t_i,X_{t_i}^{(b)}), H_n(t_i,X_{t_i}^{(b)})\right) \Delta t_i.
\end{equation}

To compute the derivative losses, for each sample $b$ and time step $j$, we generate a random perturbation $Z^{(b)} \sim \mathcal{N}(0, I_d)$ and simulate (observe) one perturbed trajectory $\{X^{t_j, X_{t_j}^{(b)}+ \epsilon Z^{(b)}}\}_{s \ge t_j}$ starting from $X_{t_j}^{(b)} + \epsilon Z^{(b)}$. The reward along this perturbed path is:
\begin{equation}\label{eq: reward der rv}
\begin{aligned}
    &\tilde R^{t_j, X_{t_j}^{(b)}+\epsilon Z^{(b)}}(\mathbf{U}_n) := g(X_T^{t_j,X_{t_j}^{(b)}+\epsilon Z^{(b)}}) \\
    +& \sum_{i = j}^{N-1} f\left(X_{t_i}^{t_j,X_{t_j}^{(b)}+\epsilon Z^{(b)}}, V_{n}(t_i,X_{t_i}^{t_j,X_{t_j}^{(b)}+\epsilon Z^{(b)}}), G_n(t_i,X_{t_i}^{t_j,X_{t_j}^{(b)}+\epsilon Z^{(b)}}), H_n(t_i,X_{t_i}^{t_j,X_{t_j}^{(b)}+\epsilon Z^{(b)}})\right) \Delta t_i.
    \end{aligned}
\end{equation}

The tractable empirical loss functions are:
\begin{equation}\label{eq: empirical loss}
\begin{aligned}
    \tilde L_v^{(B)}(\varphi_v,\mathbf{U}_n) &:= \beta\frac{1}{B}\sum_{b=1}^B\sum_{j=0}^{N-1} e^{\beta t_j}\left( \varphi_v(t_j,X_{t_j}^{(b)}) - \tilde R^{t_j,X_{t_j}^{(b)}}(\mathbf{U}_n) \right)^2 \Delta t_j, \\
    \tilde L_{g}^{(B)}(\varphi_g,\mathbf{U}_n) &:= \beta\frac{1}{B}\sum_{b = 1}^{B}\sum_{j = 0}^{N-1} e^{\beta t_j}\left| \varphi_g(t_j,X_{t_j}^{(b)}) - \frac{Z^{(b)}}{\epsilon} \tilde R^{t_j, X_{t_j}^{(b)}+\epsilon Z^{(b)}}(\mathbf{U}_n) \right|^2 \Delta t_j, \\
    \tilde L_{h}^{(B)}(\varphi_{h},\mathbf{U}_n) &:= \frac{1}{B}\sum_{b = 1}^{B}\sum_{j = 0}^{N-1} e^{\beta t_j}\left| \varphi_h(t_j,X_{t_j}^{(b)}) - \frac{Z^{(b)}(Z^{(b)})^{\top} - I_d}{\epsilon^2} \tilde R^{t_j, X_{t_j}^{(b)}+\epsilon Z^{(b)}}(\mathbf{U}_n) \right|^2 \Delta t_j.
\end{aligned}
\end{equation}

The empirical total loss is defined by
\begin{equation}\label{eq: empirical total loss}
    \tilde L^{(B)}(\varphi,\mathbf{U}_n):=  \tilde L_v^{(B)}(\varphi_v,\mathbf{U}_n) + \tilde L^{(B)}_{g}(\varphi_g,\mathbf{U}_n) + \tilde L^{(B)}_{h}(\varphi_{h},\mathbf{U}_n).
\end{equation}
The complete procedure is summarized in Algorithm \ref{alg:zo}.

\begin{algorithm}
\caption{Zeroth-Order PDE Solver via Independent Networks}\label{alg:zo}
\begin{algorithmic}[1] % Added [1] for line numbering usually
\State \textbf{Input:} Function $f$, terminal condition $g$, time horizon $T$, initial state $x_0$, time grid $ \mathcal{G} = \{0=t_0<\dots<t_N= T\}$ with $ \Delta t_j = t_{j+1} - t_j$, batch size $ B$, outer iterations $M$, gradient steps per iteration $K$, learning rate $\eta$, perturbation size $\epsilon$.
\State \textbf{Required Program:} A simulator generating sample trajectories whose infinitesimal generator is $\mathcal L$.
\State \textbf{Initialize:} Independent neural networks $\mathbf{U}_0 = (V_0, G_0, H_0)$, parameter $\theta_0$.
\For{$n = 0$ to $M-1$}
    \State Initialize candidate networks $\varphi^{\theta} = (\varphi^{\theta}_v, \varphi^{\theta}_g, \varphi^{\theta}_h)$ with parameters $\theta = \theta_n$
    \For{$k = 1$ to $K$}
        \State Observe $ B$ trajectories $ \{X^{(b)}_{s}\}_{s\in[0,T],b = 1,\dots,B}$ starting from $ (0,x_0)$.
        \State For $ 0\le j \le N$, $ 1\le b \le B$, evaluate $G_{n}(t_j,X_{t_j}^{(b)})$ and $H_{n}(t_j,X_{t_j}^{(b)})$ using frozen networks.
        \State Compute empirical reward random variable $ \tilde R^{t_j,X^{(b)}_{t_j}}(\mathbf{U}_n)$ by \eqref{eq: reward rv}.
        \State Compute $\tilde L^{(B)}_v :=  \beta\frac{1}{B}\sum_{b=1}^B\sum_{j=0}^{N-1}e^{\beta t_j}\left( \tilde R^{t_j,X_{t_j}^{(b)}}(\mathbf{U}_n) - \varphi^{\theta}_v(t_j,X_{t_j}^{(b)})\right)^2\Delta t_j.$
        \State Sample $\{Z^{(b)}\}_{b=1}^B \sim \mathcal{N}(0,I_d)$, for each $ 0\le j \le N$, $ 1\le b \le B$, observe one trajectory $ \left\{ X_{t}^{t_j,X_{t_j}^{(b)}+\epsilon Z^{(b)}}\right\}_{t_j\le t\le T}$ starting at $ (t_j,X_{t_j}^{(b)}+\epsilon Z^{(b)})$.
        \State For $ 0\le j\le N$, $ 1\le b\le B$, $ j\le i\le N$, evaluate $G_{n}(t_i,X_{t_i}^{t_j,X_{t_j}^{(b)}+\epsilon Z^{(b)}})$ and $H_{n}(t_i,X_{t_i}^{t_j,X_{t_j}^{(b)}+\epsilon Z^{(b)}})$ on the perturbed paths.
        \State Compute perturbed reward random variable $ \tilde R^{t_j,X_{t_j}^{(b)}+\epsilon Z^{(b)}}(\mathbf{U}_n)$ by \eqref{eq: reward der rv}.
        \State Compute the derivatives loss by zeroth-order estimators:
         \begin{align*}
        \tilde L^{(B)}_g &:= \beta\frac{1}{B}\sum_{b = 1}^{B}\sum_{j = 0}^{N-1} e^{\beta t_j}\left( \frac{Z^{(b)}}{\epsilon}\tilde R^{t_j,X_{t_j}^{(b)}+\epsilon Z^{(b)}}(\mathbf{U}_n) - \varphi^{\theta}_g(t_j,X_{t_j}^{(b)})\right)^2 \Delta t_j,\\
        \tilde L^{(B)}_h &:= \frac{1}{B}\sum_{b = 1}^{B}\sum_{j = 0}^{N-1} e^{\beta t_j}\left( \frac{Z^{(b)}(Z^{(b)})^{\top} - I_d}{\epsilon^2}\tilde R^{t_j,X_{t_j}^{(b)}+\epsilon Z^{(b)}}(\mathbf{U}_n) - \varphi^{\theta}_h(t_j,X_{t_j}^{(b)})\right)^2 \Delta t_j.
        \end{align*}
        \State Total Loss: $ \tilde L^{(B)}: = \tilde L^{(B)}_v+ \tilde L^{(B)}_g+ \tilde L^{(B)}_h$
        \State Update $ \theta \gets \theta - \eta D_{\theta}\tilde L^{(B)}$.
    \EndFor
    \State Set $\theta_{n+1} \gets \theta$ and $\mathbf{U}_{n+1} \gets \varphi^{\theta_{n+1}}$
\EndFor
\State \textbf{Output:} Trained networks $\mathbf{U}_{M} = (V_{M}, G_{M}, H_{M})$
\end{algorithmic}
\end{algorithm}

\section{Convergence Analysis}
\label{sec:convergence analysis}
In this section, we establish the convergence properties of the proposed zeroth-order based learning algorithm. Let $p(t,\cdot)$ denote the probability density function of $X_t$. We define the following weighted norms:
\begin{equation}
    \|v\|_{L_2(p)}^2 := \int_{\R^d} |v(t,x)|^2p(t,x)\dd x, \quad\|v\|_{\beta}^2 := \int_0^T e^{\beta t}\int_{\R^d} |v(t,x)|^2p(t,x)\dd x\dd t,
\end{equation}
as well as the mixed derivative norm:
\begin{equation}
    \|v\|^2_{\beta, \text{mix}}: = \beta \|v\|_{\beta}^2 + \beta\|\nabla v\|_{\beta}^2 + \|\nabla^2 v\|_{\beta}^2.
\end{equation}
For a function triplet $\mathbf{U} := (V, G, H)$, we abuse the notation slightly and define its corresponding norm as:
\begin{equation}
    \|\mathbf{U}\|_{\beta, \text{mix}}^2 : = \beta \|V\|_{\beta}^2 + \beta\|G\|_{\beta}^2 + \|H\|_{\beta}^2.
\end{equation}
We formally define the solution mapping $\mathcal{S}$ that maps a coefficient triplet to a solution triplet. Let $\mathbf{U} = (V, G, H)$ be a given triplet in an appropriate Hilbert space. Let $u$ be the unique classical solution to the linear parabolic PDE:
\begin{equation} \label{eq: linearized_pde_step}
    \begin{cases}
        \partial_t u + \mathcal{L}u + f\big(x, V(t,x), G(t,x), H(t,x)\big) = 0, & (t,x) \in [0,T) \times \mathbb{R}^d, \\
        u(T,x) = g(x), & x \in \mathbb{R}^d.
    \end{cases}
\end{equation}
The solution mapping $\mathcal{S}$ is then defined as the operator that returns the solution and its spatial derivatives:
\begin{equation}
    \mathcal{S}(\mathbf{U}) := (u, \nabla u, \nabla^2 u).
\end{equation}
Under this definition, the fixed-point problem $\mathbf{u}^* = \mathcal{S}(\mathbf{u}^*)$ is equivalent to the original nonlinear PDE \eqref{eq:pde} with $ v^*$ being the solution of PDE and $ \mathbf{u}^* = (v^*, \nabla v^*, \nabla^2 v^*)$.

To proceed with the analysis, we introduce the following contraction assumption on $\mathcal{S}$.
\begin{assumption}[Contraction]\label{assum: contraction}
    There exist constants $\beta > 0$ and $\gamma < 1$, such that for any two triplets $\mathbf{U}_1 = (V_1, G_1, H_1)$ and $\mathbf{U}_2 = (V_2, G_2, H_2)$, the following inequality holds
    \begin{equation}
        \|\mathcal{S}(\mathbf{U}_1) - \mathcal{S}(\mathbf{U}_2)\|_{\beta, \text{mix}} \le \gamma \|\mathbf{U}_1 - \mathbf{U}_2\|_{\beta, \text{mix}}.
    \end{equation}
\end{assumption}
%The primary contribution of this work is the development of a data-driven learning algorithm that integrates value iteration with the ZOD estimators. A key advantage of this approach is that it circumvents the need to differentiate through the SDE dynamics (via Malliavin calculus; see, e.g. \citealt{han26}), making it applicable to black-box simulators.

In our context, Assumption \ref{assum: contraction} represents the fundamental condition required for the well-posedness and convergence of the underlying value iteration scheme itself, independent of the numerical approximation errors.  While establishing the necessary and sufficient conditions for general nonlinear PDEs to satisfy this contraction property is beyond the scope of this paper, we will provide a rigorous justification for its validity in an important and relevant setting. Specifically, in Appendix \ref{appendix: sufficient condition for contraction} we will study the case of Langevin dynamics with a stationary distribution and  present a sufficient condition for Assumption \ref{assum: contraction}. % in Appendix \ref{appendix: sufficient condition for contraction}. Furthermore, we explicitly verify that this assumption holds for the case of Langevin dynamics with a stationary distribution, thereby demonstrating the theoretical soundness of our framework.

\subsection{Approximate Value Iteration}

In this subsection, we analyze the propagation of learning errors through the iterative process and establish a global error bound for the proposed algorithm.

Let $\mathcal{H}$ denote the space of neural network triplets.
We consider the empirical risk minimization (ERM) problem (cf. \citealt{shalev2014understanding}) at iteration $n$. Let $\mathbf{U}_n$ be the empirical minimizer:
\begin{equation}\label{eq: ERM minimizer}
    \mathbf{U}_n \in \arg\min_{\varphi \in \mathcal{H}} \tilde L^{(B)}(\varphi, \mathbf{U}_{n-1}),
\end{equation}
where $\tilde L^{(B)}$ is the empirical loss defined in \eqref{eq: empirical total loss}.
Under the ERM framework, omitting the optimization error, at each iteration $n \ge 1$, the algorithm produces a neural network triplet $\mathbf{U}_n = (V_n, G_n, H_n)$, which is the ERM minimizer, intended to approximate the theoretical target $ \mathcal{S}(\mathbf{U}_{n-1})$.

We define the {\it one-step approximation error} at iteration $n$ as the deviation between the trained network and the exact image of the previous iterate under the solution map $\mathcal{S}$:
\begin{equation} \label{eq: one_step_error_def}
    E_n := \|\mathbf{U}_n - \mathcal{S}(\mathbf{U}_{n-1})\|_{\beta, \text{mix}}.
\end{equation}
The quantity $E_n$ aggregates two sources of error: the statistical error arising from the Monte Carlo estimation of the reward and its derivatives (via ZOD estimators), and the approximation error associated with the neural network hypothesis space (optimization and statistical error) to be defined momentarily.

The following theorem characterizes the convergence of the sequence $\{\mathbf{U}_n\}_{n\ge 0}$ to the ground truth $\mathbf{u}^*$.

\begin{theorem}
\label{thm: global_convergence}
Suppose that Assumption \ref{assum: contraction} holds with a contraction constant $\gamma \in [0, 1)$. Let $\{\mathbf{U}_n = (V_n,G_n,H_n)\}_{n \ge 0}$ be the sequence of function triplets, which are ERM minimizers defined by \eqref{eq: ERM minimizer}. Then, almost surely, for any $n \ge 1$, the total error is bounded by:
\begin{equation} \label{eq: global_bound}
    \|\mathbf{U}_n - \mathbf{u}^*\|_{\beta, \text{mix}} \le \gamma^n \|\mathbf{U}_0 - \mathbf{u}^*\|_{\beta, \text{mix}} + \sum_{k=1}^n \gamma^{n-k} E_k.
\end{equation}
\end{theorem}

Theorem \ref{thm: global_convergence} presents the convergence result of the proposed method as an \textit{approximate value iteration}. The error bound consists of two distinct components:
\begin{enumerate}
    \item The {\it contraction term} $\gamma^n \|\mathbf{U}_0 - \mathbf{u}^*\|_{\beta, \text{mix}}$, which decays exponentially. This term captures  the inherent stability of the value iteration for the underlying PDE.
    \item The {\it accumulation term} $ \sum_{k=1}^n \gamma^{n-k} E_k$, which represents the accumulated error. The magnitude of $ E_k$ is determined by the expressivity of the neural networks (approximation error) and the variance of the ZOD estimators (statistical error).
\end{enumerate}
This result highlights a critical trade-off: while the contraction constant $\gamma$ is dictated and fixed by the PDE physics, the approximation accuracy $ E_k$ can be controlled by increasing the network capacity and the batch size $B$, and properly tuning the ZOD perturbation $\epsilon$. In Section 5.2, we will provide a detailed analysis of $E_k$, explicitly characterizing its dependence on the sample size $B$ and perturbation parameter $\epsilon$.

\subsection{One-Step Sample Complexity Analysis}

In this subsection, we provide an analysis of the one-step approximation error $E_{n}$. Recall from \eqref{eq: one_step_error_def} that $E_n = \|\mathbf{U}_n - \mathcal{S}(\mathbf{U}_{n-1})\|_{\beta, \text{mix}}$, where $\mathbf{U}_n$ is the trained network triplet (the ERM minimizer) and $\mathcal{S}(\mathbf{U}_{n-1})$ is the ground truth target defined by the solution mapping.

To characterize the learning target, we first introduce the \textit{ZOD solution mapping} $\mathcal{S}^\epsilon$, which maps the previous iterate to the expected values of the zeroth-order estimators. Namely,  $  \mathcal{S}^\epsilon(\mathbf{U}_{n-1}):= (V^\epsilon, G^\epsilon, H^\epsilon) $ is the target triplet defined by the conditional expectations:
\begin{align}
    V^\epsilon(t,x) &:= \E\left[ R^{t,x}(\mathbf{U}_{n-1}) \right] \equiv \hat{V}_n(t,x)
    , \label{eq: V_eps_def}\\
    G^\epsilon(t,x) &:= \E\left[ \frac{Z}{\epsilon} R^{t,x+\epsilon Z}(\mathbf{U}_{n-1}) \right], \label{eq: G_eps_def} \\
    H^\epsilon(t,x) &:= \E\left[ \frac{ZZ^\top - I_d}{\epsilon^2} R^{t,x+\epsilon Z}(\mathbf{U}_{n-1}) \right]. \label{eq: H_eps_def}
\end{align}
Note that $\mathcal{S}^\epsilon(\mathbf{U}_{n-1})$ represents the ``smoothed" solution induced by the Gaussian perturbation, which serves as the proxy for the true solution $\mathcal{S}(\mathbf{U}_{n-1})$.

Under the strong simulator in Condition \ref{cond:strong simulator}, the derivative components of the empirical loss can alternatively be evaluated by the multi-point ZOD estimators $\hat g$ and $\hat h$ in Proposition \ref{prop: zod variance reduction}. In this multi-point construction, the rewards entering each estimator are generated with the same Brownian path upon a single query. By the symmetry of $Z$ and the linearity of expectation, these multi-point ZOD estimators have the same conditional means as $G^\epsilon$ and $H^\epsilon$. Thus the population target remains $\mathcal S^\epsilon(\mathbf U_{n-1})$; in other words, the multi-point construction changes the variance and the empirical-process analysis, but not the target being learned.

The following lemma establishes the relationship between the population loss function and this ZOD target.

\begin{lemma}
\label{lemma: total loss decomp}
For any triplet $\varphi$, the population loss with respect to the previous iterate $\mathbf{U}_{n-1}$ decomposes as:
\begin{equation}
    L(\varphi, \mathbf{U}_{n-1}) = \|\varphi - \mathcal{S}^\epsilon(\mathbf{U}_{n-1})\|_{\beta, \text{mix}}^2 + C,
\end{equation}
where $ C$ is a constant independent of $\varphi$.
\end{lemma}

As a direct consequence of Lemma \ref{lemma: total loss decomp}, the triplet $\mathcal{S}^\epsilon(\mathbf{U}_{n-1})$ is the unique minimizer of the population loss over the space of all measurable functions:
\begin{equation}
    \mathcal{S}^\epsilon(\mathbf{U}_{n-1}) \in \arg\min_{\varphi} L(\varphi, \mathbf{U}_{n-1}).
\end{equation}
Hence, the learning process can be viewed as approximating the ZOD target $\mathcal{S}^\epsilon(\mathbf{U}_{n-1})$ via empirical risk minimization.

We now define the network architecture and hypothesis space used in this paper. For a smooth activation function $\rho:\mathbb R\to\mathbb R$, we use the same notation $\rho(z)$ for its element-wise application to a vector $z$.

\begin{definition}
\label{def: hypothesis space}
Given input dimension $d_{in}$, output dimension $d_{out}$, depth $L$, and width vector $\mathbf m=(m_0,m_1,\ldots,m_L)$ with $m_0=d_{in}$ and $m_L=d_{out}$, we define the multilayer perceptron (MLP) $\Phi(\cdot;\vartheta):\mathbb R^{d_{in}}\to\mathbb R^{d_{out}}$ by
\begin{equation}
\Phi(x;\vartheta)
:=
\mathcal A_L\circ \rho\circ \mathcal A_{L-1}\circ \cdots \circ \rho\circ \mathcal A_1(x),
\end{equation}
where each $\mathcal A_l:\mathbb R^{m_{l-1}}\to\mathbb R^{m_l}$ is the affine map $\mathcal A_l(z)=W_lz+b_l$. The parameter vector $\vartheta$ collects all trainable parameters of this MLP, namely $\vartheta={(W_l,b_l)}_{l=1}^L$.
The hypothesis space is defined by
\begin{equation}
\mathcal H
:=
\left\{
\varphi_{\theta}
=
(V_{\theta},G_{\theta},H_{\theta})
\mid
\theta\in\Theta
\right\},
\end{equation}
where $V_{\theta}:[0,T]\times\mathbb R^d\to\mathbb R$, $G_{\theta}:[0,T]\times\mathbb R^d\to\mathbb R^d$, and $H_{\theta}:[0,T]\times\mathbb R^d\to\mathbb R^{d\times d}$ are three possibly different MLPs. The output of $H_{\theta}$ has dimension $d^2$ and is reshaped into a $d\times d$ matrix.
\end{definition}

Note that we use the same symbol $\theta$ for all the three networks: $\theta$ denotes the collection of their trainable parameters, and when appearing in one of $V_\theta$, $G_\theta$, or $H_\theta$ it refers to the corresponding subcollection of parameters.

\begin{assumption}
\label{assum: hypothesis space}
The activation function satisfies $\rho\in C_b^5(\mathbb R)$.  The parameter set $\Theta\subset\mathbb R^p$ is compact and satisfies $|\theta|_\infty\le R$ for all $\theta\in\Theta$, where $p$ is the total number of trainable parameters in the three networks.
\end{assumption}

Under the above construction and assumption, due to the smoothness of $\rho$ and the compactness of $\Theta$, all functions in $\mathcal{H}$ and their partial derivatives up to the fourth order with respect to $x$ are uniformly bounded.

Next we introduce the regularity assumption on the expected reward function induced by any candidate in the hypothesis space. This is crucial for establishing the one-step sample complexity result.

For any $\varphi \in \mathcal{H} $, denote
\begin{equation}
    J_1(t, x; \varphi) := \E \left[ R^{t,x}(\varphi) \right],\quad J_2(t,x;\varphi):= \E[(R^{t,x}(\varphi))^2], \quad \tilde f(t,x;\varphi) := f(t,x,\varphi(t,x)).
\end{equation}

\begin{assumption}
\label{assum: reward and f regularity}
The functions $J_1(\cdot, \cdot; \varphi)\in C_p^{1,4}([0,T]\times\R^d), J_2(\cdot,\cdot;\varphi)\in C_p^{1,2}([0,T]\times \R^d), \tilde f \in C_p^{1,2}([0,T]\times \R^d)$ uniformly over the hypothesis space $ \varphi\in \mathcal{H}$ in the sense that their derivatives have polynomial growth with constants
independent of $\varphi$. The terminal function $g$ has polynomial growth.
\end{assumption}

Assumption \ref{assum: reward and f regularity} is a regularity condition on the solutions to the linear PDE in the value iterate. Typical sufficient conditions for it to hold include stronger regularity conditions on the SDE coefficients and/or uniform ellipticity.

\begin{theorem}\label{thm: one step error bound}
Suppose Assumptions \ref{assum: sde coef}, \ref{assum: hypothesis space}, and \ref{assum: reward and f regularity} hold. Let $0<\epsilon<1$ and $0<\delta<1$, and define $L_{B,\delta}:=\log B+\log(1/\delta)$. Let $\mathcal P$ be a polynomial determined by the architecture of $\mathcal H$, $C$ be a constant depending only on the PDE coefficients and the constants in the above assumptions. Let \(q\) be a constant depending only on the
polynomial growth orders of the functions in Assumption \ref{assum: reward and f regularity}.
 \begin{enumerate}[(a)]
\item With the weak simulator in Condition \ref{cond:weak simulator}, if the derivative terms in the empirical loss are evaluated using the one-point ZOD estimators, then, with probability at least $1-\delta$,
    \begin{equation}
        \begin{aligned}
        E_n^2 \le& C\left( \underbrace{\frac{1}{N\epsilon^{4}}}_{\text{Discretization Error}} + \underbrace{\epsilon^{4}}_{\text{ZOD Bias}}
        + \underbrace{\left(1+\sqrt{\log(1/\delta)}\right)
        \frac{\mathcal P(d)\sqrt{\log(1/\epsilon)}L_{B,\delta}^{q}}
        {B^{1/2}\epsilon^4}}_{\text{Statistical Error}}\right) \\
        &+ \underbrace{2 \inf_{\varphi\in \mathcal{H}} \normmix{\varphi - \mathcal{S}^\epsilon(\mathbf{U}_{n-1})}^2}_{\text{Approximation Error}}.
        \end{aligned}
    \end{equation}
\item Suppose further that $b,\sigma \in C^{0,5}_b([0,T]\times\R^d)$, $g\in C^{4}_p(\R^d)$ and $f(t,x,\varphi(t,x)) \in C^{1,5}_p([0,T]\times\R^d)$ uniformly over the hypothesis space $ \varphi\in \mathcal{H}$ in the sense that their derivatives have polynomial growth with constants independent of $\varphi$. With the strong simulator in Condition \ref{cond:strong simulator}, if the derivative terms in the empirical loss are evaluated using the multi-point estimators stipulated in Proposition \ref{prop: zod variance reduction}, then, with probability at least $1-\delta$,
       \begin{equation}
        \begin{aligned}
        E_n^2 \le& C\left( \underbrace{\frac{1}{N}}_{\text{Discretization Error}} + \underbrace{\epsilon^{4}}_{\text{ZOD Bias}}
        + \underbrace{\left(1+\sqrt{\log(1/\delta)}\right)
        \frac{\mathcal P(d)\Gamma_{\epsilon,B,\delta}}{B^{1/2}}}_{\text{Statistical Error}}\right)\\
        &+ \underbrace{2 \inf_{\varphi\in \mathcal{H}} \normmix{\varphi - \mathcal{S}^\epsilon(\mathbf{U}_{n-1})}^2}_{\text{Approximation Error}},
        \end{aligned}
    \end{equation}
    where
$\Gamma_{\epsilon,B,\delta}
        :=
        \min\left\{
        \epsilon^{-4}\sqrt{\log(1/\epsilon)}L_{B,\delta}^{q},
        \exp\!\left(C\sqrt{L_{B,\delta}}\right)
        \right\}$.
    
    \end{enumerate}
\end{theorem}
% \begin{proof}
%     See Section \ref{proof sec: one step error bound}.
% \end{proof}

Theorem \ref{thm: one step error bound} provides high-probability upper bounds on the one-step error,  under both the weak and strong simulators,  in terms of four components: the discretization error from approximating the time integral, the ZOD bias from using ZOD-based derivative representations instead of exact stochastic representations, the statistical error from the  empirical loss due to a finite sample size, and the approximation error of the neural-network family. Several points are worth noting. First, the discretization error is controlled by the number of time-grid points used to compute the integrals. Since we use the simplest forward Euler scheme, this term is of order $N^{-1}$, which can be improved by adopting a higher-order numerical integration scheme. Importantly, increasing the number of time-grid points does not introduce instability into the algorithm or the loss function. Second, the ZOD bias is controlled by the bandwidth and is of order $\epsilon^4$ in both simulator settings. This part of the error follows from Propositions \ref{prop:zod one-point} and \ref{prop: zod variance reduction}. Third, the statistical error is jointly controlled by the sample size and the complexity of the neural-network class. Our statistical bound is conservative: the squared one-step error contains a sample-size term of order \(B^{-1/2}\) and depends polynomially on the dimension of the state variable; this rate arises from our use of a global Rademacher-complexity bound for the empirical-loss term.\footnote{%rather than the fast \(B^{-1}\) rate available in well-specified least-squares regression under stronger localized complexity and Bernstein-type conditions.
Sharper localized critical radius analyses are possible in principle, but they require additional technical assumptions and substantially more involved arguments, which are not the focus of this paper; see, e.g., \citet{bartlett2005Local,wainwright2019HighDimensional}.} Finally, the approximation error is in terms of the weighted Sobolev norm of the neural network classes, which reflects the quality of  the neural network classes chosen.

Moreover, under the weak simulator, the ZOD estimator has exploding variances. As a consequence, both the discretization error and the statistical error contain the factor $\epsilon^{-4}$, which substantially worsens the bound and creates a bias--variance tradeoff against the ZOD bias. Ignoring logarithmic factors, the dominant terms are of the order $(N^{-1}+\mathcal P(d)B^{-1/2})\epsilon^{-4}+\epsilon^4$. This suggests choosing  $\epsilon=O((N^{-1}+\mathcal P(d)B^{-1/2})^{1/8})$ to yield the squared one-step error $E_n^2=O((N^{-1}+\mathcal P(d)B^{-1/2})^{1/2})$, up to logarithmic factors and the approximation error.

By contrast, under the strong simulator, the variance of the multi-point ZOD estimator is bounded uniformly in $\epsilon$. Hence neither the discretization error nor the statistical error suffers from the $\epsilon^{-4}$ blow-up. We note, however, that to cancel this $ \epsilon^{-4}$ term, one relies on the tangent process. Since the global Rademacher complexity argument requires a bounded loss class, a truncation argument is needed. The tangent process typically has a heavier tail than the original state process, which leads to the factor $\exp(C\sqrt{\log B+\log(1/\delta)})$ in the statistical error bound. Although this factor is worse than the logarithmic factor appearing in the weak one-point estimate, it is still sub-polynomial in both $B$ and $1/\delta$, because $\exp(C\sqrt{L_{B,\delta}})=(B/\delta)^{o(1)}$.  Consequently, in the asymptotic sense,  choosing $\epsilon=O((N^{-1}+\mathcal P(d)B^{-1/2})^{1/4})$ balances the ZOD bias with the discretization and statistical errors, and gives the squared one-step error $E_n^2=O(N^{-1}+\mathcal P(d)B^{-1/2})$, up to the approximation error. To sum, compared with the weak simulator, the strong simulator gives a substantially better asymptotic order by removing the unfavorable $\epsilon^{-4}$ dependence.

Combining our analysis for the approximate value iteration and one-step sample complexity results, Theorem \ref{thm:total error} presents the total error of the algorithm up to the $n$-th value iterate with any fixed bandwidth $\epsilon$, the number of time discretization grids $N$, and the sample size $B$ for training.

\begin{theorem}
\label{thm:total error}
Suppose Assumptions \ref{assum: sde coef}--\ref{assum: reward and f regularity} hold. Let $0<\epsilon<1$ and $0<\delta<1$, and write $L_{B,\delta,n}:=\log B+\log(n/\delta)$. Let $\mathcal P$ be a polynomial determined by the architecture of $\mathcal H$,  $C$ be a constant depending only on the PDE coefficients and the constants in the assumptions. Let \(q\) be a constant depending only on the
polynomial growth orders of the functions in Assumption \ref{assum: reward and f regularity}.
\begin{enumerate}[(a)]
\item Under the conditions of Theorem \ref{thm: one step error bound}-(a),
with probability at least $1-\delta$, we have
\begin{equation}
\begin{aligned}
    \|\mathbf{U}_n - \mathbf{u}^*\|_{\beta, \text{mix}} &\le \gamma^n \|\mathbf{U}_0 - \mathbf{u}^*\|_{\beta, \text{mix}} + \sqrt{2}\sum_{k=1}^n\gamma^{n-k}\inf_{\varphi \in \mathcal{H}}\|\varphi - \mathcal{S}^\epsilon(\mathbf{U}_{k-1})\|_{\beta,\text{mix}}\\
    + C\frac{1-\gamma^n}{1-\gamma}&\left( N^{-1/2}\epsilon^{-2} + \epsilon^2 + \left( 1+ (\log (n/\delta))^{1/4}\right)\frac{\sqrt{\mathcal{P}(d)}(\log(1/\epsilon))^{1/4}L_{B,\delta,n}^{q/2}}{B^{1/4}\epsilon^2}\right).
    \end{aligned}
\end{equation}
\item Under the conditions of Theorem \ref{thm: one step error bound}-(b),
with probability at least $1-\delta$, we have
    %Suppose further that $b,\sigma \in C^{0,5}_b([0,T]\times\R^d)$, $f(t,x,\mathbf{U}_{n-1}(t,x)) \in C^{1,5}_p([0,T]\times\R^d)$, and $g\in C^{4}_p(\R^d)$. Under the strong simulator in Condition \ref{cond:strong simulator}, if the derivative components of the empirical loss are evaluated by the multi-point estimators in Proposition \ref{prop: zod variance reduction} with common-noise coupling, then with probability at least $1-\delta$,
    %\[
%        \Gamma_{\epsilon,B,\delta,n}
%        :=
%        \min\left\{
%        \epsilon^{-4}\sqrt{\log(1/\epsilon)}L_{B,\delta,n}^{4},
%        \exp\!\left(C\sqrt{L_{B,\delta,n}}\right)
%        \right\}.
%    \]
    \begin{equation}
\begin{aligned}
    \|\mathbf{U}_n - \mathbf{u}^*\|_{\beta, \text{mix}} \le &\gamma^n \|\mathbf{U}_0 - \mathbf{u}^*\|_{\beta, \text{mix}} + \sqrt{2}\sum_{k=1}^n\gamma^{n-k}\inf_{\varphi \in \mathcal{H}}\|\varphi - \mathcal{S}^\epsilon(\mathbf{U}_{k-1})\|_{\beta,\text{mix}}\\
    &+ C\frac{1-\gamma^n}{1-\gamma}\left( N^{-1/2} + \epsilon^2 + \left(1+(\log(n/\delta))^{1/4}\right)
        \frac{\sqrt{\mathcal P(d)\Gamma_{\epsilon,B,\delta,n}}}{B^{1/4}}\right),
    \end{aligned}
\end{equation}
where $\Gamma_{\epsilon,B,\delta,n}
        :=
        \min\left\{
        \epsilon^{-4}\sqrt{\log(1/\epsilon)}L_{B,\delta,n}^{q},
        \exp\!\left(C\sqrt{L_{B,\delta,n}}\right)
        \right\}$. 
 \end{enumerate}
\end{theorem}

% --- SECTION 4: NUMERICAL EXPERIMENTS ---
\section{Numerical Experiments}
\label{sec:numerical_experiments}

In this section, we investigate the numerical performance of the proposed zeroth-order solver in learning the full solution triplet
$(v^\star,\nabla v^\star,\nabla^2 v^\star)$.

Let $\mathbf U_n=(V_n,G_n,H_n)$ denote the learned value, gradient, and Hessian networks respectively after the $n$-th value iteration. Throughout the numerical experiments, we report empirical relative root mean squared errors (rRMSE) between the learned networks and the ground truth solution by evaluating the errors on a set of randomly generated test samples. Specifically, given test samples $z_i=(t_i,x_i)$ and a target quantity $Q^\star$ evaluated at these samples, define
\begin{equation}
    \operatorname{rRMSE}(\widehat Q,Q^\star)
    =
    \left(
    \frac{\sum_{i=1}^{N_{\mathrm{test}}}
    \left|\widehat Q(z_i)-Q^\star(z_i)\right|^2}
    {\sum_{i=1}^{N_{\mathrm{test}}}
    \left|Q^\star(z_i)\right|^2}
    \right)^{1/2},
    \label{eq:numexp_rrmse}
\end{equation}
with the absolute value for scalar quantities, the Euclidean norm for vectors, and the Frobenius norm for matrices.\footnote{In our theoretical analysis, we rely on the weighted Sobolev norm that depends on a constant $\beta$. This constant exists in theory but is difficult to determine in experiments. Therefore, in the numerical study, we only calculate the unweighted Sobolev norm in training and present the results for function value, gradient, and Hessian in terms of the relative mean squared errors respectively.} For the solution triplet, $\widehat Q$ is taken directly from $V_n$, $G_n$, or $H_n$ and compared with $v^\star$, $\nabla v^\star$, or $\nabla^2v^\star$, respectively. Note once again that our approach learns the three networks separately, and the derivatives are not obtained by auto-differentiating the learned value network $V_n$.

For the PDE-solving experiments reported below, we use a short initialization stage as a pre-training stage. With the initial scalar network $V_0$ fixed, the derivative networks $G_0$ and $H_0$ are pretrained to fit $\nabla_x V_0$ and $\nabla_x^2 V_0$, respectively, where these derivatives are computed by automatic differentiation of $V_0$ on samples from the corresponding training law. This pretraining is used only to initialize the derivative networks. In the subsequent value iterations, the networks are warm-started: the parameters of $V_{n+1}, G_{n+1}, H_{n+1}$ are initialized from the trained parameters of $V_n,G_n,H_n$ rather than from a fresh random initialization. All numerical experiments were run on a single NVIDIA RTX 3090 GPU with 24 GB memory. Code is available at \url{https://github.com/du-ouyang/ZO4PDE}.

In the following subsections, we first present a simple toy example of a linear PDE in Subsection \ref{subsec:spiky-diagnostic}, which is effectively a regression problem explained in Subsection \ref{sec:illustrate toy}, to contrast our approach to the
%\sout{conventional learning-then-differentiating one}
existing ones. In Subsections \ref{sec:numexp:hjb_ou_20d} and \ref{sec:fully nonlinear}, we compare our method to the deep Picard iteration (DPI) in \cite{han26} on benchmark nonlinear PDEs where the accuracy can be measured against exact solutions. For a fairer  comparison in these DPI benchmarks, we use the multi-point ZOD estimators (ZOD-m) for gradients and Hessians (see Proposition \ref{prop: zod variance reduction}). This is because the DPI method is model-based, which assumes knowing the operator $\mathcal L$, a much stronger assumption including the availability of a strong simulator as a special case. Moreover, the SDEs in these benchmarks can be exactly simulated without numerical error. Hence, the time integrals appearing both in the training loss and in the reward random variable $R$ can be estimated by sampling $t$ uniformly on $[0,T]$ and then sampling $X_t$ exactly from its marginal law.\footnote{
More precisely, the time integral in the loss is evaluated by first sampling \(t\sim{\rm Unif}[0,T]\) and then sampling \(X_t\) exactly from its marginal law, using the identity \(\int_0^T e^{\beta t}\E[\ell(t,X_t)]\,\dd t=T\E[e^{\beta t}\ell(t,X_t)]\). Similarly, for the reward \(R^{t,x}=g(X_T^{t,x})+\int_t^T f(s,X_s^{t,x},\mathbf U(s,X_s^{t,x}))\,\dd s\), the integral term is estimated by first sampling \(s\sim{\rm Unif}[t,T]\) and then sampling \(X_s^{t,x}\) exactly from the transition law, so that \(\int_t^T \E[f(s,X_s^{t,x},\mathbf U(s,X_s^{t,x}))]\,\dd s=(T-t)\E[f(s,X_s^{t,x},\mathbf U(s,X_s^{t,x}))]\).} In this way, the reported errors isolate the learning and ZOD-estimation errors, without having to consider the time discretization errors.

\subsection{A Linear PDE: Regression}
\label{subsec:spiky-diagnostic}
We consider the following  (one-dimensional) linear PDE: % to demonstrate the limitation of the learning-then-differentiating
%paradigm.
\begin{equation}
\label{eq:spiky-value}
\partial_t v(t,x) + \frac{1}{2}\sigma^2\partial_{xx}v(t,x) = 0,\  (t,x)\in [0,T)\times \mathbb R,
\end{equation}
with terminal condition $v(T,x) = g(x)$.

We are interested only in the time-0 value, gradient, and Hessian, namely $u(x)=v(0,x)$, $u'(x)=\partial_x v(0,x)$, and $u''(x)=\partial_{xx} v(0,x)$. By the Feynman--Kac formula, this is equivalent to computing the conditional expectation
\[ u(x) = \E\left[g\left(X_T\right)\mid X_0 =x  \right], \]
where $\dd X_t = \sigma \dd W_t$. It is also equivalent  to a regression task introduced in Subsection \ref{sec:illustrate toy}, with $Y = g(X + \sigma \xi)$.

Set the terminal function $g$ as
\begin{equation}
    g(y)
    =
    A_1\sin(\alpha y)+A_2\cos(\beta y)
    +
    \sum_{j=1}^{J}
    a_j
    \exp\left(
        -\frac{(y-c_j)^2}{2\ell_j^2}
    \right),
\end{equation}
and the ground truth solution is available by Gaussian convolution:
\begin{equation}
\begin{aligned}
u^*(x)
&=
A_1 e^{-\alpha^2\sigma^2/2}\sin(\alpha x)
+
A_2 e^{-\beta^2\sigma^2/2}\cos(\beta x)  \\
&\quad+
\sum_{j=1}^{J}
a_j
\frac{\ell_j}{s_j}
\exp\left(
-\frac{(x-c_j)^2}{2s_j^2}
\right),
\qquad
s_j^2=\ell_j^2+\sigma^2 .
\end{aligned}
\end{equation}
The ground truth derivatives can then be computed analytically from the above expression, denoted by ${u^*}',{u^*}''$.

For the learning-then-differentiating paradigm, we first train a scalar network \(V_\vartheta\) with the least-squares loss function:
\begin{equation}
    \mathcal L_{\mathrm{NN}}(\vartheta)
    =
    \frac1B
    \sum_{i=1}^{B}
    \left|
        V_\vartheta(x_i)-g(x_i+\sigma\xi_i)
    \right|^2 ,
\end{equation}
where \(x_i\sim\mathrm{Unif}[-2,2]\) and
\(\xi_i\sim\mathcal N(0,T)\). The derivative estimates are then obtained by
automatic differentiation:
\[
    \widehat\Delta_{\mathrm{NN}}(x)=\partial_xV_\vartheta(x),
    \qquad
    \widehat\Gamma_{\mathrm{NN}}(x)=\partial_{xx}V_\vartheta(x).
\]

For our ZOD approach, we train three networks
\((V_{\theta_v},G_{\theta_g},H_{\theta_h})\), parameterized  collectively by
\(\theta\). At each step, we sample \(\{x_i\}_{i=1}^B\sim\mathrm{Unif}[-2,2]\),
\(\{z_i\}_{i=1}^B\sim\mathcal N(0,1)\), and \(\{\xi_i\}_{i=1}^B\sim\mathcal N(0,T)\). With the
multi-point ZOD estimator, define
\begin{align}
    \widehat V_i
    &=
    g(x_i+\sigma\xi_i),\\
    \widehat G_i
    &=
    \frac{z_i}{2\epsilon}
    \left[
        g(x_i+\epsilon z_i+\sigma\xi_i)
        -
        g(x_i-\epsilon z_i+\sigma\xi_i)
    \right],\\
    \widehat H_i
    &=
    \frac{z_i^2-1}{2\epsilon^2}
    \left[
        g(x_i+\epsilon z_i+\sigma\xi_i)
        +
        g(x_i-\epsilon z_i+\sigma\xi_i)
        -
        2g(x_i+\sigma\xi_i)
    \right].
\end{align}
The three networks are trained jointly by
\begin{equation}
\begin{aligned}
    \mathcal L_{\mathrm{ZOD}}(\theta)
    =
    \frac1B\sum_{i=1}^{B}
    \Big(
    &|V_{\theta_v}(x_i)-\widehat V_i|^2
    +
    |G_{\theta_g}(x_i)-\widehat G_i|^2  +
    |H_{\theta_h}(x_i)-\widehat H_i|^2
    \Big).
\end{aligned}
\end{equation}
Note that the same Brownian noise \(\xi_i\) is reused in definition of the multi-point ZOD estimator
\(\widehat G_i\) and \(\widehat H_i\), which is available only with the strong simulator.

In the experiment, we use
\[
\begin{gathered}
\sigma=0.02,\quad T= 1,\quad
A_1=0.22,\quad A_2=0.06,\quad
\alpha=1.3,\quad \beta=4.7,\quad J=7,\\
(c_j)=(-1.55,-1.05,-0.62,-0.18,0.24,0.71,1.28),\\
(\ell_j)=(0.065,0.055,0.070,0.060,0.055,0.065,0.060),\\
(a_j)=(0.035,-0.030,0.032,0.028,-0.034,0.030,-0.027).
\end{gathered}
\]
Both auto-differentiation and ZOD methods use MLPs with four hidden layers of width \(256\) and \(\tanh\)
activation. The rRMSEs are evaluated on \(1000\) uniformly generated initial spatial
points in \([-2,2]\).  NN-autodiff is trained with 10000 gradient steps and  batch size \(B = 32768\), learning rate \(3\times10^{-4}\). Our  ZOD estimator is trained with 5000 gradient steps, batch size \(B = 16384\), learning rate \(5\times10^{-4}\), and ZOD bandwidth \(\epsilon=0.01\).

\begin{table}[h]
\centering
\caption{\textbf{rRMSEs for the linear PDE \eqref{eq:spiky-value}.}}
\label{tab:spiky-results}
\small
\begin{tabular}{lcccc}
\toprule
Method & Value Error & Gradient Error & Hessian
Error & Time (s) \\
\midrule
NN-autodiff & $5.425\times10^{-2}$ & $4.532\times10^{-1}$ & $9.368\times10^{-1}$ & $47.2$ \\
ZOD-m & $7.184\times10^{-2}$ & $6.751\times10^{-2}$ & $3.633\times10^{-1}$ & $44.1$ \\
\bottomrule
\end{tabular}
\end{table}

Table~\ref{tab:spiky-results} summarizes the rRMSEs of both approaches under comparable running times. Observe that although the
value errors are of the same order, our ZOD approach yields
substantially smaller derivative errors in both the first and second order and outperforms by at least 3 times. Judging from the magnitude of these errors, one might conclude that NN-autodiff, while lagging behind ZOD,  could still achieve reasonable accuracy in the derivatives. However, rRMSE is an {\it averaged} measure (over the initial states) that may not fully capture the point-by-point errors of a learned function from its oracle. To visualize the latter, we draw diagnostic plots in Figure~\ref{fig:spiky-value-diagnostic}, where we put  the ground truth solution  and the learned functions by the two approaches together as functions of $x$. In terms of the learned value function, both methods produce very close fit to the true one, whereas the gap becomes significant for
the gradient and Hessian. In particular, the ZOD-based derivatives follow the oracle solutions closely all the time, capturing almost every peak and trough as well as the shapes of the functions, but the ones by NN-autodiff smoothly vary over the whole interval and miss most of the peaks and troughs. Thereby, despite the small average errors in the gradient and Hessian, the learned derivative functions are not close at all with NN-autodiff. This example highlights the limitations of auto-differentiating a learned value function directly when accurate derivative approximation is required, 
and justifies our method that aims for theoretical guarantee in not only function values, but also in derivatives values.

\begin{figure}[h]
\centering
\includegraphics[width=\textwidth]{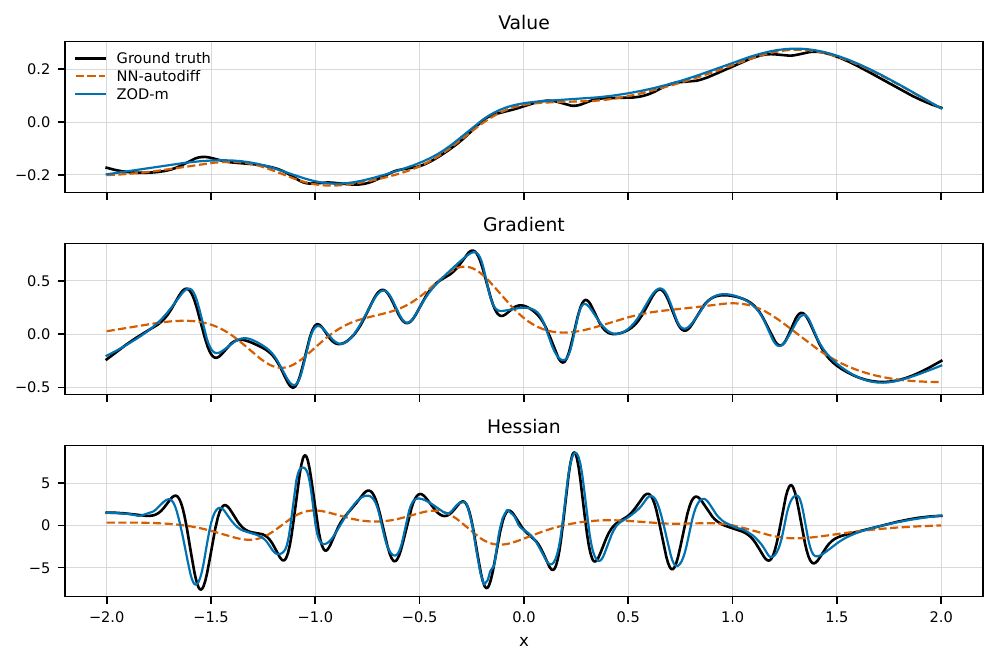}
\caption{\textbf{The learned value function, gradient, and Hessian for example \eqref{eq:spiky-value}}. NN-autodiff learns the value function by the least-squares method and obtains derivatives by automatic
differentiation. ZOD-m trains value, gradient, and Hessian networks directly
using multi-point ZOD estimators.}
\label{fig:spiky-value-diagnostic}
\end{figure}

\subsection{A Semilinear PDE}
\label{sec:numexp:hjb_ou_20d}

We next consider a semilinear PDE studied in \cite{han26}.
\begin{equation}
    \partial_t u(t,x)
    + \frac{1}{2}\Delta u(t,x)
    + x\cdot \nabla u(t,x)
    - \frac{1}{2}|\nabla u(t,x)|^2
    - d = 0,
    \qquad (t,x)\in [0,T)\times \mathbb{R}^d,
\label{eq:hjb_ou_20d_numexp}
\end{equation}
with terminal condition $
    u(T,x) = g(x) := -\log p_0(x)$. Here $p_0$ is the density of a five-component Gaussian mixture
\begin{equation}\label{eq:gmm_terminal_numexp}
    p_0(x) = \sum_{k=1}^{5}\pi_k\,\mathcal{N}(x;\mu_0^{(k)},\Sigma_0^{(k)}),
\end{equation}
where the means $\mu_0^{(k)}$ are sampled uniformly from $[-1,1]^d$, the covariance matrices are diagonal with $\Sigma_0^{(k)} = 2I_d$, and the weights $\pi_k$ are positive and normalized with $\sum_{k=1}^5 \pi_k=1$.

This equation has differential operator $\mathcal L u = \frac{1}{2}\Delta u$, which corresponds to the SDE
\begin{equation}\label{eq:sde for hjb}
    \dd X_t =    \dd W_t, \quad X_0 = \xi,
\end{equation}
and the nonlinear source term is $f(t,x,\nabla u) = x\cdot \nabla u(t,x) -\frac{1}{2}|\nabla u(t,x) |^2 - d$.

Then we can find the exact solution $ u^*$ of the PDE \eqref{eq:hjb_ou_20d_numexp} as\footnote{This is because we know the density $p(t,x)$ of $\hat X_t$, where $ \dd \hat X_t = -\hat X_t \dd t + \dd W_t$,  satisfies the Fokker--Planck equation $\partial_t p = \nabla\cdot(xp) + \frac{1}{2}\Delta p$. And we can verify that $ u(t,x) = -\log p(T- t,x)$ satisfies PDE \eqref{eq:hjb_ou_20d_numexp}. Hence, if $ \hat X_t$ has the initial distribution $p_0(x)$, we can deduce its distribution time-$t$ distribution $p(t,x)$.}
\begin{equation}
    u^\ast(t,x) = -\log p(T-t,x),
\label{eq:exact_solution_numexp}
\end{equation}
where
\begin{equation}\label{eq:gmm_time_evolved_numexp}
    p(t,x) = \sum_{k=1}^{5}\pi_k\,\mathcal{N}(x;\mu_{t}^{(k)},\Sigma_{t}^{(k)}),
\end{equation}
with
\begin{equation}
    \mu_{t}^{(k)} = e^{-t}\mu_0^{(k)},
    \qquad
    \Sigma_{t}^{(k)} = 2e^{-2t}I_d + \frac{1-e^{-2t}}{2}I_d. \label{eq:gmm_time_evolved_moments_numexp}
\end{equation}

In our experiment, we set $
d=20,T=1$. The above configuration is identical to that in \citet{han26} except that we lower the dimension to $d=20$. This is because we  also need to evaluation and compare the accuracy of the learned Hessian, which has an output dimension $d^2$, whereas \citet{han26} only evaluate the gradient whose output dimension is $d$. For the SDE \eqref{eq:sde for hjb} used to generate training and test datasets, we set $ \xi \sim \mathcal{N}(0,4 I_d)$ as in \citet{han26}.

For a fair comparison, both ZOD and the DPI baseline use the same neural network framework from \cite{han26}, which explicitly enforces the terminal condition. Denote
\begin{equation}
    u_\theta(t,x)
    =
    (r_\eta(T - t) - r_0(0))\,\langle N_\gamma(T-t,x),x\rangle
    +
    \bigl(1-r_\eta(T-t) + r_{\eta}(0)\bigr)\,g\!\left(e^{-(T-t)/2}x\right),
    \label{eq:pisgradnet_ansatz}
\end{equation}
where $N_\gamma(t,x)\in\mathbb{R}^d$ and $r_\eta(t) \in \R^+$ are neural networks.
Consequently,
$u_\theta(T,x)=g(x)$.
%The essential difference between the two solvers therefore lies not in the value-network ansatz, but in how derivative information is learned and controlled.

Note that we compare our ZOD approach with DPI in \cite{han26} on the $20$-dimensional PDE \eqref{eq:hjb_ou_20d_numexp}, using the same Gaussian-mixture terminal density \eqref{eq:gmm_terminal_numexp} and the same network  \eqref{eq:pisgradnet_ansatz} for the value function. This isolates the main algorithmic difference: how derivative information is represented and controlled over iterations.

For ZOD, we use the three-network model,
where the value network uses the parameterization \eqref{eq:pisgradnet_ansatz} and the gradient and Hessian networks are 4-layer MLPs with width $512$ and ELU activations. We use the Adam optimizer with a learning rate $10^{-3}$, batch size $16384$, $10$ value iterations, and $512$ gradient updating steps per iteration. The derivative networks are pretrained for $5000$ steps. The bandwidth of ZOD is $\epsilon=0.02$.

For DPI, we use the configuration in \cite{han26}. Specifically, the value network is again \eqref{eq:pisgradnet_ansatz} with the same 4 hidden layers of width $512$ and ELU activations. DPI uses $10$ outer Picard iterations, batch size $512$, data size $4096$, and $16$ epochs per outer iteration, with the Adam optimizer and a learning rate $10^{-3}$. Gradient supervision is enabled with the same weight as in \cite{han26}. Moreover,  second-order information is obtained by differentiating the learned scalar value approximation.

Figure~\ref{fig:hjb_ou_20d_zod_vs_dpi} reports the rRMSE of the learned value, gradient, and Hessian over iterations for the problem \eqref{eq:hjb_ou_20d_numexp}. The curves show the means over three random seeds, and the shaded regions indicate the empirical $2.5\%$--$97.5\%$ quantile range across seeds. We see that our ZOD approach converges faster than DPI in all the three quantities, and eventually achieves slightly smaller errors in terms of the value and gradient after sufficient numbers of iterations. However, the performance differs dramatically in terms of the Hessian accuracy. Over iterations, DPI barely reduces the error while ZOD effectively decreases it although the terminal error tends to be larger than its gradient and value counterparts.

\begin{figure}[h]
    \centering
    \includegraphics[width=\textwidth]{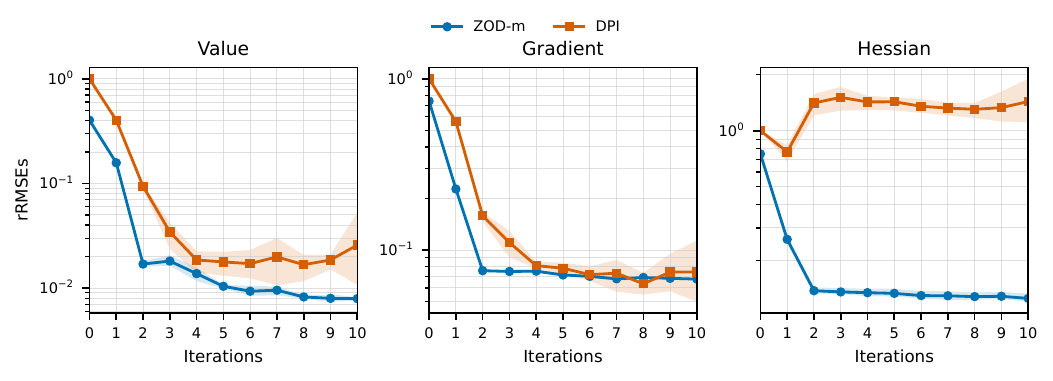}
    \caption{\textbf{rRMSE of the learned value, gradient, and Hessian over iterations} for the PDE \eqref{eq:hjb_ou_20d_numexp} with dimension $d=20$. ZOD-m and DPI use the same network \eqref{eq:pisgradnet_ansatz} for value function.}
    \label{fig:hjb_ou_20d_zod_vs_dpi}
\end{figure}

Table~\ref{tab:hjb_ou_20d_final} summarizes the corresponding mean final errors at the $10$ value (Picard) iterations. The two methods produce comparable approximations of the value function and its gradient. ZOD achieves a smaller rRMSE, while the gradient errors of the two are similar.
However, the final
Hessian rRMSE of ZOD is $1.252\times10^{-1}$, whereas DPI gives $1.434$, differing by more than 10 times.
This discrepancy is a result  of how the Hessian is obtained.
In DPI, the Hessian is not learned directly; rather  it is obtained by applying automatic differentiation twice on the learned value
network. Since the training loss
does not explicitly control the second-order derivative, the resulting Hessian
need not be accurate even when the value function and its gradient are
reasonably well approximated.
%This behavior is visible in Figure~\ref{fig:hjb_ou_20d_zod_vs_dpi}: under the present configuration, the DPI Hessian error remains at a relatively large level across the iterations.
By contrast, the proposed ZOD method augments the learning procedure with a dedicated Hessian network and a second-order zeroth-order objective. This
second-order training signal directly targets the Hessian, leading to a substantially more accurate approximation. Finally, it is notable that ZOD spends about 30\% less time obtaining these results.

\begin{table}[h]
    \centering
    \caption{\textbf{rRMSEs for the semilinear PDE \eqref{eq:hjb_ou_20d_numexp}}. The results are based on $10$ value (Picard) iterations.}
    \label{tab:hjb_ou_20d_final}
    \small
    \begin{tabular}{lcccc}
        \toprule
        {Method} & {Value Error} & {Gradient Error} & {Hessian Error} & {Time (s)} \\
        \midrule
        DPI & $2.576\times10^{-2}$ & $7.429\times10^{-2}$ & $1.434\times10^{0}$ & $1301$ \\
        ZOD-m & $7.921\times10^{-3}$ & $6.757\times10^{-2}$ & $1.252\times10^{-1}$ & $1025$ \\
        \bottomrule
    \end{tabular}
\end{table}

\subsection{A Fully Nonlinear PDE}\label{sec:fully nonlinear}

We next consider a fully nonlinear PDE studied in \cite{han26}:
\begin{equation}\label{eq:fully_nonlinear_benchmark}
    \partial_t u(t,x)
    +\frac{1}{2}\Delta u(t,x)
    +\frac{1}{4}\sum_{i=1}^{d}
    \left|\frac{\partial^2 u}{\partial x_i^2}(t,x)\right|
    -h(t,x)=0,
    \qquad (t,x)\in[0,T)\times\mathbb R^d .
\end{equation}
The form of this PDE is reverse-engineered from the exact solution
\begin{equation}\label{eq:fullynonlinear exact solution}
    u^\star(t,x)=\sum_{j=1}^{J} v_j
    \sin\!\left(t+\sum_{i=1}^{d} w_i^j x_i\right),
\end{equation}
with
\begin{equation}
    h(t,x)=
    \partial_t u^\star(t,x)
    +\frac{1}{2}\Delta u^\star(t,x)
    +\frac{1}{4}\sum_{i=1}^{d}
    \left|\frac{\partial^2 u^\star}{\partial x_i^2}(t,x)\right|.
\end{equation}
The PDE \eqref{eq:fully_nonlinear_benchmark} has differential operator $\mathcal L u = \frac{1}{2}\Delta u$, and the source term $f(t,x,\nabla^2 u) = \frac{1}{4}\sum_{i=1}^{d}
\left|\frac{\partial^2 u}{\partial x_i^2}(t,x)\right|-h(t,x)$. Hence, it corresponds to the SDE
$\dd X_t = \dd W_t$.

We use $d=20$, $T=1.0$, $J=2$ and set $ X_0 = 0$. The parameters in \eqref{eq:fullynonlinear exact solution} are sampled once and then fixed throughout the experiment:
\[
    w_i^j\sim d^{-1/2}\mathcal N(0,1),
    \qquad
    v_j\sim\mathcal N(0,1).
\]

For the DPI baseline, we use a
configuration with data size $4096$, integral sample size $1024$, three hidden
layers of width $128$, $16$ epochs per iteration, and in total $40$ iterations.

For ZOD, we use the three-network model with $3$ hidden layers of width $64$,
ELU activations, ZOD bandwidth $\epsilon=0.05$. The derivative networks are pretrained for
$5000$ steps. To ensure  a fair comparison, we control the computational time of ZOD to be comparable to that of DPI. To this end, in our implementation of ZOD, we use batch size
$32768$ and $4096$ gradient steps per value iteration.

%Alternatively, to see  and a more accurate ZOD configuration with the same batch size and $8192$ gradient steps.  Both ZOD configurations use $10$ outer value iterations.

Figure~\ref{fig:fn_case1_iter_curve} compares the time-matched ZOD run with the
$40$-iteration DPI run.  The rapid decay of the ZOD errors in the first few
value (Picard) iterations is consistent with the contraction estimate in
Theorem~\ref{thm: global_convergence}.  After a small number of Picard
iterations, the curves approach stable accuracy levels.  This behavior agrees
with the structure of \eqref{eq: global_bound}: once the one-step learning error
has been reduced to a fixed floor, the geometrically weighted accumulation term
limits the additional gain from further iterations.  Meanwhile, DPI decreases more
gradually under the reported configuration.  In particular, directly learning
the Hessian through a second-order ZOD target gives a substantially lower
second-order error than differentiating a learned scalar value approximation. Table~\ref{tab:fully-nonlinear-final-errors} further reports the final errors and
runtimes.  ZOD  is more accurate than DPI for all the
three quantities, and the gap is the largest for the Hessian.  %Increasing the ZOD training budget from $4096$ to $8192$ gradient steps further lowers the gradient and Hessian rRMSEs at a moderate additional computational cost.

\begin{figure}[!htbp]
\centering
\includegraphics[width=\linewidth]{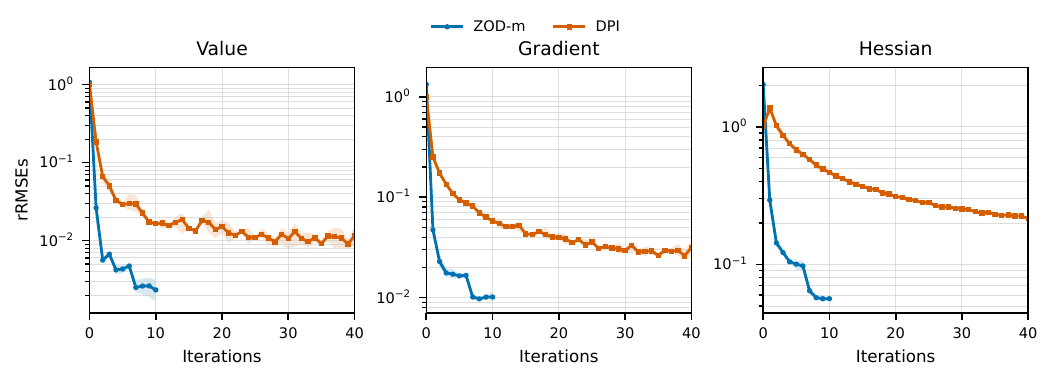}
\caption{\textbf{rRMSE of the learned value, gradient, and Hessian over iterations} for the fully nonlinear PDE \eqref{eq:fully_nonlinear_benchmark} with
$d=20$. The panels show value, gradient, and Hessian rRMSEs over
iterations for the time-matched ZOD configuration and DPI. Curves are
averaged over three random seeds, and shaded regions indicate empirical
$2.5\%$--$97.5\%$ quantile bands.}
\label{fig:fn_case1_iter_curve}
\end{figure}

\begin{comment}
Figure~\ref{fig:fully-nonlinear-10step-comparison} compares the first $10$
iterations of DPI with the two ZOD configurations.  The $4096$-step ZOD
run is the time-matched configuration, while the $8192$-step run represents the
more accurate ZOD configuration.  Increasing the number of gradient steps
reduces the one-step learning error in Theorem~\ref{thm: one step error bound}
and leads to a lower final error floor in Theorem~\ref{thm: global_convergence}.
Even under the time-matched budget, ZOD gives substantially smaller derivative
errors than DPI, with the clearest improvement appearing in the Hessian.

\begin{figure}[!htbp]
\centering
\includegraphics[width=\linewidth]{Figures/fully_nonlinear_10step_comparison.pdf}
\caption{Ten-iteration comparison for the fully nonlinear benchmark
\eqref{eq:fully_nonlinear_benchmark}. The panels show value, gradient, and
Hessian rRMSE histories for DPI, the time-matched ZOD run with $4096$ gradient
steps per iteration, and the more accurate ZOD run with $8192$ gradient
steps per iteration.}
\label{fig:fully-nonlinear-10step-comparison}
\end{figure}
\end{comment}

\begin{table}[!htbp]
\centering
\caption{\textbf{rRMSE for the fully nonlinear PDE
\eqref{eq:fully_nonlinear_benchmark}}. The results of ZOD are based on 10 value iterations,
while DPI is based on 40 iterations.}
\label{tab:fully-nonlinear-final-errors}
\small
\begin{tabular}{lcccc}
\toprule
{Method} & {Value Error} & {Gradient Error} & {Hessian Error} & {Time (s)} \\
\midrule
DPI & $1.155\times10^{-2}$ & $3.167\times10^{-2}$ & $2.160\times10^{-1}$ & $1553$ \\
ZOD-m & $2.333\times10^{-3}$ & $1.015\times10^{-2}$ & $5.649\times10^{-2}$ & $1886$ \\
\bottomrule
\end{tabular}
\end{table}

Finally, we examine two controlled ablations for our ZOD approach.
Table~\ref{tab:fully-nonlinear-ablation-summary} summarizes the impact of
training budget, batch size, and availability of the strong simulator.  In the training-budget panel, we vary the number of gradient
steps while fixing the batch size at $32768$ and  the bandwidth at
$\epsilon=0.05$.  In the batch-size panel, we
vary the batch size while fixing the training budget at $4096$ steps and the
bandwidth at $\epsilon=0.05$.  These two
panels demonstrate the expected monotone improvement as the training
budget and batch size increase,  consistent with the theoretical decomposition in Theorem \ref{thm: one step error bound}. Specifically, increasing the number of gradient steps reduces the optimization error in training the empirical loss (although we did not include the optimization error in Theorem~\ref{thm: one step error bound}). On the other hand, the batch-size ablation aligns directly with the theorem: a larger batch size decreases the statistical error term in the one-step bound, and therefore is expected to reduce the overall error. These results are also illustrated in Figure \ref{fig:fully-nonlinear-ablation-studies}.

The last panel of Table~\ref{tab:fully-nonlinear-ablation-summary} shows the effect of the strong simulator availability when everything else is the same. As predicted by the theoretical results, access to the strong simulator boosts performance.  In Appendix \ref{appendix:addition zod-1}, we further examine the performance of ZOD when only the weak simulator is available. Overall, ZOD-1 is slightly worse than ZOD-m, and  a small bandwidth is not always preferred with the former.

\begin{table}[!htbp]
\centering
\caption{\textbf{ZOD ablations for the fully nonlinear PDE
\eqref{eq:fully_nonlinear_benchmark}}. Each row reports the mean final rRMSE over
three random seeds.}
\label{tab:fully-nonlinear-ablation-summary}
\scriptsize
\resizebox{\textwidth}{!}{%
\begin{tabular}{lcccccccc}
\toprule
Setting & \tabincell{c}{Gradient\\Steps} & \tabincell{c}{Batch\\Size} & \tabincell{c}{Bandwidth\\$\epsilon$} & \tabincell{c}{Strong\\Simulator} & \tabincell{c}{Value\\Error} & \tabincell{c}{Gradient\\Error} & \tabincell{c}{Hessian\\Error} & \tabincell{c}{Time\\(s)} \\
\midrule
\multirow{3}{*}{\shortstack[l]{Varying\\training steps}} & 2048 & 32768 & 0.05 & \cmark & $3.601\times10^{-3}$ & $1.336\times10^{-2}$ & $7.736\times10^{-2}$ & $1006$ \\
 & 4096 & 32768 & 0.05 & \cmark & $2.333\times10^{-3}$ & $1.015\times10^{-2}$ & $5.649\times10^{-2}$ & $1886$ \\
 & 8192 & 32768 & 0.05 & \cmark & $2.032\times10^{-3}$ & $7.518\times10^{-3}$ & $4.083\times10^{-2}$ & $3665$ \\
\midrule
\multirow{3}{*}{\shortstack[l]{Varying\\batch size}} & 4096 & 16384 & 0.05 & \cmark & $3.016\times10^{-3}$ & $1.140\times10^{-2}$ & $6.713\times10^{-2}$ & $1061$ \\
 & 4096 & 32768 & 0.05 & \cmark & $2.333\times10^{-3}$ & $1.015\times10^{-2}$ & $5.649\times10^{-2}$ & $1886$ \\
 & 4096 & 65536 & 0.05 & \cmark & $1.997\times10^{-3}$ & $7.607\times10^{-3}$ & $4.575\times10^{-2}$ & $3636$ \\
\midrule
\multirow{2}{*}{\shortstack[l]{Strong simulator\\availability}} & 2048 & 65536 & 0.25 & \xmark & $6.398\times10^{-3}$ & $2.781\times10^{-2}$ & $6.192\times10^{-1}$ & $1330$ \\
 & 2048 & 65536 & 0.25 & \cmark & $3.687\times10^{-3}$ & $1.205\times10^{-2}$ & $5.178\times10^{-2}$ & $1819$ \\
% \midrule
% \multirow{4}{*}{\shortstack[l]{Varying \\bandwidth}} & 2048 & 65536 & 0.1 & \xmark & $3.854\times10^{-2}$ & $4.571\times10^{-2}$ & $2.320\times10^{0}$ & $1328$ \\
%  & 2048 & 65536 & 0.2 & \xmark & $9.834\times10^{-3}$ & $3.168\times10^{-2}$ & $9.591\times10^{-1}$ & $1326$ \\
%  & 2048 & 65536 & 0.25 & \xmark & $6.398\times10^{-3}$ & $2.781\times10^{-2}$ & $6.192\times10^{-1}$ & $1330$ \\
%  & 2048 & 65536 & 0.5 & \xmark & $4.255\times10^{-3}$ & $3.697\times10^{-2}$ & $2.703\times10^{-1}$ & $1327$ \\
% \midrule
% \multirow{3}{*}{\shortstack[l]{Varying \\batch size}} & 2048 & 65536 & 0.25 & \xmark & $6.398\times10^{-3}$ & $2.781\times10^{-2}$ & $6.192\times10^{-1}$ & $1330$ \\
%  & 2048 & 131072 & 0.25 & \xmark & $4.896\times10^{-3}$ & $2.481\times10^{-2}$ & $4.939\times10^{-1}$ & $2512$ \\
%  & 2048 & 262144 & 0.25 & \xmark & $5.000\times10^{-3}$ & $1.982\times10^{-2}$ & $3.955\times10^{-1}$ & $4880$ \\
\bottomrule
\end{tabular}
}
\end{table}

\begin{figure}[!htbp]
\centering
\includegraphics[width=0.8\linewidth]{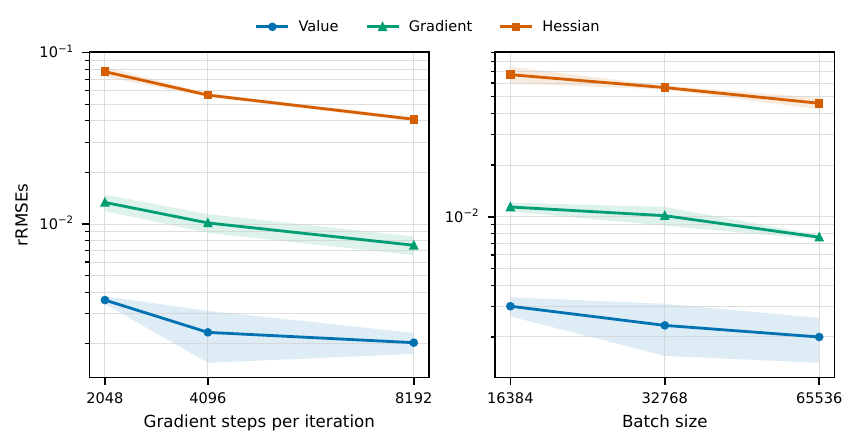}
\caption{\textbf{Ablations for the fully nonlinear PDE
\eqref{eq:fully_nonlinear_benchmark}.} All curves report the mean rRMSE over
3 seeds, and shaded regions indicate one empirical standard deviation.}
\label{fig:fully-nonlinear-ablation-studies}
\end{figure}

\section{Concluding Remarks}\label{conclusions}

In this paper we develop a data-driven learning method for solving a broad class of potentially high-dimensional, fully nonlinear black-box PDEs. The thrust of our analysis is to represent the derivatives of the solutions via ZODs and then learn them in the same way as learning the solutions, rather than outright differentiating the learned solutions. Thus, the learning ultimately corresponds to policy evaluation \citep{jia2022policy} in continuous-time reinforcement learning (RL). Indeed, we solve the nonlinear PDEs through iterations each of which is a linear PDE and hence a policy evaluation problem. This idea further inspires us to formulate learning the gradient and Hessian also as policy evaluation problems.\footnote{It should be noted that although the PDEs studied in this paper include {\it in form} the class of HJB equations as a special case, our assumption that the value of the nonlinear source term is known for each given input excludes the applicability to HJB in the model-free RL setting.
However, \cite{jia2022pg,jia2022q} decompose the RL into policy evaluation and policy improvement/q-learning, and show that the latter can be reformulated as policy evaluation. In this sense, black-box HJB equations can also be solved essentially by policy evaluation, albeit in a very different way from this paper. }

The approach developed in this paper hints on an important implication in the study of data-driven methodologies in machine learning: the design of algorithms should be tailored to the availability of the data-generating mechanisms or the properties of simulators. With strong simulators similar to the one introduced in this paper, various variance reduction tricks in simulations may be employed to improve the efficiency of the algorithms. This opens a gate to further studying its potential in other tasks.

% Acknowledgements and Disclosure of Funding should go at the end, before appendices and references

\acks{
%All acknowledgements go at the end of the paper before appendices and references. Moreover, you are required to declare funding (financial activities supporting the submitted work) and competing interests (related financial activities outside the submitted work). More information about this disclosure can be found on the JMLR website.
Huy\^en  Pham is supported by the Soci\'et\'e G\'en\'erale Chair ``Risques Financiers", FiME (Laboratory of Finance and Energy Markets), and the EDF–CACIB Chair ``Finance and Sustainable Development''. Yanwei Jia is supported by the University Start-up Fund at The Chinese University of Hong Kong. Xun Yu Zhou is supported by the Nie Center for Intelligent Wealth Management at Columbia. The authors thank Xuefeng Gao, Chun Liu, and Jiale Zha for discussions on the contents of this paper. 
}

% Manual newpage inserted to improve layout of sample file - not
% needed in general before appendices/bibliography.

\newpage

\appendix
\section{A Sufficient Condition for Assumption \ref{assum: contraction}}\label{appendix: sufficient condition for contraction}

We provide a concrete sufficient condition for Assumption \ref{assum: contraction}.

Consider the operator $ \cL v :=  \Delta v - \nabla h \cdot\nabla v$ that corresponds to the Langevin SDE:
\begin{equation}
    \dd X_t = -\nabla h(X_t) \dd t + \sqrt{2}\dd W_t.
\end{equation}
Consider the norm defined in Section \ref{sec:convergence analysis} with a weighting parameter $\beta$ to be specified later:
\begin{equation}
    \|v\|^2_{\beta,\textrm{mix}}: = \beta \|v\|_{\beta}^2 + \beta\|\nabla v\|_{\beta}^2 + \|\nabla^2 v\|_{\beta}^2.
\end{equation}
For a triplet $\mathbf U=(V,G,H)$, we denote the corresponding norm as
\begin{equation}
    \|\mathbf U\|^2_{\beta,\textrm{mix}}: = \beta \|V\|_{\beta}^2 + \beta\|G\|_{\beta}^2 + \|H\|_{\beta}^2,
\end{equation}
and denote the solution mapping as
\[\mathbf U \mapsto v = \mathcal S(\mathbf U),\text{ such that } \partial_t v + \cL v +f(x,V,G, H) = 0,\;\;v(T,x)=g(x) .  \]
For any two input triplets $\mathbf U_i=(V_i,G_i,H_i)$, $i=1,2$. Denote $ \delta f := f(x,\mathbf{U}_1) - f(x,\mathbf{U}_2)$, $ v := \mathcal{S}(\mathbf{U}_1) - \mathcal{S}(\mathbf{U}_2)$. Then $v$ satisfies the linear PDE
\begin{equation}\label{eq: linear pde langevin}
    \partial_t v + \cL v = -\delta f, \quad v(T,x) = 0.
\end{equation}
We aim to prove the contraction property through the bound
\begin{equation}
    \|v\|_{\beta,\textrm{mix}} \le \gamma\|\mathbf U_1-\mathbf U_2\|_{\beta,\textrm{mix}}
\end{equation}
for some constant $\gamma<1$. 
We need the following assumption in this section.
\begin{assumption}\label{assum: langevin contraction setting}
\begin{enumerate}[(i)]
    \item \label{item: smooth confining potential} $ h \in C^2_p(\R^d)$, the eigenvalues of $\nabla^2h$ have a uniform lower bound $-k$ for some $k\ge0$, and
    \begin{equation}\label{eq: increase condition}
        \lim_{|x|\to\infty}\frac{h(x)}{\log |x|}=\infty.
    \end{equation}
    \item The nonlinearity $f$ is Lipschitz in the value, gradient, and Hessian variables:
    \begin{equation}
        |f(x,V,G, H) - f(x,V',G', H')| \le L_v |V-V'| + L_g|G - G'| + L_h|H - H'|.
    \end{equation}
    \item  $v:= \mathcal{S}(\mathbf{U}_1)-\mathcal{S}(\mathbf{U}_2) \in C^{1,3}_p([0,T]\times\mathbb R^d)$.
\end{enumerate}
\end{assumption}

Under Assumption \ref{assum: langevin contraction setting} \eqref{item: smooth confining potential}, the potential $h$ is a smooth confining potential, and the Langevin dynamics admits the invariant density, see, e.g., \cite[Proposition 4.2]{pavliotis2014Stochastic},
    \begin{equation}\label{eq: density of stationary dist}
        p(x) :=\frac{1}{C}\exp\left\{ -h(x)\right\},
    \end{equation}
    where $C:=\int_{\R^d}\exp\{-h(x)\}\dd x<\infty$. 
    
    From now on, we take $X_0\sim p(\cdot)$, and then $X_t\sim p(\cdot)$ for all $t\ge0$. Consequently, the norms have the following explicit form:
    \begin{equation}
    \|v\|_{L_2(p)}^2 := \int_{\R^d} |v(t,x)|^2p(x)\dd x, \quad\|v\|_{\beta}^2 := \int_0^T e^{\beta t}\int_{\R^d} |v(t,x)|^2p(x)\dd x\dd t.
    \end{equation}

The following theorem provides a sufficient condition for the contraction property used in the main text.

\begin{theorem}\label{thm: langevin sufficient contraction}
    Suppose Assumption \ref{assum: langevin contraction setting} holds. If $12L_h^2<1$, then there exists a sufficiently large $\beta\ge k$ such that
    \begin{equation}\label{eq: eta beta sufficient condition}
        \eta_\beta
        :=12\left((L_v^2+L_g^2)\frac{\beta+1}{\beta^2}
        +L_h^2\frac{\beta+1}{\beta}\right)<1,
    \end{equation}
    and set $\gamma=\sqrt{\eta_\beta}$. Then
    \begin{equation}
        \|\mathcal S(\mathbf U_1)-\mathcal S(\mathbf U_2)\|_{\beta,\textrm{mix}}
        \le \gamma \|\mathbf U_1-\mathbf U_2\|_{\beta,\textrm{mix}}.
    \end{equation}
That is, Assumption \ref{assum: contraction} holds.
\end{theorem}

The rest of this appendix is devoted to proving Theorem \ref{thm: langevin sufficient contraction}. We first analyze the norm of the value and the gradient. The following lemma holds in a general diffusion setting, where $\mathcal L$ is the infinitesimal generator of the diffusion, $ \mathcal L v=b\cdot\nabla v+\frac{1}{2}\operatorname{Tr}\left(\sigma\sigma^{\top}\nabla^2 v\right)$.

\begin{lemma}
    Suppose Assumption \ref{assum: langevin contraction setting} holds. Assume that $\sigma\sigma^{\top}$ is uniformly elliptic, i.e., there is a constant $\lambda >0 $, such that $x^{\top} \sigma\sigma^{\top} x \ge \lambda |x|^2$. Then the value and gradient have the following upper bound:
    \begin{equation}\label{eq: value gradient bound}
    \frac{\beta}{2}\|v\|_{\beta}^2 + \lambda\|\nabla v\|_{\beta}^2 \le \frac{2}{\beta}\|\delta f\|_{\beta}^2.
\end{equation}
\end{lemma}
\begin{proof}
Denote $ Y_t:= e^{\beta t}v(t,X_t)^2$. It follows from Ito's formula that
\begin{equation}
    \begin{aligned}
        \dd Y_t = \beta e^{\beta t}v^2 \dd t + 2e^{\beta t}v(\partial_t + \mathcal{L})v \dd t + e^{\beta t}\nabla v^{\top}\sigma\sigma^{\top}\nabla v\dd t + \dd M_t,
    \end{aligned}
\end{equation}
where $ M$ is a martingale. Taking expectation on both sides and noting $ Y_T = 0, (\partial_t + \mathcal{L})v = -\delta f$, we have
\begin{equation}
    -\E[Y_0] = \E\left[\int_0^T \beta e^{\beta t}v^2 - 2e^{\beta t}v\delta f  + e^{\beta t}\nabla v^{\top} \sigma\sigma^{\top}\nabla v \dd t\right].
\end{equation}
Thus,
\begin{equation}
    \beta \|v\|_{\beta}^2 + \E\left[\int_0^Te^{\beta t}\nabla v^{\top} \sigma\sigma^{\top}\nabla v \dd t\right] = \E\left[\int_0^T2e^{\beta t}v\delta f\dd t\right] - \E[Y_0].
\end{equation}
It follows from the uniform ellipticity of $ \sigma\sigma^{\top}$ that $ \nabla v^{\top} \sigma\sigma^{\top}\nabla v \ge \lambda \|\nabla v\|^2$. This along with the inequality $ 2|v\delta f|\le \frac{\beta}{2}|v|^2 + \frac{2}{\beta}|\delta f|^2$ yields
\begin{equation}
    \beta\|v\|_{\beta}^2 + \lambda\|\nabla v\|_{\beta}^2 \le \frac{\beta}{2}\|v\|_{\beta}^2 + \frac{2}{\beta}\|\delta f\|_{\beta}^2-\E[Y_0].
\end{equation}
However,  $ Y_0\ge 0$; thus
\begin{equation}
    \frac{\beta}{2}\|v\|_{\beta}^2 + \lambda\|\nabla v\|_{\beta}^2 \le \frac{2}{\beta}\|\delta f\|_{\beta}^2.
\end{equation}
The desired result then follows.
\end{proof}

The preceding estimate controls the value component and provides an initial control on the gradient. To handle the Hessian component, we use the reversibility of the Langevin generator with respect to the stationary density.

\begin{lemma}\label{lemma: int by parts}
Suppose Assumption \ref{assum: langevin contraction setting} \eqref{item: smooth confining potential} holds. Let $r\in C^2_p(\R^d)$ and $g\in C^1_p(\R^d)$.  Then
    \begin{equation}
        \int_{\R^d} (\cL r) g(x) p(x) \dd x = -\int_{\R^d}\nabla r\cdot \nabla g p(x) \dd x.
    \end{equation}
\end{lemma}
\begin{proof}
Note that $ \nabla p = \frac{1}{C}\exp\left\{ -h(x)\right\} (-\nabla h) = -p\nabla h$. Thus,
\begin{equation}\label{eq: identity for p}
    \nabla p + p\nabla h = 0.
\end{equation}
    It follows from the integration by parts (the remainder terms vanish due to \eqref{eq: increase condition}) that
    \begin{equation}
        \begin{aligned}
            \int_{\R^d}(\cL r) g(x) p(x) \dd x &= \int_{\R^d} (\Delta r - \nabla h \cdot\nabla r) g(x) p(x) \dd x\\
            & = -\int_{\R^d} \nabla r \cdot \nabla (gp) \dd x - \int_{\R^d}\nabla h \cdot\nabla r g p \dd x\\
            & = -\int_{\R^d}\nabla r\cdot\nabla g p\dd x - \int_{\R^d} \nabla r\cdot (\nabla p + p\nabla h)g \dd x\\
            & = -\int_{\R^d} \nabla r\cdot \nabla g p\dd x.
        \end{aligned}
    \end{equation}
    where the last equality follows from the identity \eqref{eq: identity for p}.
\end{proof}
%The value and gradient bound can be derived by Ito's formula, and they satisfy the upper bound \eqref{eq: value gradient bound}.
Next, we turn to the Hessian bound. The following lemma represents the norm of the Hessian $ \nabla^2 v$.

\begin{lemma} Under Assumption \ref{assum: langevin contraction setting}, for $ \cL v = \Delta v - \nabla h\cdot \nabla v$, we have
\begin{equation}\label{eq: bochner}
    \frac{1}{2}\cL(|\nabla v|^2) - \nabla v \cdot \nabla(\cL v) = |\nabla^2 v|^2 + \nabla v^{\top} \nabla^2 h \nabla v.
\end{equation}
\end{lemma}
\begin{proof}
    Note that
    \begin{equation}
        \begin{aligned}
            \frac{1}{2}\Delta (|\nabla v|^2) = \frac{1}{2} \sum_{i=1}^d \partial_i^2 \sum_{j=1}^d (\partial_j v)^2 = \frac{1}{2}\sum_{i,j=1}^d \partial_i (2\partial_j v \partial_{ij}^2 v) &= \sum_{i,j=1}^d (\partial_{ij}^2 v \partial_{ij}^2 v + \partial_j v \partial_i (\partial_{ij}^2 v))\\
            & = |\nabla^2 v|^2 + \nabla v \cdot \nabla (\Delta v).
        \end{aligned}
    \end{equation}
    Thus,
    \begin{equation}\label{eq: mid}
        \begin{aligned}
        \frac{1}{2} \cL(|\nabla v|^2) - \nabla v \cdot \nabla (\cL v) &= \frac{1}{2} \Delta (|\nabla v|^2) - \frac{1}{2} \nabla h\cdot \nabla (|\nabla v|^2) - \nabla v \cdot \nabla (\Delta v - \nabla h \cdot \nabla v)\\
        & = \frac{1}{2}\Delta (|\nabla v|^2) -\nabla v \cdot \nabla (\Delta v) - \frac{1}{2}\nabla h \cdot \nabla (|\nabla v|^2) + \nabla v \cdot \nabla (\nabla h \cdot \nabla v)\\
        & = |\nabla^2 v|^2 - \frac{1}{2}\nabla h \cdot \nabla(|\nabla v|^2) + \nabla v \cdot \nabla(\nabla h \cdot \nabla v).
        \end{aligned}
    \end{equation}
We now compute the second and third terms, respectively. Indeed, %For the second term on the right hand side of \eqref{eq: mid},
\begin{equation}\label{eq: mid 2}
    \begin{aligned}
        -\frac{1}{2}\nabla h\cdot \nabla(|\nabla v|^2) = -\frac{1}{2}\sum_{i=1}^d \partial_i h\partial_i (\sum_{j=1}^d (\partial_j v)^2) = -\sum_{i,j=1}^d \partial_i h \partial_j v \partial_{ij}^2 v,
    \end{aligned}
\end{equation}
while
%For the third term on the right hand side of \eqref{eq: mid},
\begin{equation}\label{eq: mid 3}
    \begin{aligned}
        \nabla v \cdot \nabla(\nabla h \cdot \nabla v) = \sum_{i=1}^d \partial_i v \partial_i(\sum_{j=1}^d \partial_j h\partial_j v) &= \sum_{i,j=1}^d \partial_i v  \partial_{ij}^2 h \partial_j v + \sum_{i,j=1}^d \partial_i h \partial_j v \partial_{ij}^2 v\\
        & = \nabla v^{\top} \nabla^2 h \nabla v+ \sum_{i,j=1}^d \partial_i h \partial_j v \partial_{ij}^2 v.
    \end{aligned}
\end{equation}
The desired result follows from \eqref{eq: mid}, \eqref{eq: mid 2} and \eqref{eq: mid 3}.
\end{proof}

Integrating the above identity against the invariant density converts the local Hessian identity into an $L_2(p)$ estimate. The integration by parts formula above removes the total-divergence term.

\begin{lemma}\label{lemma: hessian bound mid}
    Under Assumption \ref{assum: langevin contraction setting}, we have
    \begin{equation}
        \|\nabla^2 v\|_{\beta}^2 = \|\cL v\|_{\beta}^2 - \int_{0}^T e^{\beta t} \int_{\R^d} \nabla v^{\top} \nabla^2 h\nabla v p \dd x\dd t.
    \end{equation}
\end{lemma}
\begin{proof}
    It follows from the formula \eqref{eq: bochner} that
    \begin{equation}\label{eq: bochner int R}
        \begin{aligned}
            \int_{\R^d} \frac{1}{2} \cL(|\nabla v|^2) p\dd x - \int_{\R^d} \nabla v\cdot \nabla(\cL v) p \dd x = \int_{\R^d} |\nabla^2 v|^2 p \dd x + \int_{\R^d}\nabla v^{\top}\nabla^2 h\nabla v p\dd x.
        \end{aligned}
    \end{equation}
    Lemma \ref{lemma: int by parts} yields that the first term on the left hand side of the above equation equals $ 0$:
    \begin{equation}
        \int_{\R^d} \frac{1}{2}\cL(|\nabla v|^2) p \dd x = -\int_{\R^d} \nabla (|\nabla v|^2) \cdot \nabla \left(\frac{1}{2}\right) p \dd x = 0.
    \end{equation}
    The second term can be written as
    \begin{equation}
        -\int_{\R^d} \nabla v\cdot \nabla(\cL v) p\dd x = \int_{\R^d}(\cL v)^2 p \dd x.
    \end{equation}
    Therefore, integrating \eqref{eq: bochner int R} over $ [0,T]$ with the weight $ e^{\beta t}$, we have
    \begin{equation}
        \|\nabla^2 v\|_{\beta}^2 = \|\cL v\|_{\beta}^2 - \int_0^T e^{\beta t} \int_{\R^d} \nabla v^{\top}\nabla^2 h\nabla v p\dd x \dd t .
    \end{equation}
    The desired result follows.
\end{proof}

Lemma \ref{lemma: hessian bound mid} gives the unfinished bound of Hessian $ \nabla^2 v$, which contains the norm of $ \cL v$. To bound $ \cL v$, we have to analyze $ \partial_t v$.

\begin{lemma}\label{lemma: partial t bound}
    Under Assumption \ref{assum: langevin contraction setting}, we have
    \begin{equation}
        \|\partial_t v\|_{\beta}^2 + \beta\|\nabla v\|_{\beta}^2 \le \|\delta f\|_{\beta}^2.
    \end{equation}
\end{lemma}
\begin{proof}
    Multiply \eqref{eq: linear pde langevin} by $\partial_t v$ and integrate over $\mathbb{R}^d$ with weight $p$ to get
    \begin{equation}\label{eq: partial t v step 1}
        \begin{aligned}
            \int_{\R^d} (\partial_t v)^2 p \dd x + \int_{\R^d} \cL v \partial_t v p \dd x = \int_{\R^d} -\delta f\partial_t vp \dd x.
        \end{aligned}
    \end{equation}
    For the second term on the left hand side of the above equation, by Lemma \ref{lemma: int by parts},
    \begin{equation}
        \int_{\R^d} \cL v \partial_t v p \dd x = -\int_{\R^d} \nabla v \cdot \nabla (\partial_t v) p\dd x = -\frac{1}{2}\frac{\dd }{\dd t}\int_{\R^d} |\nabla v|^2 p\dd x.
    \end{equation}
    Integrating over $ [0,T]$ with weight $ e^{\beta t}$ and using integration by parts with respect to $ t$, we obtain
    \begin{equation}
    \begin{aligned}
        \int_0^T e^{\beta t} -\frac{1}{2}\frac{\dd }{\dd t}\int_{\R^d} |\nabla v|^2 p\dd x \dd t &= -\frac{1}{2} \int_0^T \frac{\dd }{\dd t}\left(e^{\beta t} \int_{\R^d} |\nabla v|^2 p\dd x\right) \dd t \\
        &\qquad+ \frac{\beta}{2}\int_0^T   e^{\beta t} \int_{\R^d} |\nabla v|^2 p\dd x  \dd t\\
        & = \frac{1}{2}\int_{\R^d} |\nabla v(0,\cdot)|^2 p\dd x + \frac{\beta}{2}\|\nabla v\|_{\beta}^2 \ge \frac{\beta}{2}\|\nabla v\|_{\beta}^2,
        \end{aligned}
    \end{equation}
    where in the second equality, we used the fact that $ v(T,x) = 0$. This along with \eqref{eq: partial t v step 1} yields
    \begin{equation}
        \|\partial_t v\|_{\beta}^2 + \frac{\beta}{2}\|\nabla v\|_{\beta}^2 \le \int_0^T e^{\beta t}\int_{\R^d} |\delta f \partial_t v| p \dd x\dd t \le \frac{1}{2}\|\delta f\|_{\beta}^2 + \frac{1}{2} \|\partial_t v\|_{\beta}^{2}.
    \end{equation}
    This proves the desired result.
\end{proof}
Combining these lemmas, we have the upper bound for the norm of Hessian $ \nabla^2 v$.
\begin{theorem}\label{thm: hessian bound}
    Under Assumption \ref{assum: langevin contraction setting},
    \begin{equation}
        \|\nabla^2 v\|_{\beta}^2 \le 4\|\delta f\|_{\beta}^2 - (2\beta - k)\|\nabla v\|_{\beta}^2.
    \end{equation}
\end{theorem}
\begin{proof}
    It follows from Lemma \ref{lemma: hessian bound mid} and Lemma \ref{lemma: partial t bound} that
    \begin{equation}
    \begin{aligned}
        \|\nabla^2 v\|_{\beta}^2 &= \|\cL v\|_{\beta}^2 - \int_{0}^T e^{\beta t} \int_{\R^d} \nabla v^{\top} \nabla^2 h\nabla v p \dd x \dd t \\
        &\le 2\|\partial_t v\|_{\beta}^2 + 2\|\delta f\|_{\beta}^2 - \int_{0}^T e^{\beta t} \int_{\R^d} \nabla v^{\top} \nabla^2 h\nabla v p \dd x \dd t\\
        &\le 4\|\delta f\|_{\beta}^2 - 2\beta \|\nabla v\|_{\beta}^2- \int_{0}^T e^{\beta t} \int_{\R^d} \nabla v^{\top} \nabla^2 h\nabla v p \dd x \dd t.
        \end{aligned}
    \end{equation}
    Since the eigenvalues of $ \nabla^2 h$ have a uniform lower bound $ -k$, $ \nabla^2 h + kI$ is positive semi-definite. Therefore, $$ \int_{0}^T e^{\beta t}\int_{\R^d} \nabla v^{\top} \nabla^2 h\nabla v p \dd x \dd t \ge -k \|\nabla v\|_{\beta}^2,$$
    which yields
    \begin{equation}
        \|\nabla^2 v\|_{\beta}^2 \le 4\|\delta f\|_{\beta}^2 - (2\beta - k) \|\nabla v\|_{\beta}^2.
    \end{equation}
    This proves the desired result.
\end{proof}

We now finish the proof of the sufficient condition by relating $\delta f$ to the distance between the two input triplets.

\begin{proof}[Proof of Theorem \ref{thm: langevin sufficient contraction}]
For the Langevin dynamics considered here, \eqref{eq: value gradient bound} holds with $\lambda=2$. Hence
\begin{equation}
    \beta\|v\|_{\beta}^2  \le \frac{4}{\beta}\|\delta f\|_{\beta}^2,\quad  \beta\|\nabla v\|_{\beta}^2 \le \|\delta f\|_{\beta}^2.
\end{equation}
Together with Theorem \ref{thm: hessian bound}, this gives
\begin{equation}
\begin{aligned}
    \|v\|^2_{\beta,\textrm{mix}} &= \beta\|v\|_{\beta}^2 + \beta\|\nabla v\|_{\beta}^2 + \|\nabla^2 v\|_{\beta}^2\\
    &\le 4\|\delta f\|_{\beta}^2 - (2\beta -k)\|\nabla v\|_{\beta}^2 + \beta\|\nabla v\|_{\beta}^2 + \frac{4}{\beta}\|\delta f\|_{\beta}^2\\
    & \le 4(1+\frac{1}{\beta}) 3\left(L_v^2 \|V_1-V_2\|_{\beta}^2 + L_g^2 \|G_1-G_2\|_{\beta}^2 + L_h^2\|H_1-H_2\|_{\beta}^2\right) - (\beta - k)\|\nabla v\|_{\beta}^2\\
    &\le 12\left((L_v^2+L_g^2)\frac{\beta+1}{\beta^2}+L_h^2\frac{\beta+1}{\beta}\right)\|\mathbf U_1-\mathbf U_2\|^2_{\beta,\textrm{mix}} -  (\beta - k)\|\nabla v\|_{\beta}^2.
\end{aligned}
\end{equation}
For any $\beta\ge k$, the last term is nonpositive, and therefore
\begin{equation}
    \|v\|^2_{\beta,\textrm{mix}}
    \le \eta_\beta \|\mathbf U_1-\mathbf U_2\|^2_{\beta,\textrm{mix}},
\end{equation}
where $\eta_\beta$ is defined in \eqref{eq: eta beta sufficient condition}. If $\eta_\beta<1$, taking square roots gives the contraction inequality with $\gamma=\sqrt{\eta_\beta}<1$. Finally, if $12L_h^2<1$, then $\eta_\beta\to 12L_h^2$ as $\beta\to\infty$, so such a choice of $\beta$ is available.
\end{proof}

\section{Addition Numerical Results with Weak Simulator}
\label{appendix:addition zod-1}
We examine the performance of our approach when only the weak simulator is available on the fully nonlinear PDE example in Subsection \ref{sec:fully nonlinear}, as a robustness check. The results are summarized in Table \ref{tab:fully-nonlinear-ablation-summary zod-1}. Overall, consistent with our theoretical results, the performance of ZOD-1 is slightly worse than that with ZOD-m. Indeed, in this case it is no longer true that the smaller  bandwidth the better performance.

\begin{table}[!htbp]
\centering
\caption{\textbf{Summary of ZOD ablations for the fully nonlinear PDE
\eqref{eq:fully_nonlinear_benchmark}}. Each row reports the mean final rRMSE over
three random seeds.}
\label{tab:fully-nonlinear-ablation-summary zod-1}
\scriptsize
\resizebox{\textwidth}{!}{%
\begin{tabular}{lcccccccc}
\toprule
Setting & \tabincell{c}{Gradient\\Steps} & \tabincell{c}{Batch\\Size} & \tabincell{c}{Bandwidth\\$\epsilon$} & \tabincell{c}{Strong\\Simulator} & \tabincell{c}{Value\\Error} & \tabincell{c}{Gradient\\Error} & \tabincell{c}{Hessian\\Error} & \tabincell{c}{Time\\(s)} \\
\midrule
% \multirow{3}{*}{\shortstack[l]{Varying\\training steps}} & 2048 & 32768 & 0.05 & \cmark & $3.601\times10^{-3}$ & $1.336\times10^{-2}$ & $7.736\times10^{-2}$ & $1006$ \\
%  & 4096 & 32768 & 0.05 & \cmark & $2.333\times10^{-3}$ & $1.015\times10^{-2}$ & $5.649\times10^{-2}$ & $1886$ \\
%  & 8192 & 32768 & 0.05 & \cmark & $2.032\times10^{-3}$ & $7.518\times10^{-3}$ & $4.083\times10^{-2}$ & $3665$ \\
% \midrule
% \multirow{3}{*}{\shortstack[l]{Varying\\batch size}} & 4096 & 16384 & 0.05 & \cmark & $3.016\times10^{-3}$ & $1.140\times10^{-2}$ & $6.713\times10^{-2}$ & $1061$ \\
%  & 4096 & 32768 & 0.05 & \cmark & $2.333\times10^{-3}$ & $1.015\times10^{-2}$ & $5.649\times10^{-2}$ & $1886$ \\
%  & 4096 & 65536 & 0.05 & \cmark & $1.997\times10^{-3}$ & $7.607\times10^{-3}$ & $4.575\times10^{-2}$ & $3636$ \\
% \midrule
% \multirow{2}{*}{\shortstack[l]{Simulator\\Availability}} & 2048 & 65536 & 0.25 & \xmark & $6.398\times10^{-3}$ & $2.781\times10^{-2}$ & $6.192\times10^{-1}$ & $1330$ \\
%  & 2048 & 65536 & 0.25 & \cmark & $3.687\times10^{-3}$ & $1.205\times10^{-2}$ & $5.178\times10^{-2}$ & $1819$ \\
% \midrule
\multirow{4}{*}{\shortstack[l]{Varying \\bandwidth}} & 2048 & 65536 & 0.1 & \xmark & $3.854\times10^{-2}$ & $4.571\times10^{-2}$ & $2.320\times10^{0}$ & $1328$ \\
 & 2048 & 65536 & 0.2 & \xmark & $9.834\times10^{-3}$ & $3.168\times10^{-2}$ & $9.591\times10^{-1}$ & $1326$ \\
 & 2048 & 65536 & 0.25 & \xmark & $6.398\times10^{-3}$ & $2.781\times10^{-2}$ & $6.192\times10^{-1}$ & $1330$ \\
 & 2048 & 65536 & 0.5 & \xmark & $4.255\times10^{-3}$ & $3.697\times10^{-2}$ & $2.703\times10^{-1}$ & $1327$ \\
\midrule
\multirow{3}{*}{\shortstack[l]{Varying \\batch size}} & 2048 & 65536 & 0.25 & \xmark & $6.398\times10^{-3}$ & $2.781\times10^{-2}$ & $6.192\times10^{-1}$ & $1330$ \\
 & 2048 & 131072 & 0.25 & \xmark & $4.896\times10^{-3}$ & $2.481\times10^{-2}$ & $4.939\times10^{-1}$ & $2512$ \\
 & 2048 & 262144 & 0.25 & \xmark & $5.000\times10^{-3}$ & $1.982\times10^{-2}$ & $3.955\times10^{-1}$ & $4880$ \\
\bottomrule
\end{tabular}
}
\end{table}
\section{Proof of Statements}
\subsection{Proof of Lemma \ref{lemma:zod classical}}
\begin{proof}
    We prove the $ d= 1$ case for illustration. It follows from Taylor's expansion that
    \begin{equation}\label{eq: taylor 1 der}
        \frac{Z}{\epsilon}u(t,x+\epsilon Z) = \frac{Z}{\epsilon}\left( u(t,x) + \nabla u(t,x)\epsilon Z + \frac{1}{2}\nabla^2u(t,x)\epsilon^2Z^2 + c(t,x,x + \epsilon Z)\epsilon^3Z^3\right).
    \end{equation}
    Thus,
    \begin{equation}
        \begin{aligned}
            \left|\E\left[ \frac{Z}{\epsilon} u(t,x+\epsilon Z)  \right]- \nabla u(t,x)\right|  &= \left| \E\left[ \frac{1}{2}\nabla^2 u(t,x)\epsilon Z^3 + c(t,x,x+\epsilon Z) \epsilon^2  Z^4\right]\right|\\
            & = \left| \E[c(t,x,x+\epsilon Z)Z^4]\right|\epsilon^2\le C\epsilon^2.
        \end{aligned}
    \end{equation}
    Similarly, if $ u$ has the fourth-order Taylor expansion, then
    \begin{equation}\label{eq: taylor 2 der}
        \begin{aligned}
            &\frac{ZZ^{\top}- I_d}{\epsilon^2}u(t,x+\epsilon Z) \\
            &=   \frac{Z^2 - 1}{\epsilon^2}u(t,x) + \frac{Z^3 - Z}{\epsilon}\nabla u(t,x) + \frac{1}{2}(Z^4 - Z^2)\nabla^2u(t,x)\\
            &\quad\quad\quad+ \frac{1}{6}\epsilon \nabla^3u(t,x)(Z^5 - Z^3)+c'(t,x,x+\epsilon Z)(Z^6 - Z^4)\epsilon^2.
        \end{aligned}
    \end{equation}
    Consequently,
    \begin{equation}
        \left|\E\left[\frac{ZZ^{\top}- I_d}{\epsilon^2}u(t,x+\epsilon Z) \right]- \nabla^2 u(t,x)\right| =  \left|\E[c'(t,x,x+\epsilon Z)(Z^6 - Z^4)]\right|\epsilon^2 \le C\epsilon^2.
    \end{equation}
    The variance parts follow from \eqref{eq: taylor 1 der}, \eqref{eq: taylor 2 der}  and the fact that $ c, c'$ have polynomial growth.
\end{proof}

\subsection{Proof of Lemma \ref{lemma: zod multi-point gradient}}
\begin{proof}
For notational simplicity, we suppress the variable $t$. The Taylor expansions of $u(x \pm \epsilon Z)$ are:
\begin{align*}
u(x+\epsilon Z) &= u(x) + \epsilon \nabla u^\top Z + \frac{\epsilon^2}{2} Z^\top \nabla^2 u(x) Z + O(\epsilon^3), \\
u(x-\epsilon Z) &= u(x) - \epsilon \nabla u^\top Z + \frac{\epsilon^2}{2} Z^\top \nabla^2 u(x) Z + O(\epsilon^3).
\end{align*}
Subtracting the second expansion from the first cancels all even-order terms:
\[ u(x+\epsilon Z) - u(x-\epsilon Z) = 2\epsilon \nabla u^\top Z + O(\epsilon^3). \]
Substituting this into the estimator definition in Eq.~\eqref{eq:symmetric_gradient_estimator}:
\[ \hat{g}(x; \epsilon) = \frac{2\epsilon \nabla u^\top Z + O(\epsilon^3)}{2\epsilon} Z = (\nabla u^\top Z)Z + O(\epsilon^2). \]
The expectation is $\E[(\nabla u^\top Z)Z] = \nabla u$, since $\E[ZZ^\top] = I_d$. The bias is therefore dominated by the $O(\epsilon^2)$ term.

For the variance, observe that the leading term of the estimator, $(\nabla u^\top Z)Z$, is independent of $\epsilon$. Its variance is a constant determined by the moments of $Z$ and the magnitude of $\nabla u$:
\[ \lim_{\epsilon \to 0} \Var[\hat{g}(x; \epsilon)] = \Var[(\nabla u^\top Z)Z]. \]
This variance is finite and does not depend on $\epsilon$. Thus, the total variance is bounded.
\end{proof}

\subsection{Proof of Lemma \ref{lemma: zod multi-point hessian}}
\begin{proof}
We suppress the variable $t$. Using the Taylor expansions up to the fourth order, the central difference numerator becomes:
\[ u(x+\epsilon Z) - 2u(x) + u(x-\epsilon Z) = \epsilon^2 Z^\top \nabla^2u Z + \frac{\epsilon^4}{12} \sum_{1\le i,j,k,l\le d}\frac{\partial^4 u(x)}{\partial x_i \partial x_j \partial x_k \partial x_l}Z_iZ_jZ_kZ_l + o(\epsilon^4). \]
 The symmetric construction cancels all the odd-order terms. Substituting this into Eq.~\eqref{eq:symmetric_hessian_estimator} yields:
\begin{equation}\label{eq: 3-point hessian}
\hat{H}(x; \epsilon) = \left( \frac{1}{2} Z^\top \nabla^2u Z + \frac{\epsilon^2}{24} \sum_{1\le i,j,k,l\le d}\frac{\partial^4 u(x)}{\partial x_i \partial x_j \partial x_k \partial x_l}Z_iZ_jZ_kZ_l + o(\epsilon^2) \right) (ZZ^\top - I_d). \end{equation}
Thus, the variance is bounded when $\epsilon$ tends  to $0$. For the bias part, we take the expectation of the leading term. Using the property of Gaussian moments that $\E[(Z^\top \nabla^2u Z)(ZZ^\top - I_d)] = 2H$ for a symmetric matrix $\nabla^2u$ (derived from Isserlis' theorem), we find
\begin{align*}
    \E[\hat{H}(x; \epsilon)] &= \frac{1}{2} \E\left[ (Z^\top \nabla^2u Z) (ZZ^\top - I_d) \right] + O(\epsilon^2) \\
    &= \frac{1}{2} (2H) + O(\epsilon^2) = \nabla^2u + O(\epsilon^2).
\end{align*}
The bias is therefore of order $O(\epsilon^2)$.
\end{proof}

\subsection{Proof of Proposition \ref{prop: zod variance reduction}}
 \begin{proof}
    The proof of the results for the weak simulator case is the same as the Proposition \ref{prop:zod one-point}. We only prove the strong simulator case.

     For simplicity, we use $ R(t,x)$ to represent $ R^{t,x}(\mathbf{U}_n)$. We only prove the variance result for $ \hat h$; the other parts are similar.  Due to the regularity conditions of $ b$ and $ \sigma$, it follows from \citet[Theorem 40]{protter2005Stochastic} and \citet[Theorem 3.4.2]{kunita2019Stochastic} that $ X_s^{t,x}$ is four times continuously differentiable in $ x$, and for every $ p>2$, there is a constant $ C_p>0$ such that for any multi-index $ \bm i$ with $ 0<|\bm i|\le 4$, it holds that
     \begin{equation}\label{eq: bounded moments of derivatives}
         \sup_{x\in\R^d,\ 0\le t\le T}\E\left[ \sup_{t\le s\le T}|\partial^{\bm i}_x X_s^{t,x}|^p\right] \le C_p.
     \end{equation}
     Therefore, $ R(t,x)$ is also four  times continuously differentiable in $x$ and admits the Taylor expansion
     \begin{equation}
     \begin{aligned}
         R(t,x+\epsilon Z) &= R(t,x) + \epsilon \nabla R(t,x)^{\top}Z + \frac{\epsilon^2}{2}Z^{\top}\nabla^2R(t,x)Z + \frac{\epsilon^3}{6} \nabla^3 R(t,x)(Z,Z,Z) \\
         &\qquad+ \frac{\epsilon^4}{24} \nabla^4R(t,x+\theta_1 \epsilon Z)(Z,Z,Z,Z),\\
         R(t,x-\epsilon Z) &= R(t,x) - \epsilon \nabla R(t,x)^{\top}Z + \frac{\epsilon^2}{2}Z^{\top}\nabla^2R(t,x)Z - \frac{\epsilon^3}{6} \nabla^3 R(t,x)(Z,Z,Z) \\
         &\qquad+ \frac{\epsilon^4}{24} \nabla^4R(t,x-\theta_2 \epsilon Z)(Z,Z,Z,Z),
         \end{aligned}
     \end{equation}
     where $ \theta_1,\theta_2 \in (0,1)$ and
     \begin{equation}
     \begin{aligned}
         \nabla^3R(t,x)(Z,Z,Z)&: = \sum_{1\le i,j,k\le d} \frac{\partial^3_x R(t,x)}{\partial x_i\partial x_j \partial x_k}Z_iZ_jZ_k,\\
         \nabla^4R(t,x)(Z,Z,Z,Z)&: = \sum_{1\le i,j,k,l\le d}\frac{\partial^4_x R(t,x)}{\partial x_i\partial x_j \partial x_k \partial x_l}Z_iZ_jZ_kZ_l.
         \end{aligned}
     \end{equation}
    Denote $ \hat h(t,x) := \frac{ZZ^{\top} - I_d}{2 \epsilon^2}\left( R(t,x+\epsilon Z) + R(t,x - \epsilon Z) - 2R(t,x)\right)$. Then
    \begin{equation}
        \begin{aligned}
            \hat h(t,x)  =& \frac{ZZ^{\top} - I_d}{2} Z^{\top}\nabla^2 R(t,x)Z \\
            &+ \frac{\epsilon^2}{48} (ZZ^{\top}- I_d)\left[ \nabla^4 R(t,x+\theta_1 \epsilon Z) + \nabla^4 R(t,x - \theta_2 \epsilon Z)\right](Z,Z,Z,Z).
        \end{aligned}
    \end{equation}
     It follows from the regularity conditions of $ g, \tilde f$ and \eqref{eq: bounded moments of derivatives} that for every $ p>0$ there are constants $ C_p>0 $ and  $ q>0$ such that for all $ t,x \in [0,T]\times \R^d$,
     \begin{equation}
         \E\left[ |\nabla^2R(t,x)|^p\right] \le C_p(1 + |x|^q),\quad \E\left[ |\nabla^4 R(t,x)|^p\right] \le C(1+|x|^q).
     \end{equation}
     This along with the moments bound of $ X_t$ yields
     \begin{equation}
     \begin{aligned}
         \E[ |\hat h|^2] &\le 2\E\left[ \left| \frac{ZZ^{\top} - I_d}{2} Z^{\top}\nabla^2 R(t,X_t)Z\right|^2\right]\\
         &\qquad+ 2\E\left[ \left| \frac{\epsilon^2}{48} (ZZ^{\top}- I_d)\left[ \nabla^4 R(t,X_t+\theta_1 \epsilon Z) + \nabla^4 R(t,X_t - \theta_2 \epsilon Z)\right](Z,Z,Z,Z)\right|^2\right]\\
         &\le \E\left[ \left| ZZ^{\top} - I_d\right|^2\left| \nabla^2 R(t,X_t)\right|^2 \left| Z\right|^4\right]\\
         &\qquad+ \frac{\epsilon^2}{24}\E\left[ \left| ZZ^{\top} - I_d\right|^2\left| \nabla^4 R(t,X_t+\theta_1 \epsilon Z) + \nabla^4 R(t,X_t - \theta_2 \epsilon Z)\right|^2\left| Z\right|^8\right]\\
         &\le C,
     \end{aligned}
     \end{equation}
     for some constant $ C>0$. This proves the desired result.
 \end{proof}

\subsection{Proof of Theorem \ref{thm: global_convergence}}
\begin{proof}
By the triangle inequality, we decompose the total error at step $n$ into the computational error introduced at the current step and the propagated error from the previous step:
\begin{equation}
    \|\mathbf{U}_n - \mathbf{u}^*\|_{\beta, \text{mix}} \le \|\mathbf{U}_n - \mathcal{S}(\mathbf{U}_{n-1})\|_{\beta, \text{mix}} + \|\mathcal{S}(\mathbf{U}_{n-1}) - \mathbf{u}^*\|_{\beta, \text{mix}}.
\end{equation}
The first term on the right-hand side is exactly the definition of the one-step error $E_n$. For the second term, using the fixed-point property $\mathbf{u}^* = \mathcal{S}(\mathbf{u}^*)$ and the Contraction Assumption \ref{assum: contraction}, we have:
\begin{equation}
    \|\mathcal{S}(\mathbf{U}_{n-1}) - \mathbf{u}^*\|_{\beta, \text{mix}} = \|\mathcal{S}(\mathbf{U}_{n-1}) - \mathcal{S}(\mathbf{u}^*)\|_{\beta, \text{mix}} \le \gamma \|\mathbf{U}_{n-1} - \mathbf{u}^*\|_{\beta, \text{mix}}.
\end{equation}
Combining these inequalities yields the recursive relationship:
\begin{equation}
    \|\mathbf{U}_n - \mathbf{u}^*\|_{\beta, \text{mix}} \le \gamma \|\mathbf{U}_{n-1} - \mathbf{u}^*\|_{\beta, \text{mix}} + E_n.
\end{equation}
Thus,
\begin{align*}
    \|\mathbf{U}_n - \mathbf{u}^*\|_{\beta, \text{mix}} &\le \gamma \left( \gamma \|\mathbf{U}_{n-2} - \mathbf{u}^*\|_{\beta, \text{mix}} + E_{n-1} \right) + E_n \nonumber \\
    &= \gamma^2 \|\mathbf{U}_{n-2} - \mathbf{u}^*\|_{\beta, \text{mix}} + \gamma E_{n-1} + E_n \nonumber \\
    &\le \gamma^n \|\mathbf{U}_0 - \mathbf{u}^*\|_{\beta, \text{mix}} + \sum_{k=1}^n \gamma^{n-k} E_k.
\end{align*}
This proves the desired result.
\end{proof}

\subsection{Proof of Lemma \ref{lemma: total loss decomp}}
\begin{proof}
    We derive the result for the gradient component only; the others follow similarly. Let $Y_g = \frac{Z}{\epsilon}R^{t, X_t+\epsilon Z}$ denote the stochastic gradient target. The gradient loss is:
    \begin{equation}
        L_g(\varphi_g,\mathbf{U}_{n-1}) = \beta \E\left[ \int_0^T e^{\beta t}|Y_g - \varphi_g(t,X_t)|^2 \dd t\right].
    \end{equation}
    Note that
    \begin{align*}
        \E \left[ \left| \varphi_g(t,X_t) - Y_g \right|^2 \right] = \E \left[ \E\left[ |\varphi_g(t,X_t) - G^\epsilon(t,X_t) + G^\epsilon(t,X_t) - Y_g |^2 \big| t,X_t \right] \right].
    \end{align*}
    Since $G^\epsilon(t,X_t) = \E[Y_g|t,X_t]$, the cross-term in the expansion of the square above vanishes. Thus
    \begin{align*}
        L_g(\varphi_g,\mathbf{U}_{n-1}) &= \beta\E \left[ \int_0^T e^{\beta t}|\varphi_g(t,X_t) - G^\epsilon(t,X_t)|^2 \dd t\right] + \beta \E\left[\int_0^Te^{\beta t} |G^\epsilon(t,X_t) - Y_g|^2 \dd t\right] \\
        &=: \beta\|\varphi_g - G^{\epsilon}\|_{\beta}^2 + C,
    \end{align*}
    where $ C$ is a constant independent of $ \varphi$. Summing the components for $V, G, H$ yields the result.
\end{proof}

\subsection{Proof of Theorem \ref{thm: one step error bound}}\label{proof sec: one step error bound}
We first prove Theorem \ref{thm: one step error bound}-(a), the case with the weak simulator.
For simplicity, we suppress the notation $\mathbf{U}_{n-1}$ and use $L(\varphi), \tilde L^{(B)}(\varphi)$ for total loss and empirical total loss, respectively. The following lemma provides the upper bound of $E_n$ in terms of the statistical error, approximation error and the bias due to the ZOD estimator.
\begin{lemma}\label{lemma: one step error decom}
    The one step approximation error $E_n = \normmix{\mathbf{U}_n - \mathcal{S}(\mathbf{U}_{n-1})}$ has the upper bound
    \begin{equation}
        E_n^2 \le 4\sup_{\varphi \in \mathcal{H}}|L(\varphi) - \tilde L^{(B)}(\varphi)| + 2 \inf_{\varphi\in \mathcal{H}} \normmix{\varphi - \mathcal{S}^\epsilon(\mathbf{U}_{n-1})}^2+  2 \normmix{\mathcal{S}^\epsilon(\mathbf{U}_{n-1}) - \mathcal{S}(\mathbf{U}_{n-1})}^2.
    \end{equation}
\end{lemma}
\begin{proof}
    Denote $\bU^{\mathcal{H}}_n \in \arg\min_{\varphi\in \mathcal{H}} L(\varphi)$. It follows from the decomposition in Lemma \ref{lemma: total loss decomp} that
    \begin{equation}\label{eq: one step error decomp}
    \begin{aligned}
        E_n^2  &= \normmix{\bU_n - \mathcal{S}(\mathbf{U}_{n-1})}^2 \\
        &\le 2\normmix{\bU_n - \mathcal{S}^\epsilon(\mathbf{U}_{n-1})}^2 + 2\normmix{\mathcal{S}^\epsilon(\mathbf{U}_{n-1}) - \mathcal{S}(\mathbf{U}_{n-1})}^2\\
        & \le 2(L(\bU_n) - L(\mathcal{S}^\epsilon(\mathbf{U}_{n-1}))) + 2\normmix{\mathcal{S}^\epsilon(\mathbf{U}_{n-1}) - \mathcal{S}(\mathbf{U}_{n-1})}^2\\
        &\le 2(L(\bU_n) - L(\bU^{\mathcal{H}}_n)) + 2(L(\bU^{\mathcal{H}}_n) -  L(\mathcal{S}^\epsilon(\mathbf{U}_{n-1}))) + 2\normmix{\mathcal{S}^\epsilon(\mathbf{U}_{n-1}) - \mathcal{S}(\mathbf{U}_{n-1})}^2\\
        & = 2(L(\bU_n) - L(\bU^{\mathcal{H}}_n)) + 2 \inf_{\varphi\in \mathcal{H}} \normmix{\varphi - \mathcal{S}^\epsilon(\mathbf{U}_{n-1})}^2 + 2\normmix{\mathcal{S}^\epsilon(\mathbf{U}_{n-1}) - \mathcal{S}(\mathbf{U}_{n-1})}^2,
    \end{aligned}
    \end{equation}
    where in the last equality we used the fact $L(\bU^{\mathcal{H}}_n) -  L(\mathcal{S}^\epsilon(\mathbf{U}_{n-1})) = \normmix{\bU^{\mathcal{H}}_n - \mathcal{S}^\epsilon(\mathbf{U}_{n-1})}^2 =  \inf_{\varphi\in \mathcal{H}} \normmix{\varphi - \mathcal{S}^\epsilon(\mathbf{U}_{n-1})}^2$. For the first term on the right hand side of the above equation, note that
    \begin{align*}
        L(\bU_n) - L(\bU^{\mathcal{H}}_n) &= L(\bU_n) - \tilde L^{(B)}(\bU_n) + \tilde L^{(B)}(\bU_n) - \tilde L^{(B)}(\bU^{\mathcal{H}}_n) +  \tilde L^{(B)}(\bU^{\mathcal{H}}_n) - L(\bU^{\mathcal{H}}_n)\\
        & \le L(\bU_n) - \tilde L^{(B)}(\bU_n) + \tilde L^{(B)}(\bU^{\mathcal{H}}_n) - L(\bU^{\mathcal{H}}_n)\\
        & \le 2\sup_{\varphi \in \mathcal{H}} |\tilde L^{(B)}(\varphi) - L(\varphi)|,
    \end{align*}
    where the first inequality follows from $\tilde L^{(B)}(\bU_n) - \tilde L^{(B)}(\bU^{\mathcal{H}}_n) \le 0$ due to $\bU_n \in \arg\min_{\varphi \in \mathcal{H}} \tilde L^{(B)}(\varphi)$. This along with \eqref{eq: one step error decomp} yields the desired result.
\end{proof}
\paragraph{The ZOD error}
It follows from  Assumption \ref{assum: reward and f regularity} and Proposition \ref{prop:zod one-point} that
\begin{equation}
    \normmix{\mathcal{S}^\epsilon(\mathbf{U}_{n-1}) - \mathcal{S}(\mathbf{U}_{n-1})}^2 \le C \epsilon^4.
\end{equation}
\paragraph{The statistical error}
Now we focus on the statistical error $\sup_{\varphi \in \cH} |L(\varphi) - \tilde L^{(B)}(\varphi)|$.
Denote the time discretization version of the loss function with respect to the time grid $ \mathscr{G}:=\left\{ 0 = t_0<t_1<\cdots< t_N = T\right\}$ as
\begin{equation}
    L^{\mathscr{G}}(\varphi) := L^{\mathscr{G}}_v(\varphi_v) + L^{\mathscr{G}}_g(\varphi_g) + L^{\mathscr{G}}_h(\varphi_h)
\end{equation}
where
\begin{equation}
    \begin{aligned}
        L_v^{\mathscr{G}}(\varphi_v) &:= \beta\E\left[\sum_{j=0}^{N-1} e^{\beta t_j}\left( \varphi_v(t_j,X_{t_j}) - \tilde R^{t_j,X_{t_j}}(\mathbf{U}_{n-1}) \right)^2 \Delta t_j\right], \\
    L_{g}^{\mathscr{G}}(\varphi_g) &:= \beta\E\left[\sum_{j = 0}^{N-1} e^{\beta t_j}\left| \varphi_g(t_j,X_{t_j}) - \frac{Z}{\epsilon} \tilde R^{t_j, X_{t_j}+\epsilon Z}(\mathbf{U}_{n-1}) \right|^2 \Delta t_j\right], \\
    L_{h}^{\mathscr{G}}(\varphi_{h}) &:= \E\left[\sum_{j = 0}^{N-1} e^{\beta t_j}\left| \varphi_h(t_j,X_{t_j}) - \frac{ZZ^{\top} - I_d}{\epsilon^2} \tilde R^{t_j, X_{t_j}+\epsilon Z}(\mathbf{U}_{n-1}) \right|^2 \Delta t_j\right].
    \end{aligned}
\end{equation}
Noting  that $ \E[\tilde L^{(B)}(\varphi)] = L^{\mathscr{G}}(\varphi)$, we have
\begin{equation}\label{eq: population loss and empirical loss decomp}
    \sup_{\varphi \in \cH} \left| L(\varphi) - \tilde L^{(B)}(\varphi)\right| \le \sup_{\varphi \in \cH} \left|  \tilde L^{(B)}(\varphi) - L^{\mathscr{G}}(\varphi)\right| + \sup_{\varphi \in \cH} \left|  L^{\mathscr{G}}(\varphi) - L(\varphi)\right|.
\end{equation}
The second term on the right hand side of \eqref{eq: population loss and empirical loss decomp} is the time discretization error, which is bounded by the following lemma.
\begin{lemma}\label{lemma: discretization error}
Under Assumptions \ref{assum: hypothesis space}, \ref{assum: reward and f regularity}, \ref{assum: sde coef}, there is a constant $ C$ independent of $ \epsilon$ and $ N$ such that
\begin{equation}
    \sup_{\varphi \in \cH} \left|  L^{\mathscr{G}}(\varphi) - L(\varphi)\right| \le C \epsilon^{-4}N^{-1}.
\end{equation}
\end{lemma}
\begin{proof}
    We aim to bound the difference between the continuous-time population loss $L(\varphi)$ and its discrete-time approximation $L^{\mathcal{G}}(\varphi)$. We prove the Hessian loss for illustration, and the value and gradient components are treated in the same way and are of lower order in $\epsilon$. Let the estimator be denoted by
\begin{equation}
    \xi^\epsilon(Z) := \frac{ZZ^\top - I_d}{\epsilon^2},\quad \tilde f(t,x):=f(t,x,\mathbf U_{n-1}(t,x)).
\end{equation}
The continuous and discrete losses are defined as:
\begin{align}
    L_h(\varphi_h) &:=  \int_0^T \E \left[ e^{\beta t} \left| \varphi_h(t, X_t) - \xi^\epsilon(Z)R^{t, X_t+\epsilon Z}(\mathbf{U}_{n-1}) \right|^2 \right] \dd t, \\
    L_h^{\mathcal{G}}(\varphi_h) &:=  \sum_{i=0}^{N-1} \E \left[ e^{\beta t_i} \left| \varphi_h(t_i, X_{t_i}) - \xi^\epsilon(Z) \tilde R^{t_i, X_{t_i}+\epsilon Z}(\mathbf{U}_{n-1}) \right|^2 \right] \Delta t_i,
\end{align}
where $\tilde R^{t,x}(\bU_{n-1})$ uses the discrete reward approximation
\begin{equation}
    \tilde R^{t,x}(\bU_{n-1}) = g(X_T^{t,x}) + \sum_i f(t_i,X_{t_i}^{t,x},\bU_{n-1}(t_i,X_{t_i}^{t,x}))\Delta t_i,
\end{equation}
instead of the exact $R^{t,x}(\bU_{n-1})$
\begin{equation}
    R^{t,x}(\bU_{n-1}) = g(X_T^{t,x}) + \int_t^{T} f(s,X_s^{t,x},\bU_{n-1}(s,X_{s}^{t,x}))\dd s.
\end{equation}
Denote
\[
    \Delta(t,y)
    :=
    \tilde R^{t,y}(\mathbf U_{n-1})
    -
    R^{t,y}(\mathbf U_{n-1}).
\]
Introduce the intermediate grid loss $\bar L_h^{\mathscr G}$ obtained by
replacing the time integral in $L_h$ with the grid sum while keeping the exact
reward $R$.
\begin{equation}
    \Bar{L}_h(\varphi_h)^{\mathscr{G}} :=   \sum_{i=0}^{N-1} \E \left[ e^{\beta t_i} \left| \varphi_h(t_i, X_{t_i}) - \xi^\epsilon(Z)  R^{t_i, X_{t_i}+\epsilon Z}(\mathbf{U}_{n-1}) \right|^2 \right] \Delta t_i.
\end{equation}

We have the following inner and outer error decompositions
\begin{equation}
    \begin{aligned}
        \left|L_h(\varphi_h) - L_h^{\mathcal{G}}(\varphi_h)\right|\le \underbrace{\left|L_h^{\mathscr G}(\varphi_h)-\bar L_h^{\mathscr G}(\varphi_h) \right|}_{\textrm{Inner Discretization Error}}+ \underbrace{\left| \bar L_h^{\mathscr G}(\varphi_h)-L_h(\varphi_h)\right|}_{\textrm{Outer Discretization Error}}.
    \end{aligned}
\end{equation}
Therefore, there are two types of error -- the inner and the outer discretization errors.
For the inner discretization error, it suffices to bound
\begin{equation}\label{eq: inner decomp}
\begin{aligned}
     &\left| \E \left[ \left\| \varphi_h(t_i, X_{t_i}) - \xi^\epsilon(Z)R^{t_i, X_{t_i}+\epsilon Z} \right\|^2 \right] - \E \left[ \left\| \varphi_h(t_i, X_{t_i}) - \xi^\epsilon(Z)\tilde R^{t_i, X_{t_i}+\epsilon Z} \right\|^2 \right]\right|\\
     &\le \E\left[ \left| 2\varphi_h(t_i,X_{t_i}) - \xi^{\epsilon}(Z)(\tilde R^{t_i,X_{t_i}+ \epsilon Z} + R^{t_i,X_{t_i} + \epsilon Z})\right|\left| \xi^{\epsilon}(Z)\right|\left|\Delta(t_i,X_{t_i}+\epsilon Z)\right|\right].
\end{aligned}
\end{equation}
Due to the polynomial growth condition and the moment bounds of $ X_t$, it suffices to bound $ \|\Delta(t_i,X_{t_i} + \epsilon Z)\|_{L_2}$.
To this end,
recall that $R^{t,y} = g(X_T^{t,y}) + \int_t^T \tilde f(s, X_s^{t,y}) \dd s$. The terminal part cancels and
\[
    \Delta(t_i,y)
    =
    \sum_{k=i}^{N-1}
    \int_{t_k}^{t_{k+1}}
    \left[
    \tilde f(t_k,X_{t_k}^{t_i,y})
    -
    \tilde f(s,X_s^{t_i,y})
    \right]\dd s .
\]
By Assumption \ref{assum: reward and f regularity}, It\^o's formula gives
\[
    \dd \tilde f(s,X_s^{t_i,y})
    =
    A_s\,\dd s+ B_{s}\,\dd W_s,
    \qquad
    \E\left[|A_s|^2+|B_{s}|^2\right]
    \le C(1+|y|^q)
\]
for some $q\ge0$, uniformly in $s$.
Stochastic Fubini therefore yields
\[
\begin{aligned}
    \Delta(t_i,y)
    =
    -\sum_{k=i}^{N-1}
    \int_{t_k}^{t_{k+1}}(t_{k+1}-u)A_u\,\dd u
    -\sum_{k=i}^{N-1}
    \int_{t_k}^{t_{k+1}}(t_{k+1}-u)B_{u}\,\dd W_u .
\end{aligned}
\]
The drift term has $L^2$ norm at most $CN(T/N)^2\le C/N$.
For the martingale term, It\^o's isometry and orthogonality of martingale
increments give
\[
    \E\left|
    \sum_{k=i}^{N-1}
    \int_{t_k}^{t_{k+1}}(t_{k+1}-u)B_{u}\,\dd W_u
    \right|^2
    \le
    CNh^3
    \le
    \frac{C}{N^2}.
\]
Hence
\[
    \|\Delta(t_i,y)\|_{L^2}
    \le
    \frac{C(1+|y|^q)}{N}.
\]
Conditioning on $Z$, taking $y=X_{t_i}+\epsilon Z$, and using the moment
bounds for $X_{t_i}$ together with Gaussian moments,
\[
    \left\|
    \Delta(t_i,X_{t_i}+\epsilon Z)
    \right\|_{L^2}
    \le
    \frac{C}{N}.
\]
This along with \eqref{eq: inner decomp} yields
\[
    \sup_{\varphi_h\in\mathcal H}
    |L_h^{\mathscr G}(\varphi_h)-\bar L_h^{\mathscr G}(\varphi_h)|
    \le
    \frac{C}{N\epsilon^4}.
\]

Now we bound the outer discretization error:
\begin{equation}
    \left| \int_0^T l(t) \dd t - \sum_{i=0}^{N-1} l(t_i) \Delta t_i \right|, \quad \text{where } l(t) = e^{\beta t} \E \left[ \left\| \varphi_h(t, X_t) - \xi^\epsilon(Z)R^{t,X_t + \epsilon Z} \right\|^2 \right].
\end{equation}
It suffices to show that the time derivative $\frac{\dd}{\dd t} l(t)$ is bounded. Indeed,
\begin{equation}
\begin{aligned}
    l(t) &= e^{\beta t}\E \left[ \left\| \varphi_h(t, X_t) - \xi^\epsilon(Z)R^{t,X_t + \epsilon Z} \right\|^2 \right] \\
    &= e^{\beta t}\left(\E\left[ \left\|\varphi_h(t,X_t)\right\|^2\right] + \E \left[ \left\| \xi^{\epsilon}(Z) R^{t, X_t+\epsilon Z} \right\|^2 \right] - 2\E\left[ \langle \varphi_h(t,X_t), \xi^{\epsilon}(Z) \rangle R^{t,X_t + \epsilon Z}\right]\right).
    \end{aligned}
\end{equation}
Since $ \varphi_h$ is smooth and bounded, we need only to analyze the second and third terms.

For the second term, using the conditional expectation, we have
\begin{equation}
    I_2(t):=\E\left[ \left\| \xi^{\epsilon}(Z) R^{t, X_t+\epsilon Z} \right\|^2\right] = \E\left[ \|\xi^{\epsilon}(Z)\|^2 \E\left[ (R^{t,X_t + \epsilon Z})^2\mid Z,X_t\right]\right].
\end{equation}

Denote $u(t, x) := \E [ (R^{t,x})^2 ]$.
Under Assumption \ref{assum: reward and f regularity}, the  derivatives of $ u$ satisfy the polynomial growth condition.
Therefore, the time derivative of the expectation is:
\begin{equation}
\begin{aligned}
    \frac{\dd}{\dd t}\E\left[ \left\| \xi^{\epsilon}(Z) R^{t, X_t+\epsilon Z} \right\|^2\right]& =   \E[\|\xi^{\epsilon}(Z)\|^2 \frac{\dd}{\dd t}u(t, X_t + \epsilon Z)] \\
    &= \E \left[ \|\xi^{\epsilon}(Z)\|^2(\partial_t + \mathcal{L}) F_Z(t,x)\big|_{x = X_t} \right],
    \end{aligned}
\end{equation}
where $ F_Z(t,x) = u(t,x + \epsilon Z)$.
This confirms that $|\frac{\dd}{\dd t} I_2(t)| \le C \epsilon^{-4}$.

Similarly, for the cross term involving $\E [ \langle \varphi, \xi^\epsilon(Z)R^{t,X_t+\epsilon Z} \rangle ]$, the derivative involves terms related to $\E[R^{t,x}]$. Using Assumption \ref{assum: reward and f regularity}, the derivative is bounded by $C \epsilon^{-2}$.

Combining these bounds, the total discretization error is bounded by:
\begin{equation}
    \sup_{\varphi \in \cH} \left| L_h^{\mathcal{G}}(\varphi) - L_h(\varphi) \right| \le C \left( \epsilon^{-4} N^{-1} + \epsilon^{-2} N^{-1} \right) \le C \epsilon^{-4} N^{-1}.
\end{equation}
This proves the desired result.
\end{proof}

The first term on the right hand side of \eqref{eq: population loss and empirical loss decomp} is the classical statistical error, and  we aim to bound it via the Rademacher complexity. Typically, the Rademacher complexity framework works for bounded loss functions. We first truncate the loss function and denote $ G$ as the implicit function and $ W$ as the underlying random variable, such that $ L^{\mathscr{G}}(\varphi) = \E[G(\varphi;W)]$ and $ \tilde L^{(B)}(\varphi) = \frac{1}{B}\sum_{i=1}^B G(\varphi;W_i)$.
Write the grid loss as
\[
    G(\varphi;W)
    =
    \sum_{j=0}^{N-1}\Delta t_j\,\ell_j(\varphi;W),
\]
where $\ell_j$ is the one-point squared value-gradient-Hessian loss at time
$t_j$.  For each grid index, define the local truncation event
\[
    \Omega_{K,j}
    :=
    \left\{\sup_{0\le t\le T}|X_t|<K\right\}
    \cap
    \{|Z|<K\}
    \cap
    \left\{
    \sup_{t_j\le s\le T}
    \left|X_s^{t_j,X_{t_j}+\epsilon Z}\right|<K
    \right\}.
\]
We clip the grid loss term by term:
\begin{equation}
    G^{(K)}(\varphi;W)
    :=
    \sum_{j=0}^{N-1}
    \Delta t_j\,\ell_j(\varphi;W)\1_{\Omega_{K,j}} .
\end{equation}
By Assumption \ref{assum: reward and f regularity}, there exists a finite exponent $q\ge1$, depending only on the polynomial growth orders of $g(x)$, $J_1(t,x;\varphi)$, $J_2(t,x;\varphi)$, and $\tilde f(t,x;\varphi)$ uniformly over $\varphi\in\mathcal H$, such that on $\Omega_{K,j}$ the one-point value, gradient, and Hessian targets are bounded by $ C\epsilon^{-2}(1+K^{q})$. As a consequence of this truncation, $ G^{(K)}$ has the following bound
\begin{equation}
    |G^{(K)}(\varphi;W)|
    \le
    C\epsilon^{-4}(1+K^{q}) .
\end{equation}
Next, we aim to derive the probability of the truncation sets. To this end, we need the following lemmas.
To derive the tail probability of the trajectory of $ X$, we will use the following Bernstein martingale inequality from \cite{dzhaparidze_bernstein-type_2001}.
\begin{lemma}\label{lemma: bernstein inequality}
    Let $ M_t$ be a continuous square integrable martingale. Then
    \begin{equation}
        \p\left( \sup_{0\le t\le T}|M_t| \ge z, \ \langle M\rangle_T \le L\right) \le 2\exp\left( -\frac{1}{2}\frac{z^2}{L}\right).
    \end{equation}
\end{lemma}

Based on this lemma, we state the tail probability of the trajectory of $ X$.
\begin{lemma}[Sub-Gaussian Tail of Base Trajectory]\label{lemma: base tail}
    Under Assumption \ref{assum: sde coef}, let $X$ be the solution to the SDE \eqref{eq:sde} starting from $X_0$. Then there exist constants $C_1, C_2 > 0$ depending only on $T, L_b, \sigma_{0}$ and $X_0$, such that for any $K > 0$:
\begin{equation}
    \mathbb{P}\left( \sup_{0 \le t \le T} |X_t| > K \right) \le C_1 \exp(-C_2 K^2).
\end{equation}
\end{lemma}
\begin{proof}
The SDE in integral form is given by:
\begin{equation}
    X_t = X_0 + \int_0^t b(s, X_s) \, \mathrm{d}s + M_t, \quad \text{where } M_t = \sqrt{2}\int_0^t \sigma(s, X_s) \, \mathrm{d}W_s.
\end{equation}
Note that $M_t$ is a continuous square-integrable martingale. Taking the norm and using the linear growth condition we get
\begin{align}
    |X_t| &\le |X_0| + \int_0^t L_b(1 + |X_s|) \, \mathrm{d}s + |M_t| \nonumber \\
          &= \left( |X_0| + L_b T + |M_t| \right) + L_b \int_0^t |X_s| \, \mathrm{d}s.
\end{align}
Let $M^*_T := \sup_{0 \le t \le T} |M_t|$. Applying Grönwall's inequality yields a path-wise bound:
\begin{equation}
    \sup_{0 \le t \le T} |X_t| \le \left( |X_0| + L_b T + M^*_T \right) e^{L_b T}.
\end{equation}
Therefore, the event $\{\sup_{0 \le t \le T} |X_t| > K\}$ implies:
\begin{equation}
    M^*_T > K e^{-L_b T} - (|X_0| + L_b T).
\end{equation}
For sufficiently large $K$, the RHS is approximately proportional to $K$. Thus, it suffices to bound the tail of $M^*_T$.
We invoke Bernstein's inequality for continuous martingales (Lemma \ref{lemma: bernstein inequality}).
Under the bounded diffusion assumption, the quadratic variation is deterministically bounded:
\begin{equation}
    \langle M \rangle_T = \int_0^T \|\sigma(s, X_s)\|^2 \, \mathrm{d}s \le T \sigma_{\max}^2.
\end{equation}
Setting $L = T \sigma_{\max}^2$, we obtain:
\begin{equation}
    \mathbb{P}(M^*_T \ge z) \le 2 \exp\left( -\frac{z^2}{2 T \sigma_{\max}^2} \right).
\end{equation}
Combining this with the Grönwall bound proves the lemma.
\end{proof}
 Recall the $\psi_2$-Orlicz norm (sub-Gaussian norm) is defined as $\|X\|_{\psi_2} = \inf \{ C > 0 : \mathbb{E}[\exp(|X|^2/C^2)] \le 2 \}$.

As a direct consequence of Lemma \ref{lemma: base tail}, the supremum of the base trajectory satisfies:
\begin{equation}\label{eq: orlicz norm upper bound}
    \left\|\sup_{0\le t\le T}|X_t|\right\|_{\psi_2} \le C_{base},
\end{equation}
where $C_{base}$ is a constant depending only on the SDE coefficients and $T$.

We now extend this analysis to the branching paths used in the ZOD estimation. Let $\mathcal{G} = \{t_1, \dots, t_N\}$ be the time discretization grid. For each $j$, let $X_s^{t_j, \xi_j}$ denote the conditional path starting at time $t_j$ from $\xi_j = X_{t_j} + \epsilon Z$. The path satisfies:
\begin{equation}
    X_s^{t_j, \xi_j} = \xi_j + \int_{t_j}^s b(u, X_u^{t_j,\xi_j}) \, \mathrm{d}u + \sqrt{2}\int_{t_j}^s \sigma(u, X_u^{t_j,\xi_j}) \, \mathrm{d}W_u, \quad s \in [t_j, T].
\end{equation}

The following lemma bounds the branching trajectories.
\begin{lemma}[Uniform Tail for Branching Paths]
\label{lemma: branching tail}
    Under Assumption \ref{assum: sde coef}, there exists a constant $C > 0$ independent of $j, N$ such that for any $K>0$:
    \begin{equation}
        \sup_{0\le j<N}
        \mathbb{P}\left(
        \sup_{t_j\le s\le T}
        \left| X^{t_j, X_{t_j}+\epsilon Z}_s\right|>K
        \right)
        \le 2\exp(-CK^2).
    \end{equation}
    The same bounds hold with $X_{t_j}-\epsilon Z$ and $ X_{t_j}$ in place of
    $X_{t_j}+\epsilon Z$.
\end{lemma}

\begin{proof}
    Fix an index $j \in \{1, \dots, N\}$. Let $X_s^{(j)} := X_s^{t_j, X_{t_j} + \epsilon Z}$. By applying Grönwall's inequality under the linear growth condition on $b$, we obtain the pathwise bound:
    \begin{equation}
        \sup_{s \in [t_j, T]} |X_s^{(j)}| \le e^{L_b T} \left( |X_{t_j}| + \epsilon |Z| + L_b T + \sup_{s \in [t_j, T]} |M_s^{(j)}| \right),
    \end{equation}
    where $M_s^{(j)} = \sqrt{2}\int_{t_j}^s \sigma(u, X_u^{(j)}) \mathrm{d}W_u$.

    We analyze the Orlicz $\psi_2$-norm of each term on the right-hand side of the above:
    \begin{enumerate}
        \item \textbf{Base Path:} By \eqref{eq: orlicz norm upper bound}, we have $\|\sup |X_t|\|_{\psi_2} \le C_{base}$.
        \item \textbf{Perturbation:} Since $Z \sim \mathcal{N}(0, I_d)$, it is sub-Gaussian with $\|Z\|_{\psi_2} \le C_{Gauss}$.
        \item \textbf{Martingale:} Since the diffusion $\sigma$ is uniformly bounded, the quadratic variation is bounded by $\sigma_{\max}^2 T$. By Bernstein's inequality for continuous martingales, $\sup_{s} |M_s^{(j)}|$ is sub-Gaussian with norm bounded by a constant $C_{mart}$ independent of the path history.
    \end{enumerate}
    By the triangle inequality for the $\psi_2$-norm, the supremum of the $j$-th branching path is sub-Gaussian:
    \begin{equation}\label{eq: orlicz bound}
        \left\| \sup_{t_j \le s \le T} |X_s^{(j)}| \right\|_{\psi_2} \le C_{path},
    \end{equation}
    where $C_{path}$ is independent of $j$ and $N$.

    The tail bound follows directly from \eqref{eq: orlicz bound} and
    the definition of the $\psi_2$-norm. The proof for the branch started at
    $X_{t_j}-\epsilon Z$ is identical.
\end{proof}

Combining Lemma \ref{lemma: base tail}, the Gaussian tail of $Z$, and Lemma
\ref{lemma: branching tail} gives the probability of  the local truncation events
$\Omega_{K,j}^c$ used in the proof of Theorem \ref{thm: one step error bound}.

\begin{lemma}\label{lemma: tail probability}
    Under Assumption \ref{assum: sde coef}, there exist constants
    $C_1,C_2>0$, independent of $j$ and $N$, such that
    \begin{equation}
        \sup_{0\le j<N}\mathbb P(\Omega_{K,j}^c)
        \le
        C_1\exp(-C_2K^2).
    \end{equation}
\end{lemma}
We are ready to obtain the upper bound of the first term on the right hand side of \eqref{eq: population loss and empirical loss decomp}. It can be decomposed as
\begin{equation}\label{eq: empiricial loss decomp}
    \begin{aligned}
        \sup_{\varphi \in \cH}\left| \tilde L^{(B)}(\varphi) - L^{\mathscr{G}}(\varphi)\right| &= \sup_{\varphi\in\cH} \left| \frac{1}{B}\sum_{i=1}^B G(\varphi;W_i) - \E[G(\varphi;W)]\right| \\
        &\le \sup_{\varphi \in \cH} \left| \frac{1}{B}\sum_{i=1}^B G^{(K)}(\varphi;W_i) - \E[G^{(K)}(\varphi;W)] \right|\\
        &\quad+ \sup_{\varphi \in \cH}\left| \frac{1}{B} \sum_{i=1}^B \left( G^{(K)}(\varphi;W_i) - G(\varphi;W_i)\right)\right|\\
        &\quad+ \sup_{\varphi \in \cH}\E\left[ \left| G^{(K)}(\varphi;W) - G(\varphi;W)\right|\right].
    \end{aligned}
\end{equation}
Let
\[
    R_K(W)
    :=
    \sup_{\varphi\in\mathcal H}
    |G(\varphi;W)-G^{(K)}(\varphi;W)|.
\]
Since the networks in $\mathcal H$ are uniformly bounded, the supremum over
$\varphi$ in the local squared loss is bounded pointwise by the reward targets:
\[
    \sup_{\varphi\in\mathcal H}\ell_j(\varphi;W)
    \le
    C\left(
    1+|Y_{v,j}|^2+\|Y_{g,j}^{\epsilon}\|^2+\|Y_{h,j}^{\epsilon}\|^2
    \right),
\]
where
\[
    Y_{v,j}:=\tilde R^{t_j,X_{t_j}}(\mathbf U_{n-1}),\quad
    Y_{g,j}^{\epsilon}:=\frac{Z}{\epsilon}
    \tilde R^{t_j,X_{t_j}+\epsilon Z}(\mathbf U_{n-1}),
\]
and
\[
    Y_{h,j}^{\epsilon}:=
    \frac{ZZ^\top-I_d}{\epsilon^2}
    \tilde R^{t_j,X_{t_j}+\epsilon Z}(\mathbf U_{n-1}).
\]
The assumption \ref{assum: reward and f regularity} and the polynomial
growth of $g$ imply that, for some $q\ge0$,
\[
    |R^{t_j,y}(\mathbf U_{n-1})|
    +
    |\tilde R^{t_j,y}(\mathbf U_{n-1})|
    \le
    C\left(1+
    \sup_{t_j\le s\le T}|X_s^{t_j,y}|^q\right).
\]
Together with Assumption \ref{assum: sde coef} and the Gaussian moments of
$Z$, this gives
\begin{equation}\label{eq: weak local loss second moment}
    \sup_{0\le j<N}
    \E\!\left[
    \left(\sup_{\varphi\in\mathcal H}\ell_j(\varphi;W)\right)^2
    \right]
    \le
    C\epsilon^{-8}.
\end{equation}
Hence, by Cauchy--Schwarz applied to each
grid summand, Lemma \ref{lemma: tail probability}, and
$\sum_{j=0}^{N-1}\Delta t_j=T$,
\begin{equation}\label{eq: term 3 bound}
\begin{aligned}
    \E R_K(W)
    &\le
    \sum_{j=0}^{N-1}\Delta t_j
    \E\left[
    \sup_{\varphi\in\mathcal H}\ell_j(\varphi;W)
    \1_{\Omega_{K,j}^c}\right]                                  \\
    &\le
    \sum_{j=0}^{N-1}\Delta t_j
    \left(
    \E\!\left[
    \left(\sup_{\varphi\in\mathcal H}\ell_j(\varphi;W)\right)^2
    \right]
    \right)^{1/2}
    \mathbb P(\Omega_{K,j}^c)^{1/2}                                  \\
    &\le
    C\epsilon^{-4}\exp(-CK^2).
\end{aligned}
\end{equation}
In particular,
\[
    \sup_{\varphi\in\mathcal H}
    \E\left[
    |G^{(K)}(\varphi;W)-G(\varphi;W)|
    \right]
    \le
    \E R_K(W)
    \le
    C\epsilon^{-4}\exp(-CK^2).
\]
Therefore, by Markov's inequality, with probability at least $1-\delta$,
\begin{equation}\label{eq: clipping remainder bound}
    \frac1B\sum_{i=1}^B R_K(W_i)+\E R_K(W)
    \le
    C\delta^{-1}\epsilon^{-4}\exp(-CK^2).
\end{equation}

For the first term on the right hand side of \eqref{eq: empiricial loss decomp}, we use the Rademacher complexity. Denote $ S = (W_1,\dots,W_B)$ and define
\begin{equation}
\Phi(S):= \sup_{\varphi}\left| \frac{1}{B}\sum_{i=1}^B G^{(K)}(\varphi;W_i) - \E\left[ G^{(K)}(\varphi;W)\right]\right|.
\end{equation}
Let $ S' = (W_1,\dots,W_j',\dots,W_B)$ differ from $ S$ only at $ j$-th component. It follows from Jensen's inequality and the upper bound of $ G^{(K)}$ that
\begin{equation}
    \begin{aligned}
        \left| \Phi(S) - \Phi(S')\right| \le \sup_{\varphi\in\cH}\frac{1}{B}\left|G^{(K)}(\varphi;W_j) - G^{(K)}(\varphi;W_j') \right| \le \frac{ C \epsilon^{-4}(1+ K^{q})}{B}.
    \end{aligned}
\end{equation}
By McDiarmid's inequality,
\begin{equation}\label{eq: term1 bound}
    \p\left( \Phi(S) - \E[\Phi(S)] > C\epsilon^{-4}(1+K^{q})
    \sqrt{\frac{\log(1/\delta)}{B}}\right) < \delta.
\end{equation}
Now we use the Rademacher complexity to bound $ \E[\Phi]$. Let $ W'$ be the independent copy of $ W$, and $
\{\sigma_i\}_{i=1}^B$ are i.i.d. random variables with $ \p(\sigma_i=1) = \p(\sigma_i = -1) = \frac{1}{2}$. Then
\begin{equation}
    \begin{aligned}\label{eq: expect rademecher bound}
        \E[\Phi(S)] &= \E\left[\sup_{\varphi}\left| \frac{1}{B}\sum_{i=1}^B G^{(K)}(\varphi;W_i) - \E\left[ G^{(K)}(\varphi;W)\right]\right| \right]\\
        &= \E\left[ \sup_{\varphi \in \cH}\left| \E_{W'}\left[ \frac{1}{B}\sum_{i=1}^B\left( G^{(K)}(\varphi;W_j') -G^{(K)}(\varphi;W_j)\right)\right]\right|\right]\\
        &=\E\left[ \sup_{\varphi \in \cH}\left| \E_{W',\sigma}\left[ \frac{1}{B}\sum_{i=1}^B\sigma_i\left( G^{(K)}(\varphi;W_j') -G^{(K)}(\varphi;W_j)\right)\right]\right|\right]\\
        & \le \E_{W,W',\sigma}\left[ \sup_{\varphi \in \cH}\left| \frac{1}{B}\sum_{i=1}^B \sigma_i G^{(K)}(\varphi;W_i')\right|\right] + \E_{W,W',\sigma}\left[ \sup_{\varphi \in \cH}\left| \frac{1}{B}\sum_{i=1}^B \sigma_i G^{(K)}(\varphi;W_i)\right|\right]\\
        & \le 2\mathcal{R}_B(G^{(K)}(\cH) \cup -G^{(K)}(\cH))
    \end{aligned}
\end{equation}
where $ G^{(K)}(\cH) = \left\{ G^{(K)}(\varphi;\cdot)| \varphi \in \cH\right\}$. For the Rademacher complexity in terms of the parameterized function class, we have the following lemma by \cite{jiao2024error}.
\begin{lemma}
    Let $ \mathcal{G}$ be a parameterized function class $ \mathcal{G} = \left\{ f(\cdot;\theta): \R^d\rightarrow \R|\theta \in \Theta\right\}$, where $ \Theta = \left\{ \theta| \theta \in \R^p, \|\theta\|_{\infty}< R\right\}$. If there are constants $ C_1, C_2$, such that
    \begin{equation}
        \begin{aligned}
            \sup_{x\in \R^d} |f(x;\theta_1) - f(x;\theta_2)| &\le C_1 \|\theta_1 - \theta_2\|_{\infty},\\
            \sup_{\theta \in \Theta}\sup_{x\in \R^d}|f(x;\theta)| &\le C_2,
        \end{aligned}
    \end{equation}
    then
    \begin{equation}
        \mathcal{R}_m(\mathcal{G}) \le \frac{4}{\sqrt{m}} + \frac{6\sqrt{p}C_2}{\sqrt{m}}\sqrt{\log(2RC_1p\sqrt{m})}.
    \end{equation}
\end{lemma}

Based on this lemma, it follows from Assumption \ref{assum: hypothesis space} that
\begin{equation}\label{eq: rademecher bound}
\begin{aligned}
    &\mathcal{R}_B(G^{(K)}(\cH) \cup -G^{(K)}(\cH)) \\
    &\qquad\le \frac{4}{\sqrt{B}} + \frac{6\sqrt{p}C(1 + \epsilon^{-4}(1+K^{q}))}{\sqrt{B}}\sqrt{\log(2RC\epsilon^{-2}(1+K^{q})p\sqrt{B})},
    \end{aligned}
\end{equation}
where $ C$ is a constant depending only on the coefficients of \eqref{eq:pde}. 
The desired result follows now from \eqref{eq: term1 bound},  \eqref{eq: term 3 bound}, \eqref{eq: expect rademecher bound} and \eqref{eq: rademecher bound} by choosing $ K = C(\log B + \log (1/\delta))^{1/2}$ with $C$ large enough, which gives the factor $L_{B,\delta}^{q}$ after increasing $q$ once more if necessary.

%\paragraph{Strong-simulator multi-point case.}
We now prove Theorem \ref{thm: one step error bound}-(b) with the strong simulator. The only parts that differ from the preceding one-point proof are the ZOD target, the discretization estimate, and the envelope used in the empirical-process bound. The first entry in the minimum defining $\Gamma_{\epsilon,B,\delta}$ follows from the same truncation and Rademacher-complexity argument as in the weak-simulator case. We therefore focus below on the second entry of the minimum under the strong-simulator coupling. In the remainder of this proof, all the rewards
\[
R^{t,x+\epsilon Z}(\mathbf U_{n-1}),\qquad
R^{t,x-\epsilon Z}(\mathbf U_{n-1}),\qquad
R^{t,x}(\mathbf U_{n-1})
\]
are evaluated under the same Brownian path within one strong-simulator query. The exact multi-point derivative targets are the two central differences
\[
\frac{Z}{2\epsilon}
\left(
R^{t,x+\epsilon Z}(\mathbf U_{n-1})
-
R^{t,x-\epsilon Z}(\mathbf U_{n-1})
\right),
\]
and
\[
\frac{ZZ^\top-I_d}{2\epsilon^2}
\left(
R^{t,x+\epsilon Z}(\mathbf U_{n-1})
+
R^{t,x-\epsilon Z}(\mathbf U_{n-1})
-
2R^{t,x}(\mathbf U_{n-1})
\right).
\]
By Gaussian symmetry and linearity, their conditional means are $G^\epsilon$ and $H^\epsilon$. Let $L_m$ be the corresponding population loss and let $\tilde L_m^{(B)}$ be its quadrature empirical version. By the same conditional-expectation decomposition as in Lemma \ref{lemma: total loss decomp},
\begin{equation}
    L_m(\varphi,\mathbf U_{n-1})
    =
    \|\varphi-\mathcal S^\epsilon(\mathbf U_{n-1})\|_{\beta,\mathrm{mix}}^2
    +C_m,
\end{equation}
where $C_m$ is independent of $\varphi$. Hence the ERM argument in Lemma \ref{lemma: one step error decom} gives
\begin{equation}\label{eq: strong one step decom}
\begin{aligned}
E_n^2
\le&
4\sup_{\varphi\in\mathcal H}
\left|L_m(\varphi)-\tilde L_m^{(B)}(\varphi)\right|
+2\inf_{\varphi\in\mathcal H}
\normmix{\varphi-\mathcal S^\epsilon(\mathbf U_{n-1})}^2\\
&\quad
+2\normmix{\mathcal S^\epsilon(\mathbf U_{n-1})
-\mathcal S(\mathbf U_{n-1})}^2 .
\end{aligned}
\end{equation}
By Proposition \ref{prop: zod variance reduction}, the multi-point estimators have the same $O(\epsilon^2)$ bias as the one-point estimators under the stated smoothness assumptions. Therefore,
\begin{equation}\label{eq: strong zod bias}
    \normmix{\mathcal S^\epsilon(\mathbf U_{n-1})
    -\mathcal S(\mathbf U_{n-1})}^2
    \le C\epsilon^4.
\end{equation}

It remains to bound the statistical error term in \eqref{eq: strong one step decom}. Let $L_m^{\mathscr G}$ be the time-discretized version of $L_m$. As before,
\begin{equation}\label{eq: strong gen decomp}
\sup_{\varphi\in\mathcal H}
\left|L_m(\varphi)-\tilde L_m^{(B)}(\varphi)\right|
\le
\sup_{\varphi\in\mathcal H}
\left|\tilde L_m^{(B)}(\varphi)-L_m^{\mathscr G}(\varphi)\right|
+
\sup_{\varphi\in\mathcal H}
\left|L_m^{\mathscr G}(\varphi)-L_m(\varphi)\right|.
\end{equation}
We first control the discretization term. To separate the two sources of time discretization, let $\bar L_m^{\mathscr G}$ denote the time-grid population loss obtained from $L_m$ by replacing the time integral with the grid sum while keeping the exact rewards $R^{t,x}(\mathbf U_{n-1})$ unchanged. Then
\begin{equation}\label{eq: strong inner outer discretization split}
\begin{aligned}
    \sup_{\varphi\in\mathcal H}
    \left|L_m^{\mathscr G}(\varphi)-L_m(\varphi)\right|
    \le
    \sup_{\varphi\in\mathcal H}
    \left|L_m^{\mathscr G}(\varphi)-\bar L_m^{\mathscr G}(\varphi)\right| +
    \sup_{\varphi\in\mathcal H}
    \left|\bar L_m^{\mathscr G}(\varphi)-L_m(\varphi)\right| .
\end{aligned}
\end{equation}
Write
\[
\Delta(t,x)
:=
\tilde R^{t,x}(\mathbf U_{n-1})
-
R^{t,x}(\mathbf U_{n-1}).
\]
Because the three rewards in the multi-point estimator are generated by the same Brownian path and the same quadrature grid, the discretization error enters through the finite differences
\[
\frac{Z}{2\epsilon}\{\Delta(t,x+\epsilon Z)-\Delta(t,x-\epsilon Z)\}
\]
and
\[
\frac{ZZ^\top-I_d}{2\epsilon^2}
\{\Delta(t,x+\epsilon Z)+\Delta(t,x-\epsilon Z)-2\Delta(t,x)\}.
\]
Under the regularity assumptions $b,\sigma\in C_b^{0,5}$, $f(\cdot,\cdot,\mathbf U_{n-1}(\cdot,\cdot))\in C_p^{1,5}$, and $g\in C_p^4$,
using the same method as in the proof of Lemma \ref{lemma: discretization error}, there are constants $C>0, q>0$ such that
\begin{equation}\label{eq: strong reward discretization derivatives}
    \left\|
    \partial_x^r \Delta(t,x)
    \right\|_{L^2}
    \le \frac{C}{N}(1+|x|^q).
\end{equation}
Applying Taylor's expansion yields
\begin{align}
\left\|
\frac{Z}{2\epsilon}
\{\Delta(t,X_t+\epsilon Z)-\Delta(t,X_t-\epsilon Z)\}
\right\|_{L^2}
&\le \frac{C}{N},\\
\left\|
\frac{ZZ^\top-I_d}{2\epsilon^2}
\{\Delta(t,X_t+\epsilon Z)+\Delta(t,X_t-\epsilon Z)-2\Delta(t,X_t)\}
\right\|_{L^2}
&\le \frac{C}{N}.
\end{align}
Using the same method as in the proof of Lemma \ref{lemma: discretization error} gives
\begin{equation}\label{eq: strong inner discretization bound}
    \sup_{\varphi\in\mathcal H}
    \left|L_m^{\mathscr G}(\varphi)-\bar L_m^{\mathscr G}(\varphi)\right|
    \le \frac{C}{N}.
\end{equation}
 The outer discretization error is treated in the same way, leading to
\begin{equation}\label{eq: strong discretization bound}
    \sup_{\varphi\in\mathcal H}
    \left|L_m^{\mathscr G}(\varphi)-L_m(\varphi)\right|
    \le \frac{C}{N}.
\end{equation}

For the first term on the right hand side of  \eqref{eq: strong gen decomp}, denote and note that
\begin{equation}\label{eq: strong central identities}
\begin{aligned}
    Y_g^{\epsilon,m}(t,x)
    &:= \frac{Z}{2\epsilon}
\left(
\tilde R^{t,x+\epsilon Z}(\mathbf U_{n-1})
-
\tilde R^{t,x-\epsilon Z}(\mathbf U_{n-1})
\right)\\
&= 
    Z\int_{-1}^{1}
    \nabla_x \tilde R^{t,x+s\epsilon Z}(\mathbf U_{n-1})^\top Z\,
    \frac{\dd s}{2},                                      \\
    Y_h^{\epsilon,m}(t,x)
    &:= \frac{ZZ^\top-I_d}{2\epsilon^2}
\left(
\tilde R^{t,x+\epsilon Z}(\mathbf U_{n-1})
+
\tilde R^{t,x-\epsilon Z}(\mathbf U_{n-1})
-
2\tilde R^{t,x}(\mathbf U_{n-1})
\right) \\
&= 
    \frac{ZZ^\top-I_d}{2}
    \int_{-1}^{1}(1-|s|)
    Z^\top\nabla_x^2 \tilde R^{t,x+s\epsilon Z}(\mathbf U_{n-1})Z\,
    \dd s .
\end{aligned}
\end{equation}
Set
\[
    \tilde f(s,x):=f(s,x,\mathbf U_{n-1}(s,x)),
    \qquad
    D_s^{(r)}:=\partial_x^r X_s^{t,x},\quad r=1,\ldots,4 .
\]
For $D^{(1)}$, differentiating the SDE gives a linear equation with bounded drift and multiplicative diffusion coefficients:
\[
    \dd D_s^{(1)}
    =
    \nabla_x b(s,X_s^{t,x})D_s^{(1)}\,\dd s
    +
    \sqrt2\sum_{\ell=1}^d
    \nabla_x\sigma_\ell(s,X_s^{t,x})D_s^{(1)}\,\dd W_s^\ell .
\]
For $r=2,3,4$, $D^{(r)}$ satisfies
\begin{equation}\label{eq: strong derivative flow sde}
    \dd D_s^{(r)}
    =
    A_s^{(r)}D_s^{(r)}\,\dd s
    +
    \sqrt2\sum_{\ell=1}^d B_{\ell,s}^{(r)}D_s^{(r)}\,\dd W_s^\ell
    +
    F_s^{(r)}\,\dd s
    +
    \sqrt2\sum_{\ell=1}^d H_{\ell,s}^{(r)}\,\dd W_s^\ell ,
\end{equation}
where $A_s^{(r)}$ and $B_{\ell,s}^{(r)}$ are bounded, and the forcing terms depend polynomially on lower-order derivatives:
\[
\begin{aligned}
    |F_s^{(2)}|+\sum_{\ell}|H_{\ell,s}^{(2)}|
    &\le C|D_s^{(1)}|^2,                                      \\
    |F_s^{(3)}|+\sum_{\ell}|H_{\ell,s}^{(3)}|
    &\le C\left(|D_s^{(1)}|^3+|D_s^{(1)}||D_s^{(2)}|\right),     \\
    |F_s^{(4)}|+\sum_{\ell}|H_{\ell,s}^{(4)}|
    &\le C\left(|D_s^{(1)}|^4+|D_s^{(1)}|^2|D_s^{(2)}|
    +|D_s^{(2)}|^2+|D_s^{(1)}||D_s^{(3)}|\right).
\end{aligned}
\]
Applying BDG and Gr\"onwall first to $D^{(1)}$ and then inductively to \eqref{eq: strong derivative flow sde} yields, for every $p\ge2$,
\begin{equation}\label{eq: strong flow derivative moment growth}
    \sup_{t,x}
    \left\|
    1+\sum_{r=1}^4\sup_{t\le s\le T}|D_s^{(r)}|
    \right\|_{L^p}
    \le
    C e^{cp}.
\end{equation}
Since
\[
    \tilde R^{t_j,x}(\mathbf{U}_{n-1})
    =
    g(X_T^{t_j,x})
    +
     \sum_{i = j}^{N-1}\tilde f(t_i,X_{t_i}^{t_j,x})\Delta t_i,
\]
Fa\`a di Bruno's formula, along with the fact that  $g\in C_p^4$, $\tilde f\in C_p^{1,5}$ and the bound \eqref{eq: strong flow derivative moment growth}, implies that there are constant $ C,c>0$ and $ q>0$ such that for every $ p\ge 2$, 
\begin{equation}\label{eq: strong reward derivative moment growth}
    \sup_{j}
    \left\|
    |\nabla_x\tilde R^{t_j,x}(\mathbf{U}_{n-1})|
    +|\nabla_x^2\tilde R^{t_j,x}(\mathbf{U}_{n-1})|
    \right\|_{L^p}
    \le
    C e^{cp}(1+ |x|^q).
\end{equation}
Combining \eqref{eq: strong central identities} with
\[
    \|Z\|_{L^p}\le Cp^{1/2},
    \qquad
    \|ZZ^\top-I_d\|_{L^p}\le Cp,
\]
and the moment bound of $ X_t$ gives
\begin{equation}\label{eq: strong target moment growth}
    \sup_{0<\epsilon\le1}\sup_{j}
    \left\|
    |Y_g^{\epsilon,m}(t_j,X_{t_j})|
    +
    |Y_h^{\epsilon,m}(t_j,X_{t_j})|
    \right\|_{L^p}
    \le
    C e^{cp},
    \qquad p\ge2.
\end{equation}

Define the truncation event
\[
    \Omega_{K,m,j}
    :=
    \left\{\sup_{t_j\le t\le T} \left|X_t^{t_j,X_{t_j}}\right|<K\right\}\cap\{|Y_g^{\epsilon,m}(t_j,X_{t_j})|\le K\}
    \cap
    \{|Y_h^{\epsilon,m}(t_j,X_{t_j})|\le K\}.
\]
With $q_K=\max\{2,\lfloor \alpha\log(1+K)\rfloor\}$, Markov's inequality and \eqref{eq: strong target moment growth} give
\[
\begin{aligned}
    \mathbb P(|Y_g^{\epsilon,m}(t_j,X_{t_j})|>K)
    &\le
    \left(\frac{Ce^{cq_K}}{K}\right)^{q_K},              \\
    \mathbb P(|Y_h^{\epsilon,m}(t_j,X_{t_j})|>K)
    &\le
    \left(\frac{Ce^{cq_K}}{K}\right)^{q_K}.
\end{aligned}
\]
Choosing $ \alpha >0 $ small enough and $ K$ large enough yields
\begin{equation}\label{eq: strong target truncation tail}
    \sup_{1\le j\le N}\mathbb P(\Omega_{K,m,j}^c)
    \le
    C_1\exp\!\left(-C_2(\log(1+K))^2\right).
    \qquad K\ge2,
\end{equation}
Write the time-discretized sample loss as
\[
    G_m(\varphi;W)
    =
    \sum_{j=0}^{N-1}\Delta t_j\,\ell_{m,j}(\varphi;W),
\]
where $\ell_{m,j}$ is the squared value-gradient-Hessian loss at the grid
time $t_j$. Clip the grid loss term by the term:
\[
    G_m^{(K)}(\varphi;W)
    :=
    \sum_{j=0}^{N-1}\Delta t_j\,
    \ell_{m,j}(\varphi;W)\1_{\Omega_{K,m,j}} .
\]
Hence, after summing over the time grid and using $\sum_{j=0}^{N-1}\Delta t_j=T$, 
\begin{equation}\label{eq: strong envelope}
    |G_m^{(K)}(\varphi;W)|
    \le
    C(1+K^{q}) ,
\end{equation}
where $ q$ depends on the polynomial growth condition of $ g$ and $ f$.
Recall the decomposition \eqref{eq: empiricial loss decomp}, by McDiarmid's inequality and the parameterized Rademacher bound,
\begin{equation}\label{eq: strong clipped process}
\begin{aligned}
    &\sup_{\varphi\in\mathcal H}
    \left|
    \frac1B\sum_{i=1}^B G_m^{(K)}(\varphi;W_i)
    -
    \E[G_m^{(K)}(\varphi;W)]
    \right|                                                     \\
    &\qquad\le
C(1+\sqrt{\log(1/\delta)})
\frac{\mathcal P(d)(1+K^{q})}{\sqrt B}
\sqrt{\log(CRp(1+K^{q})\sqrt B)} .
\end{aligned}
\end{equation}
with probability at least $1-\delta$. Let
\[
    R_K(W)
    :=
    \sup_{\varphi\in\mathcal H}
    |G_m(\varphi;W)-G_m^{(K)}(\varphi;W)| .
\]
Since the networks in
$\mathcal H$ are uniformly bounded,
\[
    \sup_{\varphi\in\mathcal H}\ell_{m,j}(\varphi;W)
    \le
    C\left(
    1+|Y_v|^2+\|Y_g^{\epsilon,m}\|^2+\|Y_h^{\epsilon,m}\|^2
    \right).
\]
Therefore,
\[
\begin{aligned}
    \E R_K(W)
    &\le
    \sum_{j=0}^{N-1}\Delta t_j
    \E\left[
    \sup_{\varphi\in\mathcal H}\ell_{m,j}(\varphi;W)
    \1_{\Omega_{K,m,j}^c}
    \right]                                                  \\
    &\le
    C\sum_{j=0}^{N-1}\Delta t_j
    \left\|
    1+|Y_v|^2+\|Y_g^{\epsilon,m}\|^2+\|Y_h^{\epsilon,m}\|^2
    \right\|_{L^2}
    \mathbb P(\Omega_{K,m,j}^c)^{1/2}.
\end{aligned}
\]
It follows from the moment bound \eqref{eq: strong target moment growth}, the moments of
$X_{t_j}$, the truncation tail
\eqref{eq: strong target truncation tail}, and the identity
$\sum_{j=0}^{N-1}\Delta t_j=T$ that
\[
    \E R_K(W)
    \le
    C\exp\!\left(-C(\log(1+K))^2\right).
\]
 Markov's
inequality therefore gives,
with probability at least $1-\delta$,
\[
    \frac1B\sum_{i=1}^B R_K(W_i)+\E R_K(W)
    \le
    C\delta^{-1}
    \exp\!\left(-C(\log(1+K))^2\right).
\]
Choosing $ K=\exp\!\left(C\sqrt{\log B+\log(1/\delta)}\right)$,
we obtain
\begin{equation}\label{eq: strong empirical bound}
\begin{aligned}
    &\sup_{\varphi\in\mathcal H}
    \left|\tilde L_m^{(B)}(\varphi)-L_m^{\mathscr G}(\varphi)\right|  \\
    &\qquad\le
    C\left(1+\sqrt{\log(1/\delta)}\right)
    \frac{\mathcal P(d)}{B^{1/2}}
    \exp\!\left(C\sqrt{\log B+\log(1/\delta)}\right).
\end{aligned}
\end{equation}
Combining \eqref{eq: strong one step decom}, \eqref{eq: strong zod bias}, \eqref{eq: strong gen decomp}, \eqref{eq: strong discretization bound}, and \eqref{eq: strong empirical bound} proves the strong-simulator bound.

\subsection{Proof of Theorem \ref{thm:total error}}
The conclusion follows by taking $\delta = \delta/n$ in Theorem \ref{thm: one step error bound}, and combining with the conclusion in Theorem \ref{thm: global_convergence}.

% Note: in this sample, the section number is hard-coded in. Following
% proper LaTeX conventions, it should properly be coded as a reference:

%In this appendix we prove the following theorem from
%Section~\ref{sec:textree-generalization}:

\vskip 0.2in
\bibliography{ref}

\end{document}